\documentclass[twoside,11pt]{article}
\PassOptionsToPackage{numbers,sort&compress}{natbib}
\usepackage[preprint]{jmlr2e}
\bibpunct{[}{]}{,}{n}{}{,}
\usepackage[table]{xcolor}
\usepackage{subcaption} 

\usepackage{lastpage}
\usepackage{caption}
\usepackage{subcaption}

\definecolor{MidnightBlue}{RGB}{25,25,112}
\definecolor{MidnightBlueComplementingGreen}{RGB}{25,112,25}
\definecolor{MidnightBlueComplementingPurple}{RGB}{112,25,112}
\definecolor{MidnightBlueComplementingRed}{RGB}{112,25,69}
\definecolor{coolblack}{rgb}{0.0, 0.18, 0.39}

\definecolor{deepjunglegreen}{rgb}{0.0, 0.29, 0.29}
\definecolor{applegreen}{rgb}{0.55, 0.71, 0.0}
\definecolor{WowColor}{rgb}{.75,0,.75}
\definecolor{MildlyAlarming}{rgb}{0.85,0.25,0.1}
\definecolor{SubtleColor}{rgb}{0,0,.50}
\definecolor{SubtleColor2}{rgb}{0.6,0.21,.50}

\definecolor{lasallegreen}{rgb}{0.03, 0.47, 0.19}
\definecolor{KoneBG}{RGB}{217,207,223}     
\definecolor{KtwoBG}{RGB}{210,222,240}     
\definecolor{KthreeBG}{RGB}{211,223,206}   
\definecolor{KoneText}{RGB}{75,20,130}     
\definecolor{KtwoText}{RGB}{0,0,200}       
\definecolor{KthreeText}{RGB}{0,90,0}      

\NewDocumentCommand{\anastasis}{mo}{
    \IfValueF{#2}{
                        {{\scriptsize
                            \textcolor{deepjunglegreen}{
                            \hfill\\
                                \textbf{A:}
                                \textit{{#1}}
                            \hfill\\
                            }
                        }}
        }
    \IfValueT{#2}{
                        \marginnote{{\scriptsize
                            \textcolor{deepjunglegreen}{ 
                            \textbf{A:}
                            \textit{{#1}}
                            }
                        }}
        }
                    }

\NewDocumentCommand{\dan}{mo}{
    \IfValueF{#2}{
                        {{\scriptsize
                            \textcolor{coolblack}{ 
                            \textbf{D:}
                            \textit{{#1}}
                            }
                        }}
        }
    \IfValueT{#2}{
                        \marginnote{{\scriptsize
                            \textcolor{coolblack}{ 
                            \textbf{D:}
                            \textit{{#1}}
                            }
                        }}
        }
                    }

\NewDocumentCommand{\Bum}{mo}{
    \IfValueF{#2}{\textcolor{teal}{\textbf{Bum: }\textit{#1}}}
    \IfValueT{#2}{\marginnote{{\scriptsize\textcolor{teal}{\textbf{Bum: }\textit{#1}}}}}
}

\usepackage[utf8]{inputenc}
\usepackage[T1]{fontenc}

\usepackage{enumitem} 
\usepackage{microtype}
\usepackage{url}
\usepackage{graphicx}
\usepackage{float}

\usepackage{amssymb, amsfonts}
\usepackage{amsmath}
\usepackage{nicefrac}

\usepackage{booktabs}
\usepackage{multirow}
\usepackage{array}

\usepackage{colortbl} 
\definecolor{darkgreen}{rgb}{0, .5, 0}
\definecolor{darkcerulean}{rgb}{0.03, 0.27, 0.49}
\definecolor{smokyblack}{rgb}{0.06, 0.05, 0.03}
\definecolor{warmblack}{rgb}{0.0, 0.26, 0.26}
\definecolor{cobalt}{rgb}{0.0, 0.28, 0.67}
\definecolor{aoEnglish}{rgb}{0.0, 0.5, 0.0}
\definecolor{carribeangreen}{rgb}{0.0, 0.8, 0.6}
\definecolor{persiangreen}{rgb}{0.0, 0.65, 0.58}
\definecolor{lightgray}{gray}{0.93}
\definecolor{midgray}{gray}{0.6}

\hypersetup{
    colorlinks = true,
    linkcolor = darkcerulean,
    anchorcolor = darkcerulean,
    citecolor = darkcerulean,
    filecolor = darkcerulean,
    urlcolor = darkcerulean
}

\usepackage{algorithm}
\usepackage{algpseudocode}

\usepackage{subcaption}
\usepackage{siunitx}

\usepackage{cleveref}

\usepackage{pifont}

\newtheorem{setting}{Setting}[section]
\newtheorem{algconfig}{Algorithm Configuration}[section]

\newcommand{\eqdef}{\ensuremath{\stackrel{\mbox{\upshape\tiny def.}}{=}}}

\newcounter{termcounter}
\renewcommand{\thetermcounter}{\Roman{termcounter}}
\crefname{term}{term}{terms}
\creflabelformat{term}{#2\textup{(#1)}#3}

\makeatletter
\def\term{\@ifnextchar[\term@optarg\term@noarg}
\def\term@optarg[#1]#2{%
  \textup{#1}%
  \def\@currentlabel{#1}%
  \def\cref@currentlabel{[][2147483647][]#1}%
  \cref@label[term]{#2}}
\def\term@noarg#1{%
  \refstepcounter{termcounter}%
  \textup{(\thetermcounter)}%
  \cref@label[term]{#1}}
\makeatother

\usepackage{makecell}
\usepackage{tabularx}
\newcolumntype{L}[1]{>{\raggedright\arraybackslash}p{#1}}
\newcolumntype{C}[1]{>{\centering\arraybackslash}p{#1}}
\newcolumntype{Y}{>{\raggedright\arraybackslash}X}

\usepackage[most]{tcolorbox}
\usepackage{etoolbox} 
\usepackage{xcolor}
\definecolor{faintgray}{RGB}{245,245,245}     
\definecolor{faintborder}{RGB}{230,230,230}   
\definecolor{lightblack}{gray}{0.4}           
\newcounter{question}
\newtcolorbox[auto counter, use counter=question]{question}[1][]{
  enhanced,
  colback=faintgray,
  colframe=faintborder,
  boxrule=0.2pt,
  arc=2mm,
  title=\textcolor{lightblack}{\textbf{Question~\thequestion}},
  fonttitle=\bfseries,
  before upper={\centering\itshape},
  after title={\vspace{0.5ex}},
  boxsep=4pt,
  left=6pt,
  right=6pt,
  top=4pt,
  bottom=4pt,
  #1
}

\usepackage[page]{appendix}

\date{}

\usepackage{amssymb}  
\usepackage{pifont}   
\definecolor{deepjunglegreen}{rgb}{0.0, 0.29, 0.29}
\definecolor{dogwoodrose}{rgb}{0.84, 0.09, 0.41}
\definecolor{sanddune}{rgb}{0.59, 0.44, 0.09}
\definecolor{slategray}{rgb}{0.44, 0.5, 0.56}
\definecolor{burntsienna}{rgb}{0.91, 0.45, 0.32}

\jmlrheading
  {1}
  {2026}
  {1-\pageref{LastPage}}
  {}
  {}
  {}
  {Anastasis Kratsios, Tin Sum Cheng, Bum Jun Kim, and Daniel Roy}

\ShortHeadings
  {Adaptive Recovery and Class-Value Adaptation}
  {Kratsios, Cheng, Kim, and Roy}

\firstpageno{1}
\begin{document}

\title{\textit{When is Fine-Tuning Scalable?}
Certified Recovery from Noisy Data and Linear-Cost Fine-Tuning on Structured Task Shifts}

\author{\name Anastasis Kratsios\thanks{Corresponding author.}
  \email kratsioa@mcmaster.ca
  \\
  \addr
  McMaster University and the Vector Institute
  \\
  Department of Mathematics and Statistics
  \AND
  \name Tin Sum Cheng
  \email tinsum.cheng@unibas.ch
  \\
  \addr
  Universit\"{a}t Basel\\
  Department of Mathematics and Computer Science
  \AND
  \name Bum Jun Kim
  \email bumjun.kim@weblab.t.u-tokyo.ac.jp
  \\
  \addr
  The University of Tokyo\\
  Graduate School of Engineering
  \AND
  \name Daniel Roy
  \email daniel.roy@utoronto.ca
  \\
  \addr
  Google DeepMind, University of Toronto, and the Vector Institute\\
  Department of Statistical Sciences
}


\maketitle

\begin{abstract}
Transformers often transfer across related tasks with little fine-tuning; however, it remains unclear structures transformers can use to guarantee scalable adaptation. We study this question in noisy nonparametric regression for continuous maps $f,g:[0,1]^d\to\mathbb{R}^D$, representing the source and shifted tasks, 
using a single-query, single-head transformer $\mathcal{T}(\cdot)=V\operatorname{Softmax}(-\alpha(\cdot,K))$ with $\ell^{\infty}$-based alignment scores $\alpha$. Rather than assume that gradient-based training reaches an empirical risk minimizer, we give a model-specific two-stage procedure \textbf{certifying} source-task recovery and out-of-distribution (OOD) adaptation.

With high probability, the resulting attention head recovers $f$ to arbitrary accuracy $\varepsilon>0$ in every fixed $L^p([0,1]^d)$ norm, $1\le p<\infty$, from noisy observations. For Lipschitz targets, its rate is minimax-optimal up to logarithmic factors, cf.~\citep{stone1982optimal}, without assuming oracle access to an asymptotic ERM; cf.~\citet{schmidt2020nonparametric}. 
If $D=1$ and $f$ is Pfaffian than the values matrix factors into $AC$ with $A$ scaling at-most linearly $\mathcal{O}(\varepsilon^{-1})$ in the reciprocal approximation error $\varepsilon>0$.  Importantly, our training procedure encodes all our identified (potentially) shared geometric structure from $f$ only into the values matrix factor $C$ and the keys matrix $K$ allowing for rapid access, and all $f$-specific structure only into $A$.

Our main result concerns the OOD phase; where we identify geometric structures, which if shared by the source $f$ and OOD $g$ task; then only only the values matrix factor $A$ needs to be fine-tuned to guarantee an $L^p$-recovery of $g$.  Critically, fine-tuning only $A$ is a convex problem which involving an \textit{at-most linear} number of parameters $\mathcal{O}(\varepsilon^{-1})$; dramatically cutting down on the worst-case scaling of $\mathcal{O}(\varepsilon^{-\mathbf{d}})$ unavoidable in the source-training phase. 
\end{abstract}

\vspace{1em}




\section{Introduction}
\label{s:Introduction}

Why, and when, should fine-tuning on a new task be substantially easier than learning that task from scratch?  The primary objective of this paper is to show that fine-tuning can indeed scale linearly in the reciprocal accuracy $\tfrac{1}{\varepsilon}>0$, even when training on the original source task must scale like $\mathcal{O}\big(\tfrac{1}{\varepsilon^d}\big)$ in the problem dimension $d$ for general Lipschitz tasks.  The key is to learn the right structure from the source task and encode it into the model, so that it can be re-used out-of-distribution (OOD) with only a small fine-tuning overhead, while maintaining a high probability of OOD generalization controlled by the similarity between the source and OOD tasks.

To understand this clearer, we study the recovery and scalable transfer phenomena in nonparametric regression, where finitely many i.i.d.\ noisy training samples $\mathcal{D}^N\eqdef\{(X_n,Y_n)\}_{n=1}^N$ are observed from an unknown uniformly continuous target $f:[0,1]^d\to\mathbb{R}^D$ where
\begin{equation}
\label{eq:non_parametric_regression__def}
Y_n=f(X_n)+\varepsilon_n
\end{equation}
given centred sub-Gaussian noise $\varepsilon_n$ for $n=1,\dots,N$.
By \textbf{recovery}, we refer to the approximation in $L^p([0,1]^d;\mathbb{R}^D)$-norm of $f$ using a transformer whose parameters are determined algorithmically using only $\mathcal{D}^N$.  \textit{Scalable transfer} learning refers to the need to only re-train a small number of parameters to obtain the same guarantee for a \textit{new} training dataset where $f$ is shifted to a ``related'' but different task $g:[0,1]^d\to \mathbb{R}^D$; cf.~\cite{pan2010survey,yosinski2014transferable,donahue2014decaf,girshick2014rich}.  

We consider interpretable toy transformers $\mathcal{T}:\mathbb{R}^d\to \mathbb{R}^D$ with $\ell^{\infty}$-type features consisting of a single attention-head $\mathcal{T}\eqdef \operatorname{Attn}_{\lambda}(\cdot|K,V)$ of the following form\footnote{In the paper's abstract, we abbreviated the distance-based alignment scores of the input query ($x$) to the keys matrix ($K$) as $\alpha(\cdot)\eqdef 
\big(
\|
x
-
K_i
\|_{\infty}
\big)_{i=1}^{d_k}
$.}
\begin{equation}
\label{eq:ENCODER}
            \operatorname{Attn}_{\lambda}(x|K,V)
    \eqdef 
        \overbrace{
            V
        }^{\text{Values}}
        \,
        \operatorname{Softmax}_{\lambda}\biggl(
            \underbrace{
            -
                \Big(
                    \|
                        \overbrace{x}^{\text{Query}}
                        -
                        \overbrace{
                        K_i
                        }^{\text{Keys}}
                    \|_{\infty}
                \Big)_{i=1}^{d_k}
            }_{
            \text{Query-to-key alignment scores}
            }
        \biggr)
\end{equation}
for $d_k,D\in \mathbb{N}_+$, a keys matrix $K\in\mathbb R^{d\times d_k}$,
and a values matrix $V\in\mathbb R^{D\times d_k}$, with the input $x$ acting as a query, where $\operatorname{Softmax}_{\lambda}(z)\eqdef \tfrac{e^{z/\lambda}}{\sum_{i=1}^{d_k}\,e^{z_i/\lambda}}$ is applied componentwise on any $z\in \mathbb{R}^{d_k}$.

We will introduce a \textit{specialized training algorithm} for identifying key $\hat{K}$ and value $\hat{V}$ matrices for~\eqref{eq:ENCODER} using only the training data $\mathcal{D}^N$.  Our main results will offer \textit{certificates} on the recovery and transfer of the resulting simple transformer $\hat{\mathcal{T}}\eqdef 
\operatorname{Attn}_{\lambda}(\cdot|\hat{K},\hat{V})$.  Before overviewing our main contributions, we situate them in the modern landscape.

\subsection{The Algorithmic Recovery Problem: Where's the Approximate ERM?}
\label{s:Introduction__ss:Recovery}
Several results are available which guarantee the recovery of $f$ under oracle-quality data and training algorithms, esp.\ \textit{universal approximation} theorems~\cite{yarotsky2017error,petersen2018optimal,shen2022optimal,kratsios2025kolmogorov} which provide guarantees such as: given $f$ there exist $\hat K$ and $\hat V$ such that 
$\hat{\mathcal{T}}\eqdef 
\operatorname{Attn}_{\lambda}(\cdot|\hat{K},\hat{V})$ satisfies
\begin{equation}
\label{eq:background__univ_approx}
\mathbb{E}_{X}[\|\hat{\mathcal{T}}(X)-f(X)\|]
\le \varepsilon
\end{equation}
for a pre-specified approximation error $\varepsilon>0$ determining the number of parameters in $\hat{\mathcal{T}}$.
\textit{Statistical} results guarantee the recovery of $f$ for ERMs, which we assume are obtainable by a concrete training algorithm; i.e.\
for every failure probability $0<\delta \le 1$ the following holds with probability $1-\delta$ over the draw of $\mathcal{D}^N$:
\begin{equation}
\label{eq:background__statistical}
    \underbrace{
        \mathbb{E}_{X,\varepsilon}[\|\hat{\mathcal{T}}(X)-(f(X)+\varepsilon)\|]
    }_{\text{Population Risk } \mathcal{R}(\hat{\mathcal{T}})}
\le 
    \underbrace{
        \tfrac{1}{N}\,\sum_{n=1}^N\, \|\hat{\mathcal{T}}(X_n)-Y_n\|
    }_{\text{Empirical Risk } \hat{\mathcal{R}}_{\mathcal{D}^N}(\hat{\mathcal{T}})}
    +
    \underbrace{
        \mathcal{O}\big(
            \tfrac{\sqrt{\log(1/\delta)}}{N^c}
        \big)
    }_{\text{Sample Complexity Term}}
\end{equation}
where the exponent $0<c\le \tfrac{1}{2}$ in the ``sample complexity term'' in~\eqref{eq:background__statistical} can be quantified by bounding the local Rademacher complexities~\cite{bartlett2005local}, fat-shattering dimensions~\cite{alon1997scale}, or associated limiting Gaussian processes~\cite{bartl2025we}, of the natural hypothesis class here 
$\mathcal H_{d,d_k,D}
\eqdef
\big\{
    \operatorname{Attn}_{\lambda}(\cdot\mid K,V):
    K\in\mathbb R^{d\times d_k},\
    V\in\mathbb R^{D\times d_k},\
    \max_{i\in[d_k]}\|V_i\|_2\le B
\big\}$
, or using instance-dependent bounds akin to~\cite{bartlett2017spectrally,neyshabur2018a,hou2023instance}\footnote{For the norm loss displayed in~\eqref{eq:background__statistical}, the population minimizer is a conditional geometric median, so identifying it with $f(x)=\mathbb E[Y\mid X=x]$ requires a conditional median-zero assumption; under merely centred noise, the corresponding recovery statement should instead use squared loss.}.  
Moreover, using oracle inequalities as in~\citep[Theorem 2]{schmidt2020nonparametric} and working with excess risk, one can control~\eqref{eq:background__univ_approx} for any approximate empirical risk minimizer; i.e.\ some data-dependent choice of a $\hat{\mathcal{T}}\in \mathcal{H}_{d,d_k,D}$ for which
\begin{equation}
\label{eq:approx_ERM}
\lim\limits_{N\uparrow \infty}\,
    \mathbb{E}_{\mathcal{D}^N}\biggl[
        \Big|
            \hat{\mathcal{R}}_{\mathcal{D}^N}(\hat{\mathcal{T}})
            -
            \inf_{\mathcal{T}\in  \mathcal{H}_{d,d_k,D}}
            \,
            \hat{\mathcal{R}}_{\mathcal{D}^N}(\mathcal{T})
        \Big|
    \biggr]
=0
.
\end{equation}
An a priori decay is, for instance, explicitly assumed in~\citep[Theorem 1]{schmidt2020nonparametric} or is implied by the stronger assumption that $\hat{\mathcal{T}}$ is an \textit{empirical risk minimizer} (ERM)\footnote{I.e.\ an estimator $\hat{\mathcal{T}}\in\mathcal H_{d,d_k,D}$ such that, for every realized dataset $\mathcal D^N$, $\hat{\mathcal R}_{\mathcal D^N}(\hat{\mathcal T})=\inf_{\mathcal T\in\mathcal H_{d,d_k,D}}\hat{\mathcal R}_{\mathcal D^N}(\mathcal T)$.}; cf.~\citep[Corollary 1]{schmidt2020nonparametric}.

The algorithmic identification of an ERM is a mild assumption for classical stochastic programs~\cite{bartl2022monte} such as kernel (ridge/less) regressors, e.g.\ neural tangent kernels (NTKs)~\cite{jacot2018ntk} or random feature models~\cite{gonon2020risk}, due to the convexity of the empirical risk (as a function of those models' parameters), or due to the closed-form expressions for kernel (ridge/less) regressors, allowing for a direct statistical analysis avoiding empirical process arguments~\cite{tsigler2023benign,cheng2023theoretical,barzilai2024generalization,cheng2024comprehensive}.

\paragraph{The Gap: Non-Convexity of the Training Loss}
This, however, is not generally true for deep learning models, such as modern large language models~\cite{achiam2023gpt} or even for our simple (universal) toy model~\eqref{eq:ENCODER}, due to the non-convexity of the training loss $\operatorname{Loss}:\mathbb{R}^{D\times d_k}\times\mathbb{R}^{d\times d_k}\to [0,\infty)$ sending each parameter pair $(V,K)\in 
\mathbb{R}^{D\times d_k}\times\mathbb{R}^{d\times d_k}$ to
\begin{equation}
\label{eq:Training_Loss}
    \operatorname{Loss}(V,K)
    \eqdef 
        \tfrac{1}{N}\sum_{n=1}^N
    \,
    \big\|
            \operatorname{Attn}_{\lambda}(X_n|K,V)-Y_n
    \big\|_2^2
.
\end{equation}
Indeed, convergence guarantees are well developed for \textit{convex} stochastic optimization, including robust gradient methods~\cite{holland2019better,vural2022mirror}, median-of-means procedures~\cite{lugosi2020risk,lugosi2019regularization,bartl2022montecarlo}, and robust-gradient schemes based on auxiliary argmin problems~\cite{holland2019better}. 
In addition, a wealth of \textit{global} convergence results are also available for GD~\cite{polyak1963gradient}, SGD~\cite{nemirovski2009robust,moulines2011nonasymptotic,bach2013nonstrongly}, and Nesterov-type acceleration~\cite{nesterov1983method,lan2012optimal,ghadimi2012optimal}. Adam-type methods require stronger conditions, and vanilla Adam may fail even for convex objectives~\cite{reddi2018convergence,zou2019sufficient,defossez2022simple}. Beyond convexity, GD may converge globally under a Polyak-{\L}ojasiewicz (P{\L}) condition~\cite{polyak1963gradient,louzi2024stochastic}, but such structure need not hold for general non-convex losses such as~\eqref{eq:Training_Loss}.  
Indeed, for the (universal) toy model in~\eqref{eq:ENCODER}, the training objective in~\eqref{eq:Training_Loss} is not convex%
\footnote{See Proposition~\ref{prop:nonconvex_loss} in Appendix~\ref{s:AuxiliatyResults}.} and does not it satisfy the (P{\L}) inequality\footnote{See Proposition~\ref{prop:nonPL_loss} in Appendix~\ref{s:AuxiliatyResults}.} meaning that none of the above convergence results imply recovery of an (approximate) ERM when applied to~\eqref{eq:Training_Loss}.
The \textbf{takeaway is} that unless the trained deep learning model \textit{achieves interpolation} of the training data, it is not clear how to practically apply or theoretically verify condition~\eqref{eq:approx_ERM} in~\citep[Theorem 1]{schmidt2020nonparametric} for deep learning models trained with the standard training algorithms.  

While one such algorithm is currently available for recovering $f$ from noisy training data~\cite{hong2026closedformformulaconsistentlipschitz}, it does not yield a learner that permits scalable fine-tuning under out-of-distribution task shifts.

\subsubsection{First Contribution: An Algorithm for Certified Recovery}
\label{s:Introduction__ss:Recovery___sss:TheGap}

We resolve this gap for our deep learning model without resorting to a linearization of the model class; e.g.\ as for NTKs or random feature models.
To this end, we introduce a two-phase training algorithm (Algorithms~\ref{alg:STEP_1_PURIFY} and~\ref{alg:STEP_2_CLUSTER}) which produces key $K$ and value $V$ matrices using only the training data $\mathcal{D}^N$.
Our first main result (Theorem~\ref{thrm:Main_Lp}) shows that the resulting toy transformer $\hat{\mathcal{T}}\eqdef\operatorname{Attn}_{\lambda}(\cdot|\hat{K},\hat{V})$ satisfies~\eqref{eq:background__univ_approx}, with high probability over the draw of the training data $\mathcal{D}^N$, without exogenously assuming an empirical-optimization condition such as~\eqref{eq:approx_ERM}.  In the Lipschitz setting described below, the resulting rate has the classical nonparametric Lipschitz exponent, up to logarithmic factors.

Briefly, phase $1$ (Algorithm~\ref{alg:STEP_1_PURIFY}) uses local averaging relative to a subset of adaptively chosen anchor points to approximately de-noise the training data.  The resulting new ``virtual training dataset'' consists of approximate paired samples of the form $(X,f(X))$, with the gap converging at a rate determined by the tails of the additive noise $\varepsilon_1,\dots,\varepsilon_N$ implicitly present in the labels $Y_1,\dots,Y_N$.  Phase $2$ (Algorithm~\ref{alg:STEP_2_CLUSTER}) then learns how to efficiently encode the implied structures in the domain of $f$ and its relevant function values jointly into $\hat{V}$ and $\hat{K}$.  These can either be read off or directly loaded into $\operatorname{Attn}(\cdot|\cdot)$ to determine $\hat{\mathcal{T}}$.  A number of other subtle properties related to scalable fine-tuning for efficient transfer learning will now be discussed.

The simplified formulation of our first main result available in Corollary~\ref{cor:Main_Lp__uniform} guarantees that if $f$ is $L$-Lipschitz for $L>0$ and if $X$ is uniformly sampled from $[0,1]^d$, then
\begin{equation}
\label{eq:recovery}
\mathbb{E}_{X\sim\mathbb{P}_X}
\big[
\|
\hat{\mathcal{T}}_{\lambda}(X)-f(X)
\|
\big]
\lesssim
\frac{L\sqrt d}{q}
\end{equation}
holds with probability at least $1-\delta$ over the draw of $\mathcal{D}^N$, provided that $\mathcal{D}^{2N}$ for the half-sample size $N$ (so the total sample size is $2N$)
\begin{equation}
\label{eq:intro_adaptive_sample_complexity}
    N
\gtrsim
\left(
q^d
+
\frac{\sigma^2\,q^{d+2}}{L^2}
\right)
\log
\left(
\frac{
Dq^d
}{
\delta
}
\right),
\end{equation}
where $\sigma$ determines the tails of the centred sub-Gaussian additive noise $\varepsilon$ and $q\in \mathbb{N}_+$ is the resolution parameter determining the recovery error in~\eqref{eq:recovery}.

Algorithm~\ref{alg:STEP_2_CLUSTER} returns a values matrix $\hat{V}$ that encodes the learned ``domain symmetries'' directly into $\mathcal{T}$.  The significance of this structure will now become clear as we turn to our second and main problem: that of \textit{scalable} fine-tuning to new OOD tasks.

\subsection{Structures which imply scalable transfer learning}
\label{s:Introduction__ss:Transfer}
Scalable transfer learning refers to the need to re-train only a ``small number'' of model parameters on data from a new, but related, training task while obtaining guarantees analogous to~\eqref{eq:background__statistical} or~\eqref{eq:background__univ_approx}.  
Here and throughout, ``scalable'' refers only to the number of target-dependent entries in the class-value factor, not to target sample complexity, memory, preprocessing, inference cost, wall-clock training time, or total computation.

One of the main questions is what structures learned from $f$ can be shared with a related task so that such scalable transfer is possible with our toy pre-trained model $\mathcal{T}$ in~\eqref{eq:ENCODER}, where, by \textit{source} we mean that $\hat{\mathcal{T}}$ is trained on the source task using Algorithms~\ref{alg:STEP_1_PURIFY} and~\ref{alg:STEP_2_CLUSTER}, and by \textit{fine-tuning} we mean algorithmically re-training only a small number of its parameters to recover a new ground truth $g:[0,1]^d\to\mathbb{R}^D$ observed through i.i.d.\ noisy paired training samples $(\tilde{X}_1,\tilde{Y}_1),\dots,(\tilde{X}_{\tilde{N}},\tilde{Y}_{\tilde{N}})$, where $\tilde{Y}_n=g(\tilde{X}_n)+\tilde{\varepsilon}_n$ for independent centred sub-Gaussian noise vectors $\{\tilde{\varepsilon}_n\}_{n=1}^{\tilde{N}}$.

Many available results on transfer learning, e.g.~\cite{benDavid2006analysis,maurer2016benefit,ganin2016domain,long2015learning,rojasCarulla2018invariant}, and ours assume that some representation learned from the source task is shared with the fine-tuning/OOD phase.  
Shared-representation theory gives nonlinear transfer guarantees and efficient linear multitask algorithms~\citep{du2021fewshot,tripuraneni2021provable}. Our explicit nonlinear attention construction reuses learned class assignments for compatible-region OOD adaptation through at most $D\kappa$ target-dependent entries.
However, e.g.~\cite{benDavid2006analysis,benDavid2010theory,mansour2009domain}, classical domain-adaptation results often control the target risk using a ``discrepancy'' $\operatorname{disc}(\mu,\mu^{\operatorname{OOD}})$ between the source and target distributions, where $\mu$ is the law of $(X_1,Y_1)$ and $\mu^{\operatorname{OOD}}$ is the law of $(\tilde{X}_1,\tilde{Y}_1)$.  For Lipschitz hypotheses and losses, related distribution-shift terms can, for example, be controlled using Wasserstein distances; cf.~\citep[Corollary 8]{hou2023instance} and~\citep[Theorem 7.6 and Equation (Coupling)]{benitez2024out}.
A typical schematic guarantee is then of the form
\begin{equation}
\label{eq:OOD_Benchmark}
\underbrace{
\mathbb{E}_{\tilde{X},\tilde{\varepsilon}}[\|\hat{\mathcal{T}}(\tilde{X})-(g(\tilde{X})+\tilde{\varepsilon})\|]
}_{\text{OOD Population Risk}}
\le
\tfrac{1}{N}\,\sum_{n=1}^N\, \|\hat{\mathcal{T}}(X_n)-Y_n\|
+
\mathcal{O}\big(
\tfrac{\sqrt{\log(1/\delta)}}{N^c}
\big)
+
\underbrace{
\operatorname{disc}(\mu,\mu^{\operatorname{OOD}})
}_{\text{OOD vs.\ source tasks}}
.
\end{equation}
Critically, a bound of the form~\eqref{eq:OOD_Benchmark} controls noisy OOD prediction risk and does not, by itself, guarantee that the clean recovery error $\mathbb{E}_{\tilde X}[\|\hat{\mathcal T}(\tilde X)-g(\tilde X)\|]$ converges to $0$.

Our main question is whether there are recognizable properties, e.g.\ perhaps present in image data, which: 
\begin{enumerate}
    \item[(1)] \textbf{Scalable Fine-Tuning:} allow for scalable transfer learning; i.e.\ during fine-tuning we only re-train a low-rank factor of $\hat{V}$ involving only $D\kappa$ target-dependent entries, with $\kappa\ll |\hat{V}|_0+|\hat{K}|_0$,
    \item[(2)] \textbf{OOD Recovery:} we can prove recovery on the new OOD task rather than only a bound of the form~\eqref{eq:OOD_Benchmark}.  
\end{enumerate}
Here, our objective is not to identify the most general structure that can be shared with the OOD task, but rather one which allows for (1) and (2).

\subsubsection{Second Contribution: Scalable Fine-Tuning with OOD Recovery}
\label{s:Introduction__ss:Transfer___sss:TheGap}
For our second result, we use a separate non-adaptive fixed-grid construction, consisting of Algorithm~\ref{alg:STEP_1_PURIFY__simplified} together with Algorithm~\ref{alg:STEP_2_CLUSTER}.  This construction allows us to work on a uniform grid of side length $\tfrac{1}{q}$ in $[0,1]^d$, for a scale $q\in \mathbb{N}_+$.  In this case, the keys matrix $\hat{K}$ records the centres of the cubes in this grid, while the factorization $\hat{V}=\hat{A}\hat{C}$ separates the recovered class assignment, encoded by $\hat{C}$, from the corresponding distinct values of $f$, encoded by $\hat{A}$.
For instance, in the first square ($d=2$) of Figure~\ref{fig:shared_symmetries_tasks}, sub-cubes with the same greyscale value share the same class assignment encoded by $\hat{C}\in\mathbb{R}^{\kappa\times d_k}$ ($d_k=q^d$), while the corresponding $\kappa$ distinct greyscale values ($D=1$) are encoded by $\hat{A}\in\mathbb{R}^{D\times\kappa}$.  If there are few distinct values ($\kappa\ll q^d$), e.g.\ $\kappa\in\mathcal{O}(\log(q)^r)$ for some $r>0$, then re-training $\hat{A}$ requires at most $D\kappa$ target-dependent entries to be updated, rather than updating all $Dq^d$ entries of an unrestricted grid-wise values matrix.

\begin{figure}[htp]
    \centering
    \includegraphics[width=1.0\linewidth]{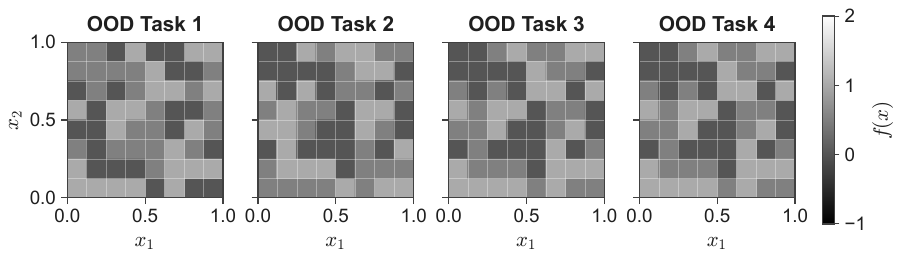}
    \caption{\textbf{Shared symmetries across tasks.}
Each panel depicts a different task on $[0,1]^2$ at a scale of $\tfrac{1}{q}=\tfrac{1}{8}$.  In each panel, the same groups of sub-squares share a common greyscale value, while those greyscale values change across panels.  This illustrates the full-compatibility case of the fixed-grid class-value adaptation mechanism: the class assignment is shared across tasks, while only the corresponding class values change; here $\kappa=4$. The general oracle theorem also allows a supplied compatible family of cells and controls the number of target-dependent class-value entries.}
    \label{fig:shared_symmetries_tasks}
\end{figure}
In this case, we can freeze the keys matrix $\hat{K}$ learned in the source task, which creates a shared representation between the source and OOD tasks; namely $x\mapsto \operatorname{Softmax}_{\lambda}((-\|x-\hat{K}_i\|_{\infty})_{i=1}^{d_k})$, and we only need to re-train the $A$ factor in the values matrix.

Our second main result (Theorem~\ref{thrm:2_Scalable_Fine_Tuning}) 
provides an oracle compatible-region guarantee showing that 
if $f$ was successfully recovered in the first training found source recovery and
if \textit{the source and OOD tasks share the same structured grouping} of the $q^d$ grid cells into $\kappa\ll q^d$ value classes, then only $\hat{A}$ must be re-trained by re-running Algorithm~\ref{alg:STEP_1_PURIFY__simplified}, and one can guarantee, when the new target $g$ is $L_g$-Lipschitz, the recovery bound
\begin{equation}
\label{eq:background__OOD_recovery}
\underbrace{
\mathbb{E}_{\tilde{X}}
\big[
\|\widetilde{\mathcal{T}}_\lambda(\tilde{X})-g(\tilde{X})\|_2
\big]
}_{\text{OOD Recovery}}
\lesssim
\frac{L_g\sqrt{d}}{q}
\end{equation}
with high probability over the training set of the OOD problem.

\paragraph{Genericity of Structure ($\kappa$ vs.\ $q$) and Scaling Laws $\kappa \in \Theta(q^r)$ for $0<r\le 1$:}
Finally, we show that for Lipschitz-regular target functions which have additional combinatorial tameness, i.e.\ continuous piecewise linear (Theorem~\ref{thrm:Lipschitz_PWL_Symmetry_Scaling}) or Pfaffian~\cite{khovanskii1991fewnomials} (Theorem~\ref{thrm:Holder_Pfaffian_Symmetry_Scaling}), cf.\ Figure~\ref{fig:pfaffian_symmetry_examples}, for every fixed $c>0$ and $r\in(0,1]$, under the corresponding schedule $\eta_q=cq^{-r}$ there is a corresponding sequence of class counts $(\kappa_q)_{q=1}^{\infty}$ satisfying
\begin{equation}
\label{eq:scaling_laws}
\kappa_q \in \Theta(q^r)
.
\end{equation}
Importantly,~\eqref{eq:scaling_laws} would yield $\mathcal{O}(Dq^r)$ OOD entries, $\mathcal{O}(Dq)$ for $r=1$, after a vector-valued extension and a source-slack-robust version of Theorem~\ref{thrm:2_Scalable_Fine_Tuning}.

\begin{figure}[htp]
    \centering
    \includegraphics[width=.48\linewidth]{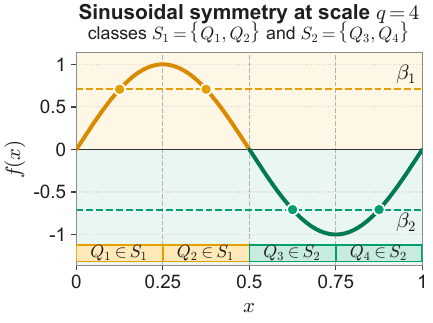}
    ~
    \includegraphics[width=.48\linewidth]{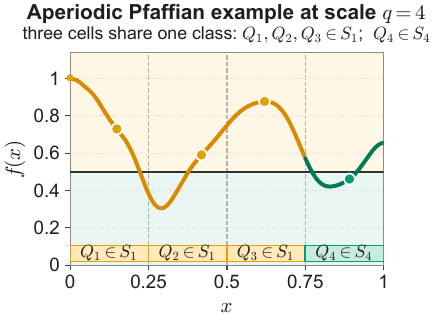}
    \caption{\textbf{Examples of symmetries in Pfaffian functions.}
    The left panel shows a familiar sinusoidal pattern, for which clustering by output value occurs as expected. 
    The right panel illustrates the quantitative conclusion of Theorems~\ref{thrm:Lipschitz_PWL_Symmetry_Scaling} and~\ref{thrm:Holder_Pfaffian_Symmetry_Scaling}: under their stated non-degeneracy conditions (e.g., $\|\nabla f(x)\|>0$ for each $x$) and the Pfaffian format assumptions and the scale $\eta_q=cq^{-r}$, the largest approximate value class has cardinality $\Theta(q^{d-r})$.
    In both panels, each curve segment is coloured according to its assigned cluster. The cluster assignment depends \textit{both} on the cell containing the segment and on the value of the curve \textit{at the midpoint} of the corresponding cube.}
    \label{fig:pfaffian_symmetry_examples}
\end{figure}

\subsection*{Organization of the Paper}

Section~\ref{s:Prelim} introduces the model, notation, and symmetry structure.  Section~\ref{s:The_Algo} presents the two-stage training procedure.  Section~\ref{s:Main_PROBLEM_1__Source} proves source recovery, while Section~\ref{s:Main_PROBLEM_2__OOD} develops scalable OOD fine-tuning under shared approximate symmetries.  Section~\ref{s:Discussion} discusses related work and Section~\ref{s:ConclusionFutureWork} concludes.  The appendices collect technical notation, proofs, and supporting symmetry results.

\section{Preliminaries}
\label{s:Prelim}

We first collect some notation and terminology used in our exposition. For a longer list required for proofs see Supplement~\ref{a:Notation}.  We then formalize and explain our model and introduce our notation of domain symmetry.

\subsection{Notation and Terminology}
\label{s:Prelim__ss:Additional}

We now collect the notation and definitions required to state and explain our main results.  
For any $n\in \mathbb{N}_+$ we write $[n]\eqdef \{m\in \mathbb{N}_+:\,m\le n\}$.

\paragraph{Regular Probability Measures}
We will typically assume that the additive measurement noise is \textit{sub-Gaussian}; that is, for $D\in\mathbb{N}_{+}$, $C\ge 0$, and $\sigma>0$, a Borel probability measure  $\mu$ on $\mathbb{R}^D$ is called $(C,\sigma^{2})$-sub-Gaussian if
\[
    \mu
    \left(
        \left\{
            x\in\mathbb{R}^D
            :
            \|x\|_{2}>t
        \right\}
    \right)
    \le
    C
    \exp
    \left(
        -\frac{t^{2}}{\sigma^{2}}
    \right)
\]
for every $t>0$.  We henceforth write $\bar{\sigma}\eqdef\max\{1,\sigma\}$.

We assume that the \textit{sampling distribution} $\mathbb{P}_X$, i.e.\ the
law of $X_1$, is \textit{Ahlfors regular} on $([0,1]^d,\|\cdot\|_\infty)$. That is, there exist constants $0<c_-\le c_+<\infty$ such that, for every $x\in[0,1]^d$ and every $0<r\le1$,
\begin{equation}
\label{eq:Ahlfors}
    c_-\,r^d
    \le
    \mathbb{P}_X\bigl(Q_d(x,r)\bigr)
    \le
    c_+\,r^d.
\end{equation}

\paragraph{Function Classes}
A modulus of continuity $\omega:[0,\infty)\to [0,\infty)$ is a \textit{strictly} increasing, continuous function satisfying $\omega(0)=0$.  
We call a function $f:(\mathbb{R}^d,\|\cdot\|_{\infty}) \to (\mathbb{R}^D,\|\cdot\|_2)$ $\omega$-uniformly continuous if for every $x_1,x_2\in \mathbb{R}^d$
\[
    \|f(x_1)-f(x_2)\|_2
\le 
    \omega\big(
        \|x_1-x_2\|_{\infty}
    \big)
.
\]
For example, if $\omega(t)=Lt^{\alpha}$ for some $L>0$ and some $0<\alpha\le 1$, then we say that $f$ is $\alpha$-H\"{o}lder with constant $L$ and exponent $\alpha$.  If, moreover, $\alpha=1$, then we say that $f$ is $L$-Lipschitz.

\hfill\\
\noindent
In this paper, a piecewise-linear will mean a function which is composed of finitely many affine pieces which are continuously glued together over regions which are themselves defined by finitely many pieces.  E.g.\ the real-valued sawtooth function $\mathbb{R}\ni x\mapsto 1-|1- x\bmod(2)|$ fails this definition as it is comprised of infinitely many pieces, and $\mathbb{R}^2\ni (x_1,x_2)\mapsto I_{x_1^2+x_2^2\le 1}$ is neither continuous nor does it decompose into affine regions.  Nevertheless, every $\operatorname{ReLU}$ MLP clearly satisfies our definition:
\begin{definition}[Piecewise-Linear Function]
\label{defn:polyhedral_piecewise_linear}
Fix $d,\Pi,Y\in\mathbb{N}_+$.  A continuous function
$
    f:[0,1]^d\to\mathbb{R}
$
is $(\Pi,Y)$-piecewise linear if there are convex polytopes
$
    P_1,\dots,P_\Pi\subseteq[0,1]^d
$
with non-empty interiors such that:
\begin{enumerate}
    \item[(i)] \textbf{Coverage:} $
    [0,1]^d
    =
    \bigcup_{\pi=1}^{\Pi}P_\pi
    $ and $
    \operatorname{int}(P_\pi)
    \cap
    \operatorname{int}(P_{\pi'})
    =
    \varnothing
    $ whenever $\pi\neq\pi'$
    \item[(ii)] \textbf{Polytope Complexity:} 
    For every $\pi\in[\Pi]$, there are
    $
        Y_\pi\in\{0,\dots,Y\}
    $,
    $
        \{u_{\pi,i}\}_{i=1}^{Y_\pi}
        \subset
        \mathbb{R}^d\setminus\{0\}
    $ 
    and
    $
        \{c_{\pi,i}\}_{i=1}^{Y_\pi}
        \subset
        \mathbb{R}
    $
    whose \textit{half-spaces} $
        \big(\{x\in\mathbb{R}^d:
        \langle u_{\pi,r},x\rangle\leq c_{\pi,r}\}
        \big)_{1}^{Y_{\pi}}
    $ satisfy
    \[
        P_\pi
        =
        [0,1]^d
        \cap
        \bigcap_{r=1}^{Y_\pi}
        \left\{
            x\in\mathbb{R}^d:
            \langle u_{\pi,r},x\rangle
            \leq
            c_{\pi,r}
        \right\}
    \]
    \item[(iii)] \textbf{Piecewise Affine:} 
    For every $\pi\in[\Pi]$, there are
    $
        a_\pi\in\mathbb{R}^d
    $
    and
    $
        b_\pi\in\mathbb{R}
    $
    such that
    \[
        f(x)
        =
        \langle a_\pi,x\rangle+b_\pi
    \qquad
        \mbox{ for each } x\in P_\pi
    .
    \]
\end{enumerate}
\end{definition}

\subsection{Tasks and Training Samples}
\label{s:Prelim__ss:Tasks}

For $m\in\mathbb N_+$, we write $[m]\eqdef\{1,\dots,m\}$, and $\#A$ denotes the cardinality of a finite set $A$.

\begin{definition}[Regression Task and Training Set]
\label{defn:regression_training_set}
A regression task is a triple
$
    \tau=(f,\mathbb{P}_X,\mathbb{P}_\varepsilon),
$
where
$
    f:[0,1]^d\to\mathbb R^D
$
is Borel measurable,
$
    \mathbb{P}_X
$
is a Borel probability measure on
$
    [0,1]^d,
$
and
$
    \mathbb{P}_\varepsilon
$
is a centred Borel probability measure on
$
    \mathbb R^D.
$
\end{definition}
Given a regression task $\tau$ and a sample size $N\in \mathbb{N}_+$, a size-$N$ $\tau$-training set is a collection $\mathcal{D}=\{(X_n,Y_n)\}_{n=1}^N$ of random variables for which there are i.i.d.\ random vectors $\{\varepsilon_n\}_{n=1}^N$ with law $\mathbb{P}_{\varepsilon}$ such that for each $n\in [N]$
\begin{equation}
\label{eq:regression_training_equation}
    Y_n=f(X_n)+\varepsilon_n,
\end{equation}
and such that $\{X_n\}_{n=1}^N$ are i.i.d.\ with law $\mathbb{P}_X$ and independent of the $\{\varepsilon_n\}_{n=1}^N$.

\subsubsection{Pilot Set}
We will often assume we have an even number of training points (by at worst discarding one data point) $2N$ and we rewrite $\mathcal{D}=\{(X_n,Y_n)\}_{n=1}^{2N}$ as
\begin{equation}
\label{eq:adaptive_two_half_samples}
    \mathcal{D}
\eqdef
        \mathcal{D}^{(0)}
    \cup
        \mathcal{D}^{(1)}
\mbox{ where }
    \mathcal{D}^{(s)}
\eqdef 
    \big(
        (X_n^{(s)},Y_n^{(s)})
    \big)_{n=1}^N
\quad s=0,1
.
\end{equation}
We call $\mathcal{D}^{(0)}$ the \textit{pilot set}; we will use it to infer the domain in Algorithm~\ref{alg:STEP_1_PURIFY} and we will use $\mathcal{D}^{(1)}$ as the
remaining data in $\mathcal{D}$, which will be used for fitting relative to the pilot set.  
For each $s\in\{0,1\}$, the subfamily
\begin{equation}
\label{eq:adaptive_half_sample}
    \mathcal{D}^{(s)}
\eqdef
    \left(
        (X_n^{(s)},Y_n^{(s)})
    \right)_{n\in[N]}
\end{equation}
is a size-$N$ $\tau$-training set, and the two half-samples are
independent.

\subsection{Interpretation of the Parameters in Our Single Attention-Head Model}
\label{s:Prelim__ss:TransformerSpecific}

We henceforth fix dimensions $d,d_k,D\in\mathbb{N}_+$.  
Recall from~\eqref{eq:ENCODER} that our model is a single attention head whose alignment scores are given by negative $\ell^\infty$-distances between the query and the keys.
For a keys matrix $K\in\mathbb{R}^{d\times d_k}$ and a values matrix $V\in\mathbb{R}^{D\times d_k}$, we write
\begin{equation}
\label{eq:ENCODER_full}
    \mathcal{T}_{\lambda}(x)
    \eqdef
    \operatorname{Attn}_{\lambda}(x\mid K,V)
    \eqdef
    V\,
    \operatorname{Softmax}_{\lambda}\biggl(
        -
        \Big(
            \|x-K_i\|_{\infty}
        \Big)_{i=1}^{d_k}
    \biggr),
    \qquad
    x\in\mathbb{R}^d.
\end{equation}
Here, $K_i$ denotes the $i$th column of $K$.  For every temperature $\lambda>0$, the softmax function is defined, for $M\in\mathbb{N}_+$, $z\in\mathbb{R}^M$, and $i\in[M]$, by
$
    \operatorname{Softmax}_{\lambda}(z)_i
    \eqdef
    \tfrac{\exp(z_i/\lambda)}
    {\sum_{j=1}^M\exp(z_j/\lambda)}
$.
In our constructions, the values matrix is typically further factorized in a way that separates the values assigned to different classes from the assignment of keys to those classes.
More precisely, for $\kappa\le d_k$, we write
\begin{equation}
\label{eq:value_factorization}
\begin{aligned}
    V
& \eqdef
    AC_{\widehat c}
\\
    A
&\eqdef
    (v_1,\dots,v_\kappa)
    \in\mathbb{R}^{D\times\kappa}
\\
    C_{\widehat c}
& \eqdef
    \big(
        I_{\{\widehat c(i)=k\}}
    \big)_{k\in[\kappa],\,i\in[d_k]}
    \in\{0,1\}^{\kappa\times d_k}
.
\end{aligned}
\end{equation}
Thus, the $i$th column of $C_{\widehat c}$ records the class assigned to the key $K_i$, while the columns of $A$ record the corresponding class values; equivalently,
$
    V_i=v_{\widehat c(i)}
$.
The factorization of $V$ in~\eqref{eq:value_factorization} will be particularly useful for transfer learning: the OOD constructions below keep the keys and class assignment fixed while updating only the class-value factor $A$.  We note that this factorization is structurally reminiscent of parameter-efficient adaptation methods such as LoRA~\cite{hu2022lora} and QLoRA~\citep{dettmers2023qlora}, although the constructions are different: here $V$ is represented directly as $AC_{\widehat c}$ rather than through an additive low-rank update, and $C_{\widehat c}$ explicitly records the assignment of keys to classes.

We next explain the role of each component in~\eqref{eq:ENCODER_full} and~\eqref{eq:value_factorization}, guided by Figures~\ref{fig:model_interpretation__KeysandQueries} and~\ref{fig:model_interpretation__VALUES}.

\subsubsection{Interpretation of the keys (K) and query (x)}
\label{s:Prelim__ss:TransformerSpecific___sss:Kx}
Figure~\ref{fig:model_interpretation__KeysandQueries} explains the role of the keys ($K$) and query ($x$) for an example with inputs in $[0,1]^2$ with $d_k=4$ and $\kappa=3$.  Here, the keys $K=(K_1,K_2,K_3,K_4)$ serve as reference points in the input domain $[0,1]^2$.
Each query $x$ is compared with these reference points through the $\ell^\infty$-distance-based attention weights
$\operatorname{Softmax}_{\lambda}((-\|x-K_i\|_{\infty})_{i=1}^4)$.
These weights record the relative proximity of $x$ to the four keys, becoming increasingly concentrated on the nearest key as $\lambda\downarrow0$ and increasingly diffuse as $\lambda$ increases.
Thus, the keys $K_1,\dots,K_{d_k}$ identify reference locations in the input domain, while the query $x$ is aligned with nearby keys according to their $\ell^\infty$ distance.

\begin{figure}[H]
    \centering

    \begin{subfigure}[t]{0.40\linewidth}
        \centering
        \vspace{0pt}
        \begin{minipage}[c][0.34\textwidth][c]{\linewidth}
            \centering
            \scalebox{1.35}{$
            \displaystyle
            \overbrace{
                \left[
                \begin{array}{
                    >{\columncolor{KoneBG}}c
                    >{\columncolor{KtwoBG}}c
                    >{\columncolor{KthreeBG}}c
                    >{\columncolor{KtwoBG}}c
                }
                \textcolor{KoneText}{\tfrac14}
                &
                \textcolor{KtwoText}{\tfrac34}
                &
                \textcolor{KthreeText}{\tfrac34}
                &
                \textcolor{KtwoText}{\tfrac34}
                \\[0.5em]
                \textcolor{KoneText}{\tfrac12}
                &
                \textcolor{KtwoText}{\tfrac78}
                &
                \textcolor{KthreeText}{\tfrac58}
                &
                \textcolor{KtwoText}{\tfrac14}
                \end{array}
                \right]
            }^{K}
            $}
        \end{minipage}
        \caption{Keys matrix.}
    \end{subfigure}
    \hfill
    \begin{subfigure}[t]{0.46\linewidth}
        \centering
        \vspace{0pt}
        \begin{minipage}[c][0.34\textwidth][c]{\linewidth}
            \centering
            \includegraphics[
                width=0.98\linewidth,
                height=0.34\textwidth,
                keepaspectratio
            ]{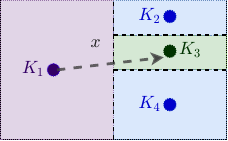}
        \end{minipage}
        \caption{Geometric role of the keys and query.}
    \end{subfigure}

    \caption{\textbf{Interpretation of the keys ($K$) and query ($x$).}}
    \label{fig:model_interpretation__KeysandQueries}
\end{figure}

\subsubsection{Interpretation of the value matrix and its factors}
\label{s:Prelim__ss:TransformerSpecific___sss:value_matrix_and_factors}

The values matrix
$
{\setlength{\arraycolsep}{2pt}
\scriptstyle
V=
\left[
\begin{array}{
>{\columncolor{KoneBG}}c
>{\columncolor{KtwoBG}}c
>{\columncolor{KthreeBG}}c
>{\columncolor{KtwoBG}}c
}
\textcolor{KoneText}{A_1}
&
\textcolor{KtwoText}{A_2}
&
\textcolor{KthreeText}{A_3}
&
\textcolor{KtwoText}{A_2}
\end{array}
\right]
}
$
in Figure~\ref{fig:model_interpretation__VALUES} assigns the $i$th value $V_i$ to the corresponding key $K_i$.
As the query $x$ moves between the input keys, the model output $\mathcal{T}_{\lambda}(x)$ interpolates between their associated values; the grey arrow in subfigure~(b) illustrates the corresponding motion in the output space.

\begin{figure}[H]
    \centering

    \begin{subfigure}[t]{0.38\linewidth}
        \centering
        \vspace{0pt}
        \begin{minipage}[c][0.24\textheight][c]{\linewidth}
            \centering
            \scalebox{1.15}{$
            \displaystyle
            \overbrace{
            \underbrace{
            \left[
            \begin{array}{
            >{\columncolor{KoneBG}}c
            >{\columncolor{KtwoBG}}c
            >{\columncolor{KthreeBG}}c
            }
            \textcolor{KoneText}{A_1}
            &
            \textcolor{KtwoText}{A_2}
            &
            \textcolor{KthreeText}{A_3}
            \end{array}
            \right]
            }_{A}
            \times
            \underbrace{
            \left[
            \begin{array}{
            >{\columncolor{KoneBG}}c
            >{\columncolor{KtwoBG}}c
            >{\columncolor{KthreeBG}}c
            >{\columncolor{KtwoBG}}c
            }
            \textcolor{KoneText}{1}
            &
            0
            &
            0
            &
            0
            \\
            0
            &
            \textcolor{KtwoText}{1}
            &
            0
            &
            \textcolor{KtwoText}{1}
            \\
            0
            &
            0
            &
            \textcolor{KthreeText}{1}
            &
            0
            \end{array}
            \right]
            }_{C}
            }^{V}
            $}
        \end{minipage}
        \subcaption{Low-rank factorization $V=AC$.}
    \end{subfigure}
    \hfill
    \begin{subfigure}[t]{0.38\linewidth}
        \centering
        \vspace{0pt}
        \begin{minipage}[c][0.24\textheight][c]{\linewidth}
            \centering
            \includegraphics[
                width=\linewidth,
                height=0.24\textheight,
                keepaspectratio
            ]{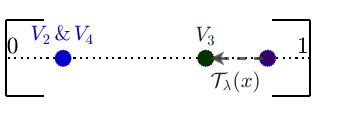}
        \end{minipage}
        \subcaption{Values in the output space.}
        \label{fig:model_interpretation_values}
    \end{subfigure}

    \caption{\textbf{Interpretation of the factors $A$ and $C$ of the values matrix $V$ in~\eqref{eq:ENCODER}.}}
    \label{fig:model_interpretation__VALUES}
\end{figure}

We typically use a low-rank factorization $V=AC$ (of rank $3$ in this example), whose two factors admit the following interpretation:
The columns of \textit{factor $C$} are canonical one-hot vectors from $\{e_i\}_{i=1}^{3}$ and record which keys share the same value, producing the three clusters $\{K_1\}$, $\{K_2,K_4\}$, and $\{K_3\}$.
The columns of \textit{factor $A$} record the three distinct values assigned to these clusters.

\subsubsection{Interpretation of the temperature ($\lambda$) and the softmax function}
\label{s:Prelim__ss:TransformerSpecific___sss:temperature_and_softmax}

The temperature $\lambda$ controls how sharply the attention mechanism selects among the keys.
For every $\lambda>0$, the softmax weights are non-negative and sum to one, so
$\operatorname{Attn}_{\lambda}(x\mid K,V)$ is a convex combination of the values
$V_1,\dots,V_{d_k}\in\mathbb{R}^D$, with larger weight assigned to values whose corresponding keys are closer to the query $x$ in $\ell^\infty$ distance.
As $\lambda\downarrow0$, these weights become increasingly concentrated on the nearest key; when the nearest key $K_{i_x}$ is unique, the model output converges to its paired value $V_{i_x}$.
Conversely, as $\lambda$ increases, the attention weights become more diffuse and the model averages more broadly across the values associated with nearby keys.

\subsection{Domain Symmetries}
\label{s:Prelim__ss:Symmetries}

Fix $d,D,q\in\mathbb N_+$.  
We define the closed cube centred at $x\in\mathbb R^d$ of radius $\rho\ge0$ by
\begin{equation}
\label{eq:closed_linf_cube}
        Q_d(x,\rho)
    \eqdef
        \left\{
            z\in\mathbb R^d:
            \|z-x\|_\infty\le\rho
        \right\}.
\end{equation}
We distinguish a special set of \textit{open} cubes $\mathcal{Q}_{q:d}$ making up a uniform grid of $[0,1]^d$ with side length $\tfrac{1}{q}$ as these will play a role in our OOD results in \S~\ref{s:Symmetries__ss:OOD}.  For brevity, we reserve special notation, aligned to~\cite{triebel2008function}, where for every
$
    i
=
    (i_1,\dots,i_d)
\in
    \{0,\dots,q-1\}^d,
$
we define
\begin{equation}
\label{eq:cubes}
        Q_i
    \eqdef
        \prod_{\ell=1}^d
        \left(
            \frac{i_\ell}{q},
            \frac{i_\ell+1}{q}
        \right),
        \qquad
        \mathcal{Q}_{q:d}
    \eqdef
        \left\{
            Q_i:
            i\in\{0,\dots,q-1\}^d
        \right\}.
\end{equation}
Thus, $Q_i\in\mathcal{Q}_{q:d}$ is an open cube, whereas $Q_d(x,\rho)$ is a closed cube.
The midpoint of $Q_i$ is
$
        \bar{Q}_i
    \eqdef
        \left(
            \frac{2i_\ell+1}{2q}
        \right)_{\ell=1}^d.
$  We will refer to the induced ``good set'' $\mathcal{G}$ given by
\begin{equation}
\label{eq:good_set_scale_q}
        \mathcal{G}
    \eqdef
        \bigcup_{Q\in\mathcal{Q}_{q:d}}Q
    =
           [0,1]^d
        \setminus
            \bigcup_{Q\in\mathcal{Q}_{q:d}}
            \partial Q
.
\end{equation}

We consider partitions of $\mathcal{Q}_{q:d}$ into $\kappa\in\{1,\dots,q^d\}$ ``clusters.''  We use the term \textit{cluster} in the sense common in TCS~\cite{mendel2013ultrametric}, namely, as a collection of points regarded as similar.  Here, similarity is determined by the structure of the target function, as illustrated in Figure~\ref{fig:pfaffian_symmetry_examples} and formalized in Section~\ref{s:Symmetries__ss:OOD}.

\begin{definition}[Symmetry $\mathcal{S}$ at scale $q$]
\label{defn:symmetry}
Fix $q\in \mathbb{N}_+$ and $\kappa \in [q^d]$.  
A $\kappa$-symmetry at scale $q$ is a partition
$
    \mathcal{S}
=
    \{\mathcal{S}_k\}_{k=1}^{\kappa}
$
of $\mathcal{Q}_{q:d}$ into non-empty sets.  \end{definition}
Given a $\kappa$-symmetry at scale $q$, $\mathcal{S}$, for every $Q\in\mathcal{Q}_{q:d}$, we associate the unique \textit{class membership} $c:\mathcal{Q}_{q:d}\to [\kappa]$ determined for any $k\in [\kappa]$ and any $Q\in\mathcal{S}_k$ by
\begin{equation}
\label{eq:class_membership}
    c_{\mathcal{S}}(Q)= k
.
\end{equation}
For brevity, we abbreviate $c=c_{\mathcal{S}}$ when contextually clear.
\begin{definition}[$(\mathcal{S},\eta)$-symmetric function]
\label{defn:symmetric_functions}
Fix $q,d,D\in \mathbb{N}_+$, $\kappa \in [q^d]$, $0\le \eta$, and a $\kappa$-symmetry at scale $q$, $\mathcal{S}$.  
A function $f:[0,1]^d\to\mathbb R^D$ is $(\mathcal{S},\eta)$-symmetric if there are $\beta_1,\dots,\beta_\kappa\in\mathbb R^D$ such that
\begin{equation}
\label{eq:symmetric_labels}
    \big\|
        f(\bar{Q})
    -
        \beta_{c(Q)}
    \big\|_{\infty}
    \le 
        \eta
\quad\mbox{ for every }
\,
Q\in\mathcal{Q}_{q:d}
\end{equation}
We say that $f$ is $(\kappa,\eta)$-symmetric at scale $q$ if it is $(\mathcal{S},\eta)$-symmetric for one such $\mathcal{S}$.
\end{definition}
While this new condition may seem like a generalized sinusoidal condition, it is actually much more common than one would expect at first.  Indeed, we will now show that controlling the regularity (thus classical constructive approximation-theoretic information; cf.~\cite{lorentz1996constructive}) and the model-theoretic or combinatorial structure of the target function (cf.~\cite{khovanskii1991fewnomials}) implies such symmetries quantitatively.

\subsubsection{Regularity + Combinatorial Tameness Yields Symmetries}
\label{s:Prelim__ss:Symmetries___sss:DeepExamples}
We have the following relationship between approximate ($\eta>0$) symmetries and H\"older regularity.  Interestingly, the H\"older exponent $\alpha$ does not enter the resulting bound on the number of symmetry classes $\kappa$; only the H\"older constant $L\geq0$ and the slack parameter $\eta>0$ appear.  
Intuitively, this is because the bound depends only on the total oscillation of $f$ over the unit cube, whose diameter is one regardless of $\alpha$.  Indeed, one can show that $\kappa \in \mathcal{O}(\tfrac{L}{\eta})$ but it is never larger than $q^d$ in this way (see Proposition~\ref{prop:holder_implies_symmetry} in Section~\ref{s:AuxiliatyResults__ss:Symmetries}).  

Regularity alone, however, does not control the number of cubes belonging to each class $\mathcal{S}_k$.  Such control can instead arise from additional combinatorial structure.  We first provide a \textit{scaling law} for one such structure, namely, that of piecewise affine functions (cf.\ Definition~\ref{defn:polyhedral_piecewise_linear}) as they are very intuitive to explain and this class subsumes every $\operatorname{ReLU}$ MLP ``teacher''.

\begin{theorem}[Scaling Law: Symmetries of Lipschitz-Regular PWL Functions]
\label{thrm:Lipschitz_PWL_Symmetry_Scaling}
Fix $d,\Pi,Y\in\mathbb{N}_+$ and $L\geq0$.  Let
$
    f:[0,1]^d\to[0,1]
$
be an $L$-Lipschitz $(\Pi,Y)$-polyhedral piecewise-linear function, with
$
    f(x)=\langle a_\pi,x\rangle+b_\pi
$
on $P_\pi$, and suppose that
\begin{equation}
\label{eq:PWL_nonflatness__theoremversion}
    \delta_{\nabla}
    \eqdef
    \min_{\pi\in[\Pi]}
    \|a_\pi\|_2
>
    0
.
\end{equation}
Then, for every $c>0$ and $r\in(0,1]$, the sequence of slack parameters
$
    (\eta_q)_{q=1}^{\infty}
$
given by
$
    \eta_q\eqdef c\,q^{-r}
$
has the following property: for every $q\in\mathbb{N}_+$, $f$ is $(\mathcal{S}^{(q)},\eta_q)$-symmetric for some $\kappa_q$-symmetry
$
    \mathcal{S}^{(q)}
    =
    \{\mathcal{S}^{(q)}_k\}_{k=1}^{\kappa_q}
$
at scale $q$, and%
\footnote{The implicit constants in~\eqref{eq:PWL_matching_scaling} depend only on $d,\Pi,L,c,$ and $\delta_{\nabla}$.}%
\begin{equation}
\label{eq:PWL_matching_scaling}
    \kappa_q
    \asymp
    q^r
\qquad
\mbox{and}
\qquad
    \max_{k\in[\kappa_q]}
    \#\mathcal{S}^{(q)}_k
    \asymp
    q^{d-r}
.
\end{equation}
\end{theorem}
\begin{proof}
See Section~\ref{s:AuxiliatyResults__ss:Symmetries___ssss:PWL}.
\end{proof}
Motivated by our piecewise linear result, we directly generalize the minimal-slope condition in~\eqref{eq:PWL_nonflatness__theoremversion} (to ~\eqref{eq:Pfaff__MaxSlope} below) and consider a natural smooth variant of this example.  Here, the combinatorial structure of the number of piecewise linear parts  and the number of affine hyperplanes defining the convex polytopal regions is replaced by the model-theoretic notion of Pfaffianity; cf.~\cite{khovanskii1991fewnomials}.  It is interesting to note that the Pfaffianity is satisfied by many modern deep-learning models, cf.~\cite{d2024vc,kratsios2026ominimal}, and its connections to learning theory through polynomial bounds on the VC dimension and pseudo dimension are known, either directly~\cite{karpinski1997polynomial} or through connections to o-minimality~\cite{guingona2013vc} (see~\cite{speissegger1997pfaffian} for the required link).  In our setting, the volume estimates of~\cite{lion1998volumes}, sharpened quantitatively by~\cite{lezeau2026tubular}, will play the central role allowing us to deduce the following scaling law.   
We refer the reader to 
\cite{kratsios2026ominimal}
for detailed definitions and examples of Pfaffian chains, degrees, and functions.
\begin{theorem}[Scaling Law: Symmetries  of H\"{o}lder-Regular Pfaffian Functions]
\label{thrm:Holder_Pfaffian_Symmetry_Scaling}
Fix $d\in\mathbb{N}_+$ and $\alpha\in(0,1]$.  Let
$
    f:\mathbb{R}^d\to[0,1]
$
be Pfaffian of degree $\beta_f\in\mathbb{N}_+$ in a Pfaffian chain of length
$\Lambda\in\mathbb{N}_+$ and degree $\Delta_{\mathrm{pf}}\in\mathbb{N}_+$, such that $f_{[0,1]^d}$ is $\alpha$-H\"older with constant $L\ge 0$ and 
\begin{equation}
\label{eq:Pfaff__MaxSlope}
    \delta_{\nabla}\eqdef\min_{x\in[0,1]^d}
    \|\nabla f(x)\|_2
>
    0.
\end{equation}
Then, for any $c>0$, $r\in (0,1]$, the sequence of slack parameters $(\eta_q)_{q=1}^{\infty}$ given by $\eta_q\eqdef c\,q^{-r}$ is such that:
$f$ is $(\mathcal{S}^{(q)},\eta_q)$-symmetric for some $\kappa_q$-symmetry
$
    \mathcal{S}^{(q)}
    =
    \{\mathcal{S}^{(q)}_k\}_{k=1}^{\kappa_q}
$
at scale $q$ and%
\footnote{The implicit constants in~\eqref{eq:Holder_Pfaffian_matching_scaling} depend only on $d,\Delta_{\mathrm{pf}},\beta_f,\Lambda,L,c,$
and $\delta_{\nabla}$.}%
\begin{equation}
\label{eq:Holder_Pfaffian_matching_scaling}
    \kappa_q
    \asymp
    q^r
\mbox{ and }
    \max_{k\in[\kappa_q]}
    \#\mathcal{S}^{(q)}_k
    \asymp
    q^{d-r}.
\end{equation}
\end{theorem}
\begin{proof}
See Section~\ref{s:AuxiliatyResults__ss:Symmetries___ssss:Pfaff}.
\end{proof}

Now that we are convinced that these symmetry relations are not only common, but we understand how their scaling laws with the parameter $q$ given the scaling slack on the order of $\Theta(q^{-r})$ for any $0<r\le 1$ which we would like, we may now analyze this structure feeds into our learning algorithm and helps us with recovery.

\section{The Algorithm}
\label{s:The_Algo}

We now describe the two subroutines and their respective roles in the analysis.  The first approximately denoises and compresses the training data, while the second constructs the matrices $\hat K$, $\hat A$, and $\hat C$, and hence $\hat V=\hat A\hat C$.  Relative to the learned keys $\hat K$, the factors $\hat A$ and $\hat C$ encode the structure of $f$.  This encoding will allow for scalable fine-tuning of $\mathcal{T}$ by only needing to re-train $\hat{A}$ (which can be very small) on structurally similar new OOD tasks.

\subsection{Step 1 - De-noise Training Data}
\label{s:The_Algo__ss:Denoise}

The first step of our procedure, Algorithm~\ref{alg:STEP_1_PURIFY}\footnote{The success probability of Algorithm~\ref{alg:STEP_1_PURIFY} below is recorded in Lemma~\ref{lem:Convergence_Samples}.}, is illustrated in Figure~\ref{fig:ALGO_Step_1} in the \textit{special case} of a uniform grid.  Briefly, the algorithm first learns $K$ representative queries (i.e.\ $x$ values), $\bar{Q}_1,\dots,\bar{Q}_K$, and then denoises the training labels relative to these representative queries.  The resulting ``virtual dataset'' consists of pairs
$
    (\bar{Q}_k,\bar Y_{k,\rho})_{k=1}^K
$
such that
$
    \bar Y_{k,\rho}
    \approx
    f(\bar{Q}_k)
$.  
None of the learned representatives $\bar{Q}_1,\dots,\bar{Q}_K$ need belong to the original set of training queries $\{X_n\}_{n=1}^N$.  Thus, even though $K<N$ typically, Algorithm~\ref{alg:STEP_1_PURIFY} is not a compression scheme in the learning-theoretic sense; cf.~\cite{FloydWarmuth1995,MoranYehudayoff2016}.

\begin{figure}[H]
    \centering
    \begin{minipage}[t]{0.35\textwidth}
        \centering
        \includegraphics[width=.8\linewidth]{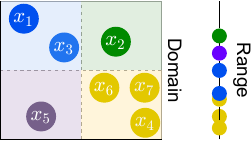}
        \caption{Training dataset: $x$ values are scattered and the $y$ values are noisy.}
        \label{fig:algorithm1_before}
    \end{minipage}
    \hspace{5em}
    \begin{minipage}[t]{0.35\textwidth}
        \centering
        \includegraphics[width=.8\linewidth]{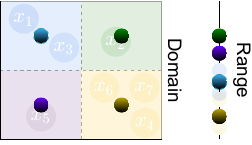}
        \caption{Virtual dataset: $x$-values are maximally representative and $y$ values are almost noiseless.}
        \label{fig:algorithm1_after}
    \end{minipage}
\caption{Step 1 - Pre-Processing: Constructing a Noiseless \textit{Virtual} Dataset via Algorithm~\ref{alg:STEP_1_PURIFY}.  The cells are learned in a data-driven manner and need not lie on a grid.}
\label{fig:ALGO_Step_1}
\end{figure}

In what follows, for fixed $K,N\in \mathbb{N}_+$ and $\delta_{\operatorname{sep}}>0$, we define the set $
    \mathfrak{Z}_{K,\delta_{\mathrm{sep}}}
    \eqdef
    \left\{
        (z_1,\dots,z_K)\in([0,1]^d)^K:
        \|z_i-z_j\|_\infty
        \ge
        \delta_{\mathrm{sep}}
        \text{ whenever }i\neq j
    \right\}
$.  
We assume that $\mathfrak{Z}_{K,\delta_{\mathrm{sep}}}\neq\emptyset$ whenever the constrained centre-selection problem in Algorithm~\ref{alg:STEP_1_PURIFY} is used.

\begin{algorithm}[H]
\caption{\textbf{Denoise and Compress Training Dataset}.}
\label{alg:STEP_1_PURIFY}
\begin{algorithmic}
\Require Dataset
$
    \mathcal{D}=\{(X_n,Y_n)\}_{n=1}^{2N}
$,
separation level
$
    \delta_{\mathrm{sep}}>0
$,
number of centres
$
    \mathrm{K}\in\mathbb N_+
$,
and localization radius
$
    \rho>0.
$

\State
$
    \{
        (x_n^{(s)},y_n^{(s)})
    \}_{(s,n)\in\{0,1\}\times[N]}
\gets
    \{(X_n,Y_n)\}_{n=1}^{2N}
$ 
\Comment{Split Dataset}

\State
$
    \mathcal{D}^\star
\gets
    \emptyset
$ 
\Comment{Initialize Virtual Dataset}

\State
$
    (\bar{Q}_1,\dots,\bar{Q}_{\mathrm{K}})
    \in
    \underset{
        (z_1,\dots,z_{\mathrm{K}})
        \in
        \mathfrak{Z}_{\mathrm{K},\delta_{\mathrm{sep}}}
    }{\operatorname{argmin}}
    \,
    \underset{1\le n\le N}{\max}\,
    \underset{1\le k\le \mathrm{K}}{\min}\,
    \|x_n^{(0)}-z_k\|_\infty
$
\footnotemark
\Comment{Get Landmarks}

\For{$k\in[\mathrm{K}]$}
    \State
    $
    \mathcal{Q}_k
    \eqdef 
    \big\{
        x\in\mathbb{R}^d:\,
        \|x-\bar{Q}_k\|_\infty 
        =
        \underset{j\in[\mathrm{K}]}{\min}\,
        \|x-\bar{Q}_j\|_\infty
    \big\}
    \setminus
    \bigcup_{i<k}\mathcal{Q}_i
    $
    \Comment{Partition $\mathbb{R}^d$ into cells}

    \State
    $
    \mathcal{Q}_{k,\rho}
    \eqdef
        \mathcal{Q}_k
        \cap
        Q_d(\bar{Q}_k,\rho)
    $
    \Comment{Localize each cell around its centre}
\EndFor
\For{$k\in[\mathrm{K}]$}

    \State
    $
        I_{k,\rho}
    \gets
        \{
            n\in[N]:
            x_n^{(1)}\in\mathcal{Q}_{k,\rho}
        \}
    $

    \If{$I_{k,\rho}\neq\emptyset$}

        \State
        $
            \bar Y_{k,\rho}
        \gets
            \frac{1}{N_{k,\rho}}
            \sum_{n\in I_{k,\rho}}
            y_n^{(1)}
        $
        where $
            N_{k,\rho}
        \gets
            \#I_{k,\rho}.
        $
        \Comment{Locally average labels}

        \State
        $
            \mathcal{D}^\star
        \gets
            \mathcal{D}^\star
            \cup
            \{(\bar{Q}_k,\bar Y_{k,\rho})\}.
        $
        \Comment{Add to virtual dataset}

    \EndIf

\EndFor

\State
\Return
$
    \mathcal{D}^\star
$

\end{algorithmic}
\end{algorithm}
\footnotetext{%
I.e.\ we have fixed a Borel measurable map (i.e.\ a tie-breaking rule)
$
    \mathcal{S}:
    \mathbb R^{N\times d}
    \to
    \mathfrak Z_{\mathrm{K},\delta_{\mathrm{sep}}}
$
such that
$
    \max_{1\le n\le N}
    \min_{1\le k\le \mathrm{K}}
    \|x_n-\mathcal{S}(x_1,\dots,x_N)_k\|_\infty
=
    \operatorname*{inf}_{
        (z_1,\dots,z_{\mathrm{K}})
        \in
        \mathfrak{Z}_{\mathrm{K},\delta_{\mathrm{sep}}}
    }
    \max_{1\le n\le N}
    \min_{1\le k\le \mathrm{K}}
    \|x_n-z_k\|_\infty
$, and we abbreviate
$
    (\bar{Q}_1,\dots,\bar{Q}_K)
    \eqdef
    \mathcal{S}(x_1^{(0)},\dots,x_N^{(0)})
$.  
Its existence is proven in Lemma~\ref{lem:measurability_check}. 
}

\subsection{Step 2 - Cluster and Encode into Attention Parameters}
\label{s:The_Algo__ss:ClusterEncode}

Algorithm~\ref{alg:STEP_2_CLUSTER} is schematically illustrated in Figure~\ref{fig:ALGO_Step_2} for the special case of a uniform grid.  In a nutshell, the ``virtual'' data points $(\bar{Q}_1,\bar Y_{1,\rho}),\dots,(\bar{Q}_K,\bar Y_{K,\rho})$ in $\mathcal D^\star$ are matched, or assigned the same colour, when their denoised labels are sufficiently close.  These matched points form ``clusters,'' whose memberships are recorded in the factor $\hat C$, while $\hat A$ records one representative value for each cluster, yielding $\hat V=\hat A\hat C$.  Finally, the keys matrix $\hat K$ simply collects the cell centres $\bar{Q}_1,\dots,\bar{Q}_K$ as its columns.

\begin{figure}[H]
    \centering
    \begin{minipage}[t]{0.35\textwidth}
        \centering
        \includegraphics[width=.8\linewidth]{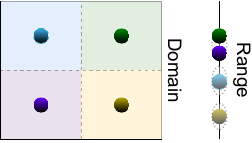}
        \caption{Virtual dataset $\mathcal{D}^{\star}$ before: many input--output pairs contain redundant information.}
        \label{fig:algorithm2_before}
    \end{minipage}
    \hspace{5em}
    \begin{minipage}[t]{0.35\textwidth}
        \centering
        \includegraphics[width=.8\linewidth]{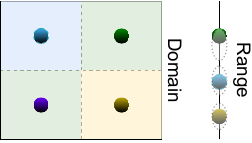}
        \caption{Clusters learned from virtual data: points in $\mathcal{D}^{\star}$ are matched according to their (approximately denoised) label.}
        \label{fig:algorithm2_after}
    \end{minipage}
\caption{Step 2 - Source: Domain Learning via Algorithm~\ref{alg:STEP_2_CLUSTER}.}
\label{fig:ALGO_Step_2}
\end{figure}

\begin{algorithm}[H]
\caption{
    \textbf{Cluster and Initialize Tied-Values Attention}.
}
\label{alg:STEP_2_CLUSTER}
\begin{algorithmic}

\Require Denoised dataset
$
    \mathcal{D}^\star
\eqdef
    \{(x_n,y_n)\}_{n=1}^{M}
\subseteq
    \mathbb R^d
    \times
    \mathbb R^D,
$
and separation parameter
$
    \Delta>0.
$

\State
$
    \widehat{K}
\gets
    (x_n)_{n=1}^{M}
\in
    \mathbb R^{d\times M}
$
\Comment{Construct the keys matrix}

\State
$
    \hat{\kappa}
\gets
    0
$
\Comment{Initialize the number of recovered classes}

\While{$\mathcal{D}^\star\neq\emptyset$}

    \State
    $
        (\mathbf x,\mathbf y)
    \gets
        \operatorname{Unif}(\mathcal{D}^\star)
    $
    \Comment{Choose an unclustered representative}

    \State
    $
        \hat{\kappa}
    \gets
        \hat{\kappa}+1
    $

    \State
    $
        \mathcal{C}_{\hat{\kappa}}
    \gets
        \left\{
            (x_n,y_n)\in\mathcal{D}^\star:
            \|y_n-\mathbf y\|_2
            \le
            \frac{\Delta}{8}
        \right\}
    $
    \Comment{Form the corresponding output-value cluster}

    \State
    $
        \widehat{A}_{\hat{\kappa}}
    \gets
        \mathbf y
    $
    \Comment{Retain the representative output value}

    \For{$(x_n,y_n)\in\mathcal{C}_{\hat{\kappa}}$}

        \State
        $
            \widehat{c}(n)
        \gets
            \hat{\kappa}
        $
        \Comment{Record the recovered class of the \(n\)-th key}

    \EndFor

    \State
    $
        \mathcal{D}^\star
    \gets
        \mathcal{D}^\star
        \setminus
        \mathcal{C}_{\hat{\kappa}}
    $
    \Comment{Remove the clustered observations}

\EndWhile

\State
$
    \mathcal{C}
\eqdef
    \{\mathcal{C}_j\}_{j=1}^{\hat{\kappa}}
$
\Comment{Record the recovered clusters}

\State
$
    \widehat{A}
\gets
    (
        \widehat{A}_1,
        \dots,
        \widehat{A}_{\hat{\kappa}}
    )
\in
    \mathbb R^{D\times\hat{\kappa}}
$
\Comment{Construct the class-value matrix}

\State
$
    C_{\widehat{c}}
\gets
    \left(
        I_{\{\widehat{c}(n)=j\}}
    \right)_{
        \substack{
            j\in[\hat{\kappa}]\\
            n\in[M]
        }
    }
\in
    \{0,1\}^{\hat{\kappa}\times M}
$
\Comment{Construct the class-assignment matrix}

\State
$
    \widehat{V}
\gets
    \widehat{A} C_{\widehat{c}}
\in
    \mathbb R^{D\times M}
$
\Comment{Construct the effective values matrix}

\State
\Return
$
    \mathcal{C},
    \widehat{K},
    \widehat{V},
    \widehat{A},
    C_{\widehat{c}}
    .
$

\end{algorithmic}
\end{algorithm}

\section{Result 1: Guaranteed Algorithmic Recovery}
\label{s:Main_PROBLEM_1__Source}
We first fix the source task and its regular-grid structure, and then present the two algorithmic stages used in the recovery results.

\begin{setting}[Source Setting]
\label{s:TheSetting_pt1}
Let $d,D,q,\kappa,N\in\mathbb{N}_{+}$ with $\kappa\le q^{d}$, and let
$C\ge 0$ and $\sigma>0$.
Let
$
    \mathcal{S}
    =
    \{\mathcal{S}_{k}\}_{k=1}^{\kappa}
$ 
be a $\kappa$-symmetry at scale $q$.  
Fix a concave, continuous, and strictly increasing modulus of continuity $\omega$.
Let $
    \tau
    =
    (f,\mathbb{P}_{X},\mathbb{P}_{\varepsilon})
$ be a task such that:
\begin{enumerate}
    \item[(i)] \textbf{Regular Sampling}: $\mathbb{P}_X$ is Ahlfors regular and supported on $[0,1]^d$,
    \item[(ii)] \textbf{Regular Noise:} $\mathbb{P}_{\varepsilon}$ is a centred $(C,\sigma^{2})$-sub-Gaussian probability measure,
    \item[(iii)] \textbf{Regular Target:} $f:[0,1]^{d}\to\mathbb{R}^{D}$ is $\omega$-uniformly continuous\footnote{That is,
    $\|f(x)-f(\widetilde{x})\|_{2}
    \le
    \omega(\|x-\widetilde{x}\|_{\infty})$
    for every $x,\widetilde{x}\in[0,1]^{d}$.}~%
    \item[(iv)] \textbf{Structured Target:}  $f$ is $\kappa$-symmetric at scale $q$ with symmetry $\mathcal{S}$.  
\end{enumerate}
Fix a separation parameter $\Delta>0$ and if $\kappa\ge 2$, then assume that
\begin{equation}
\label{eq:pretraining_separation}
    \min_{
        \substack{
            j,k\in[\kappa]\\
            j\neq k
        }
    }
    \|
        \beta_{j}-\beta_{k}
    \|_{2}
    \ge
    \Delta
\end{equation}
where $\beta_{1},\dots,\beta_{\kappa}\in\mathbb{R}^{D}$ are such that $
    f(\bar{Q})
    =
    \beta_{k}
$ for each $Q\in\mathcal{S}_{k}$.
\end{setting}

\subsection{Our Main Algorithmic Recovery Result}

We now showcase our main recovery result for the source tasks.  We now operate under the configurations of Algorithm~\ref{alg:STEP_1_PURIFY} and~\ref{alg:STEP_2_CLUSTER}.
\begin{algconfig}[Adaptive recovery]
\label{algconfig:adaptive_recovery}
Let $\mathcal{D}^{\star}$ be the output of Algorithm~\ref{alg:STEP_1_PURIFY} run on $\{(X_n,Y_n)\}_{n=1}^{2N}$ with $\mathrm{K}=(2q)^d$, $\delta_{\mathrm{sep}}=\tfrac{1}{2q}$, and $0<\rho<\tfrac{1}{4q}$; then let $\hat{A},\hat{C},\hat{K}$ denote the outputs of Algorithm~\ref{alg:STEP_2_CLUSTER} run on $\mathcal{D}^{\star}$ with separation input $8\vartheta$, where $0<\vartheta\le\omega(\tfrac{\sqrt d}{q})$. Write $\hat{V}\eqdef \hat{A}\hat{C}$. 
\end{algconfig}

\begin{theorem}[Certified Recovery from Noisy Samples]
\label{thrm:Main_Lp}
Under Setting~\ref{s:TheSetting_pt1} and Algorithm configuration~\ref{algconfig:adaptive_recovery}, fix $0<\delta \le 1$, fix $1\le p<\infty$, and $
    M_{\delta}
\eqdef
    \left\lceil
        \frac{
            4D\sigma_C^2
        }{
            \omega(\rho)^2
        }
        \log
        \left(
            \frac{
                8DK
            }{
                \delta
            }
        \right)
    \right\rceil
$.   
If $\rho\le\tfrac{\sqrt d}{q}$ then there exists a $\lambda_0>0$ such that if 
\begin{equation}
\label{eq:adaptive_hard_source_exact_sample_bound}
    N
    \ge
    \left\lceil
        \max
        \left\{
            \frac{1}{p_{\operatorname{cov}}}
            \log
            \left(
                \frac{2(4q)^d}{\delta}
            \right),
            \frac{2M_{\delta}}{\underline{p}_\rho},
            \frac{8}{\underline{p}_\rho}
            \log
            \left(
                \frac{4K}{\delta}
            \right)
        \right\}
    \right\rceil,
\end{equation}
then for every $0<\lambda\le \lambda_0$
$
        \hat{\mathcal{T}}_\lambda(x)
    \eqdef
        \operatorname{Attn}_\lambda
        \bigl(
            x\mid\widehat K,\hat{V}
        \bigr)
$ satisfies the following
\begin{equation}
\label{eq:adaptive_source_finite_temperature_rate}
    \mathbb{E}_{X\sim\mathbb{P}_X}
    \left[
        \left\|
            \hat{\mathcal{T}}_\lambda(X)
            -
            f(X)
        \right\|_2^p
    \right]
\le
    \left(
        5
        \omega\left(\frac{\sqrt d}{q}\right)
    \right)^p
\end{equation}
 with probability at least $1-\delta$.
\end{theorem}

\begin{proof}
See Section~\ref{s:Approx}.
\end{proof}

Observe that, if $f$ in Setting~\ref{s:TheSetting_pt1} is $0\le L_f$-Lipschitz then we only need to pick $\vartheta\le \tfrac{L_f\sqrt{d}}{q}$ in Theorem~\ref{thrm:Main_Lp}.  In so doing we obtain:
\begin{corollary}[Certified Recovery from Noisy Samples - Lipschitz Case]
\label{cor:Main_Lp__uniform}
Under Setting~\ref{s:TheSetting_pt1} and Algorithm configuration~\ref{algconfig:adaptive_recovery}, fix $0<\delta \le 1$, and let $f$ be $L_f$-Lipschitz for $L_f>0$. There is a $c>0$ such that if $\rho\le\tfrac{\sqrt d}{q}$ then there is a $\lambda_0>0$ such that if 
\begin{equation}
\label{eq:uniform_adaptive_source_sample_bound}
    N
    \ge
    \frac{c}{\rho^d}
    \left(
        1+
        \frac{
            \sigma^2
        }{
            L_f^2\rho^2
        }
    \right)
    \log
    \left(
        \frac{
            8D(4q)^d
        }{
            \delta
        }
    \right).
\end{equation}
then for every $0<\lambda\le \lambda_0$
$
        \hat{\mathcal{T}}_\lambda(x)
    \eqdef
        \operatorname{Attn}_\lambda
        \bigl(
            x\mid\widehat K,\hat{V}
        \bigr)
$ satisfies, with probability at least $1-\delta$,
\begin{equation}
\label{eq:adaptive_source_lipschitz_rate}
\mathbb{E}_{X\sim\mathbb{P}_X}
    \left[
        \left\|
            \hat{\mathcal{T}}_\lambda(X)
            -
            f(X)
        \right\|_2
    \right]
\le
    \frac{
    5L_f
    \sqrt d
    }{q}
.
\end{equation}
\end{corollary}
\begin{proof}
Direct consequence of Theorem~\ref{thrm:Main_Lp} with $p=1$ and $\omega(t)=L_f t$ for $t\ge 0$.
\end{proof}

\subsection{Discussion}
Before moving on, let us discuss a few subtle but key points of Corollary~\ref{cor:Main_Lp__uniform}.
\paragraph{Minimax Rates (Up to Logarithmic Factors)}
For Lipschitz targets, choosing $\rho\asymp q^{-1}$ in Theorem~\ref{thrm:Main_Lp} yields
$
\|\widehat{\mathcal T}_\lambda-f\|_{L^p(\mathbb P_X)}
\lesssim
\bigl(\log(N/\delta)/N\bigr)^{1/(d+2)}
$
for every fixed $1\le p<\infty$, which is minimax-optimal up to a logarithmic factor, while the $L^\infty$ variant in Theorem~\ref{thrm:Main_Linfty} matches the classical minimax sup-norm rate~\citep{stone1982optimal}.  

\paragraph{Do not assume oracle access to an approximate ERM}
Unlike~\citet{schmidt2020nonparametric}, whose rate theorem requires a prescribed bound on the expected empirical-risk gap to the global minimum and whose exact-ERM guarantee is a corollary, Theorem~\ref{thrm:Main_Lp} applies directly to the explicitly constructed output of Algorithms~\ref{alg:STEP_1_PURIFY} and~\ref{alg:STEP_2_CLUSTER}, without any exact- or approximate-ERM assumption for the supervised loss~\eqref{eq:Training_Loss}; the geometric centre-selection problem in Algorithm~\ref{alg:STEP_1_PURIFY} is nevertheless assumed to be solved exactly.

\subsection{Illustrative Toy Ablations}
\label{s:Ablation__Algo1}
We use brief teacher-student ablations to illustrate the behaviour of the two algorithmic stages.  These experiments are diagnostics and do not verify the hypotheses or conclusions of the theorems.

\paragraph{Algorithm $1$:}
Our aim is to isolate the effect of Algorithm~\ref{alg:STEP_1_PURIFY},
rather than to compare neural architectures.
Accordingly, all experiments use the same one-hidden-layer sigmoidal student,
the same optimizer and training budget, and identical initial parameters
within each repetition.
For each trial, we draw
$
    N=10^{4}
$
inputs uniformly from the Euclidean unit sphere and generate noisy labels
\[
    Y_n
    =
    f(X_n)
    +
    \sigma \varepsilon_n,
    \qquad
    \varepsilon_n
    \sim
    \mathcal{N}(0,1),
\]
where $f$ is a randomly initialized one-hidden-layer sigmoidal teacher.
We compare three training sets:
\textit{(A)} the original noisy sample;
\textit{(B)} the virtual sample produced by the greedy implementation of
Algorithm~\ref{alg:STEP_1_PURIFY}, using $K=10^{3}$ centres;
and \textit{(C)} the full virtual sample augmented with uniformly sampled raw
observations until its cardinality again equals $N$.  
The purpose of (C) is to test whether any performance gap is due to the reduced sample size rather than to the higher quality of the denoised training data produced by Algorithm~\ref{alg:STEP_1_PURIFY}.
We report both
\[
    \text{$f$-Recovery}
    \eqdef
    \mathbb{E}
    \big[
        |\widehat f(X)-f(X)|^{2}
    \big]
\qquad\mbox{ and }\qquad
    \text{Test Error}
    \eqdef
    \mathbb{E}
    \big[
        |\widehat f(X)-Y|^{2}
    \big]
\]
the latter of which simply records the test, or out-of-sample, prediction of the noisy labels $Y$ (not the ground-truth function $f$ whose recovery is our objective).
The first ablation varies
$
    d\in\{1,2,5,10,50\}
$, cf.\ Figure~\ref{fig:Algo1_dim},
while the remaining ablations fix $d=5$ and vary either $\sigma\in\{0,0.1,0.25,0.5,0.75,1,5\}$ or the sample size $N\in \{10^2,10^3,10^4\}$; cf.\ Figure~\ref{fig:Algo1_sigma_N}.
Each curve reports the mean over three independent trials, with
one-standard-deviation bands.

\begin{figure}[H]
    \centering

    \includegraphics[width=0.75\linewidth]{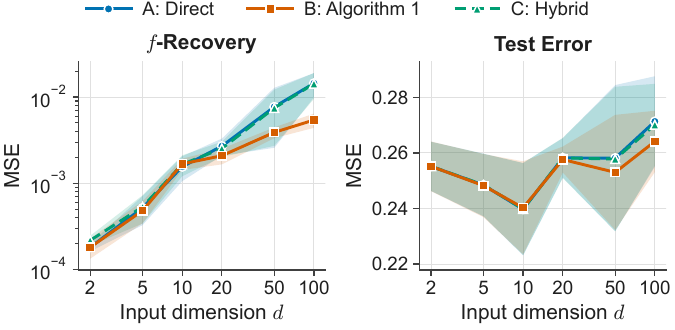}

    \caption{\textbf{Dimension Ablation.} As the problem dimension $d$ increases, the recovery performance of the student MLP deteriorates when it is trained either on the raw noisy samples~(A) or on the hybrid samples~(C).  In contrast, the student MLP trained on the data processed by Algorithm~\ref{alg:STEP_1_PURIFY}~(B) exhibits substantially more reliable recovery.
    \hfill\\
    \noindent 
    The error measured against the noisy responses $Y$, rather than the  latent signal $f(X)$, deteriorates at a comparable rate as $d$ increases, although the hybrid method (C) performs slightly better.
    }
    \label{fig:Algo1_dim}
\end{figure}

\begin{figure}[t]
    \centering

    \begin{minipage}[t]{0.49\textwidth}
        \centering
        \includegraphics[width=\linewidth]{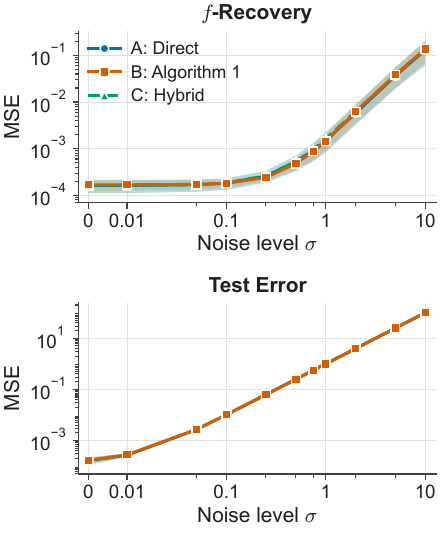}
        \caption{Noise level ablation ($\sigma$).}
    \end{minipage}
    \hfill
    \begin{minipage}[t]{0.49\textwidth}
        \centering
        \includegraphics[width=.625\linewidth]{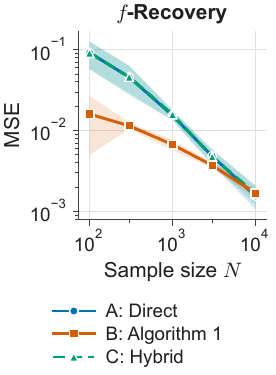}
        \caption{Sample size ablation ($N$).}
    \end{minipage}

    \caption{
        \textbf{Effect of noise level ($\sigma$) and sample size ($N$).}
        Across both ablations, the student MLP trained on the dataset produced by Algorithm~\ref{alg:STEP_1_PURIFY} (B) consistently achieves the best $f$-recovery.
        The advantage is most pronounced when the noise level is small and when the sample size is limited.
        In these regimes, Algorithm~\ref{alg:STEP_1_PURIFY} improves the signal-to-noise ratio by averaging nearby observations while retaining enough local geometric information for the student to learn the teacher reliably.
        As $\sigma$ grows, the intrinsic difficulty of the recovery problem increases for all methods.  As $N$ becomes large, the direct learner already has enough data to average out noise on its own, so the additional gain from preprocessing becomes smaller.
    }
    \label{fig:Algo1_sigma_N}
\end{figure}

\paragraph{Algorithm $2$:}
This ablation asks whether Algorithm~$2$ is an effective surrogate for Adam after purification, and how both methods compare with the standard black-box practice of training directly on raw data. We compare three versions of the same single-head $\ell^{\infty}$-attention model.

The first baseline, denoted by \textit{Adam (Raw)}, is trained directly on the original noisy sample. It uses all $N$ raw inputs as fixed keys and learns one unrestricted scalar value per observation, initialized at the corresponding noisy label, using $300$ Adam updates.

For the remaining two procedures, we first apply Algorithm~\ref{alg:STEP_1_PURIFY}, using
$
    K
    =
    \operatorname{round}(N/20)
$
centres, to obtain the denoised dataset
$
    \mathcal{D}^{\star}
    =
    \{(x_j^{\star},y_j^{\star})\}_{j=1}^{M}
$.
Both procedures use the same denoised inputs as fixed keys.
The \textit{Adam (Denoised)} baseline, denoted by \textit{Adam} in the figures,
learns one unrestricted scalar value per key, initialized at the denoised
labels, using $300$ full-batch optimization epochs.
Algorithm~$2$ instead clusters the denoised labels and ties the values of
all keys assigned to the same recovered class; hence, it requires no
gradient-based optimization.

The first ablation fixes $N=10^{4}$ and varies
$
    d
    \in
    \{2,5,10,50,100\}
$.
The second fixes $d=10$ and varies
$
    N
    \in
    \{10^{2},10^{3},10^{4},10^{5}\}
$.
Each curve reports the mean over ten independent repetitions, together with one-standard-deviation bands. For the sample-size ablation, we additionally report the runtime of each
value-learning or value-construction procedure. For the two denoised-data methods, these runtimes exclude the common
Algorithm~\ref{alg:STEP_1_PURIFY} preprocessing and attention-matrix
construction.

\begin{figure}[bth!]
    \centering

    \includegraphics[width=0.7\linewidth]{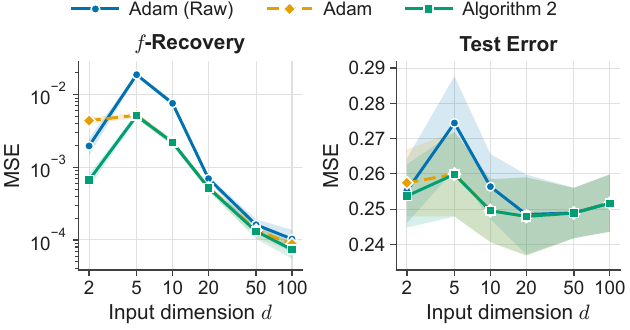}

    \caption{\textbf{Dimension ($d$) Ablation.} Algorithm~$2$ closely matches Adam on the denoised data and attains the lowest mean $f$-recovery error at every tested dimension. Adam trained directly on the raw sample generally recovers $f$ less accurately, with the clearest separation occurring at the intermediate dimensions $d=5$ and $d=10$.  At the larger dimensions, the three recovery errors become closer, although Algorithm~$2$ remains marginally best.
\hfill\\
\noindent
    By contrast, the test error against the noisy responses remains nearly unchanged across the two denoised-data methods and largely plateaus as the dimension increases. Thus, Algorithm~$2$ reproduces the performance of the optimized denoised model, while the gap to Adam (Raw) reflects the additional recovery benefit provided by Algorithm~\ref{alg:STEP_1_PURIFY}.
    }
    \label{fig:Algo2_dim}
\end{figure}

\begin{figure}[H]
    \centering

    \includegraphics[width=1.0\linewidth]{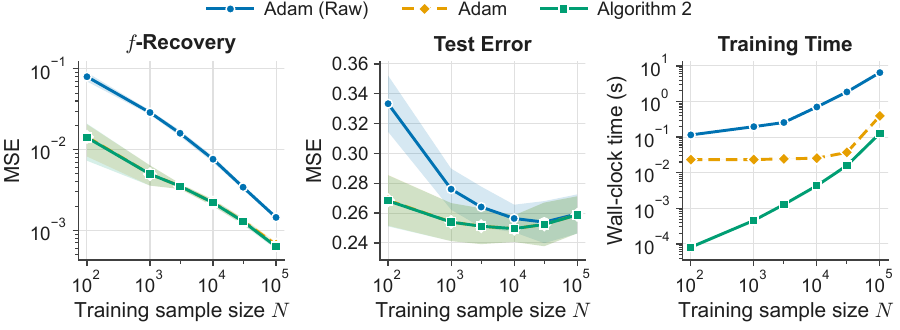}

    \caption{\textbf{Sample Size ($N$) Ablation.}
        Algorithm~$2$ tracks Adam on the denoised data throughout the tested sample-size range and achieves a slightly lower mean $f$-recovery error in every experiment. Both denoised-data methods substantially outperform Adam (Raw), particularly for small and moderate sample sizes, while the gap narrows as $N$ becomes large. Within these experiments, Algorithm~$2$ remains a close surrogate for Adam on the denoised data, while purification improves recovery relative to direct training on the raw sample.
    }
    \label{fig:Algo2_N}
\end{figure}
In Figure~\ref{fig:Algo2_N}, note that the $f$-recovery error decreases steadily with $N$, whereas the test error against the noisy responses rapidly approaches a plateau. The increasing recovery of the latent function is much more pronounced than the corresponding change in noisy-label prediction.  Algorithm~$2$ achieves this performance with substantially lower post-preprocessing runtime than either Adam baseline, with the computational advantage becoming increasingly pronounced as $N$ grows.

\section{The Scalable Fine-Tuning for OOD Transfer Learning Problem}
\label{s:Main_PROBLEM_2__OOD}

In this section, we run the following fixed-grid variant of Algorithm~\ref{alg:STEP_1_PURIFY}.
Throughout this section, both the source and target tasks are denoised on the prescribed regular grid $\mathcal{Q}_{q:d}$ using Algorithm~\ref{alg:STEP_1_PURIFY__simplified}.  Thus, the key associated with an active cell $Q$ is its fixed midpoint $\bar{Q}$.

\begin{algorithm}[H]
\caption{\textbf{\underline{Non-Adaptively} Denoise and Compress}.}
\label{alg:STEP_1_PURIFY__simplified}
\begin{algorithmic}
\Require Dataset $\mathcal{D}\eqdef\{(x_n,y_n)\}_{n=1}^{N}$, grid $\mathcal{Q}$, and localization radius $\rho>0$.
\State $\mathcal{D}^\star\gets\emptyset$
\For{$Q\in\mathcal{Q}$}
    \State $Q_\rho\gets Q\cap Q_d(\bar{Q},\rho)$
    \State $\mathcal{D}_{Q,\rho}\gets\{(x_n,y_n)\in\mathcal{D}:x_n\in Q_\rho\}$
    \If{$\mathcal{D}_{Q,\rho}\neq\emptyset$}
        \State $N_{Q,\rho}\gets\#\mathcal{D}_{Q,\rho}$
        \State $\bar Y_{Q,\rho}\gets N_{Q,\rho}^{-1}\sum_{(x_n,y_n)\in\mathcal{D}_{Q,\rho}}y_n$
        \State $\mathcal{D}^\star\gets\mathcal{D}^\star\cup\{(\bar{Q},\bar Y_{Q,\rho})\}$
    \EndIf
\EndFor
\State \Return $\mathcal{D}^\star$
\end{algorithmic}
\end{algorithm}
Note that, for the \textit{non-adaptive} version of Algorithm~\ref{alg:STEP_1_PURIFY}, we do not need the pilot set $\mathcal{D}^{(0)}$ which acted as a data-driven analogue to the grid $\{\bar{\mathcal{Q}}\}_{Q\in \mathcal{Q}_{q:d}}$.  In this setting, we perform Algorithm~\ref{alg:STEP_1_PURIFY__simplified} first, and subsequently run Algorithm~\ref{alg:STEP_2_CLUSTER}.  It is worth noting that several optimal approximation theorems either implicitly~\cite{yarotsky2018optimal,shen2022optimal} or explicitly~\cite{hong2024bridging} use such uniform grids in $[0,1]^d$.

\subsection{Shared Symmetries}
\label{s:Symmetries__ss:OOD}
The next setting records the source and target partitions, their sampling
laws, and the support conditions needed by the fixed-grid OOD construction.
We now extend and refine  Setting~\ref{s:TheSetting_pt1} to also include the OOD task.

\begin{setting}[OOD Setting]
\label{s:TheSetting_OOD}
Let $d,D,q,\kappa,\widetilde\kappa,N\in\mathbb N_+$ with $\kappa,\widetilde\kappa\le q^d$, let $C,\widetilde C\ge0$, $\sigma,\widetilde\sigma>0$, and $\widetilde\eta\ge0$. Let $\mathcal S=\{\mathcal S_k\}_{k=1}^{\kappa}$ and $\widetilde{\mathcal S}=\{\widetilde{\mathcal S}_j\}_{j=1}^{\widetilde\kappa}$ be, respectively, a $\kappa$-symmetry and a $\widetilde\kappa$-symmetry at scale $q$, with class-assignment maps $c:\mathcal{Q}_{q:d}\to[\kappa]$ and $\widetilde c:\mathcal{Q}_{q:d}\to[\widetilde\kappa]$. Fix concave, continuous, and strictly increasing moduli of continuity $\omega$ and $\widetilde\omega$.

Let $\tau=(f,\mathbb{P}_X,\mathbb{P}_\varepsilon)$ and $\widetilde\tau=(\widetilde f,\widetilde{\mathbb{P}}_X,\widetilde{\mathbb{P}}_\varepsilon)$ be regression tasks such that
\begin{enumerate}
    \item[(i)] \textbf{Regular Sampling:} $\mathbb{P}_X$ and $\widetilde{\mathbb{P}}_X$ are $d$-Ahlfors regular and supported on $[0,1]^d$;

    \item[(ii)] \textbf{Regular Noise:} $\mathbb{P}_\varepsilon$ (resp.\ $\widetilde{\mathbb{P}}_\varepsilon$) is centred and $(C,\sigma^2)$-(resp.\ $(\widetilde C,\widetilde\sigma^2)$)-sub-Gaussian;

    \item[(iii)] \textbf{Regular Targets:} $f,\tilde{f}:[0,1]^d\to\mathbb R^D$ are $\omega$- (resp.\ $\widetilde\omega$)-uniformly continuous;

    \item[(iv)] \textbf{Structured Source:} there are $\beta_1,\dots,\beta_\kappa\in\mathbb R^D$ such that
    $
        f(\bar{Q})=\beta_{c(Q)}
    $
    for every $Q\in\mathcal{Q}_{q:d}$;

    \item[(v)] \textbf{Structured Target:} there are $\widetilde\beta_1,\dots,\widetilde\beta_{\widetilde\kappa}\in\mathbb R^D$ such that
    \begin{equation}
    \label{eq:approx_ood_midpoint_condition}
        \|\widetilde f(\bar{Q})-\widetilde\beta_{\widetilde c(Q)}\|_\infty
        \le\widetilde\eta
        \qquad
        \text{for every }Q\in\mathcal{Q}_{q:d}.
    \end{equation}
\end{enumerate}
Fix $\Delta>0$; if $\kappa\ge2$, assume that $
    \underset{\substack{j,k\in[\kappa]\\j\neq k}}{\min}
    \,
    \|\beta_j-\beta_k\|_2
    \ge\Delta$.
\end{setting}
Finally, let $\mathcal{D}_{\operatorname{src}}$ be a size-$N$ $\tau$-training set.

We compare two tasks through the agreement of their grid-cell class
assignments, as illustrated in Figure~\ref{fig:shared_symmetries_tasks},
using the regression-task and training-set conventions fixed in
Definition~\ref{defn:regression_training_set}.


\begin{definition}[Task Similarity Ratio]
\label{def:SimilarityRatio}
Let $q,d,\kappa,\widetilde\kappa\in \mathbb{N}_+$ with $\kappa,\widetilde\kappa\le q^d$.    
Let $\tau=(f,\mathbb{P}_X,\mathbb{P}_{\varepsilon})$ and $\tilde{\tau}=(\tilde{f},\tilde{\mathbb{P}}_X,\tilde{\mathbb{P}}_{\varepsilon})$ be regression tasks. Let $\mathcal{S}$ be the source $\kappa$-symmetry and $\widetilde{\mathcal{S}}$ the target $\widetilde\kappa$-symmetry at scale $q$, with respective class-assignment maps $c:\mathcal{Q}_{q:d}\to[\kappa]$ and $\widetilde c:\mathcal{Q}_{q:d}\to[\widetilde\kappa]$.
We define the \textit{task similarity ratio} between $\mathcal{S}$ and $\tilde{\mathcal{S}}$ by
\begin{equation}
\label{eq:task_partition_similarity}
    \operatorname{Sim}_q
    \big(
        \widetilde{\mathcal{S}}
        \mid
        \mathcal{S}
    \big)
\eqdef
    \tfrac{
        \max_{\pi:[\kappa]\to[\widetilde\kappa]}
        \#
        \left\{
            Q\in\mathcal{Q}_{q:d}:
            \widetilde c(Q)
            =
            \pi\big(c(Q)\big)
        \right\}
    }{q^d}
\end{equation}
where the maximum in~\eqref{eq:task_partition_similarity} ranges over all maps $\pi:[\kappa]\to [\tilde{\kappa}]$.  
For $0\le\alpha\le1$, we say that $\tilde{\tau}$ is $\alpha$-similar to $\tau$ at scale $q$, and we write $\tilde{\tau}\sim_{\alpha}\tau$, if
\begin{equation}
\label{eq:alpha_task_similarity}
    \operatorname{Sim}_q
    \big(
        \widetilde{\mathcal{S}}
        \mid
        \mathcal{S}
    \big)
\ge
    \alpha
.
\end{equation}
We write $\pi_\star:[\kappa]\to[\widetilde\kappa]$ for a map attaining the maximum in~\eqref{eq:task_partition_similarity}.
Define
$
    \mathcal Q_{\pi_\star}
    \eqdef
    \{Q\in\mathcal Q_{q:d}:\widetilde c(Q)=\pi_\star(c(Q))\}
$
and
$
    \mathcal G_{\pi_\star}
    \eqdef
    \bigcup_{Q\in\mathcal Q_{\pi_\star}}Q
$.
\end{definition}


\subsection{Oracle Class-Value OOD Fine-Tuning}
We operate under the configurations of Algorithms~\ref{alg:STEP_1_PURIFY__simplified} and~\ref{alg:STEP_2_CLUSTER}.
\begin{algconfig}[Fixed Grid Configuration]
\label{algconfig:fixed_grid}
For every $Q\in\mathcal{Q}_{q:d}$, define
$
    \widetilde r_Q
    \eqdef
    \sup
    \left(
        \left\{
            \|x-\bar{Q}\|_\infty:
            x\in Q\cap\operatorname{supp}(\tilde{\mathbb{P}}_X)
        \right\}
        \cup\{0\}
    \right)
$
and
$
    \widetilde r^\star
    \eqdef
    \max_{Q\in\mathcal{Q}_{q:d}}
    \widetilde r_Q.
$
Set
$
    \widetilde\xi\eqdef\sqrt D\,\widetilde\eta
$,
$
    \widetilde w\eqdef\widetilde\omega(\widetilde r^\star)
$,
and
$
    \widetilde W_q
    \eqdef
    \widetilde\omega\left(\frac{\sqrt d}{q}\right).
$

Let $\mathcal{D}^{\star}$ be the output of Algorithm~\ref{alg:STEP_1_PURIFY__simplified} run on $\mathcal{D}_{\operatorname{src}}$ with grid $\mathcal{Q}_{q:d}$ and radius $\rho$,
and let $\widehat K$, $\widehat A$, $C_{\widehat c}$, $\hat{V}=\widehat A C_{\widehat c}$, and $\hat{\kappa}$ denote the outputs of Algorithm~\ref{alg:STEP_2_CLUSTER} run on $\mathcal{D}^{\star}$ with separation input $\Delta$.
\end{algconfig}
\paragraph{Technical notation for the second theorem}
Our OOD theorem uses the following additional notation.  We first record which learned group each grid cell $Q$ belongs to.  Namely, define
$
    \widehat c^{\,\circ}:
    \mathcal{Q}_{q:d}
    \to
    \{0\}\cup[\hat{\kappa}]
$
by
\begin{equation}
\label{eq:global_recovered_source_assignment}
    \widehat c^{\,\circ}(Q)
    \eqdef
    \begin{cases}
        \widehat c(m),
        &
        \widehat K_{:m}=\bar Q
        \text{ for some }m\in[M],
        \\
        0,
        &
        \bar Q\text{ is not a column of }\widehat K.
    \end{cases}
\end{equation}
We consider the following event, which records successful source training and on which we will subsequently condition:
\begin{equation}
\label{eq:source_pretraining_success}
\begin{aligned}
    \mathcal{E}_{\operatorname{src}}
\eqdef
    \Big\{
        \{\bar{Q}\}_{Q\in\mathcal{Q}_{q:d}}
        =
        \{\widehat{K}_{:m}\}_{m=1}^{M}
        \,\&\,
        \hat{\kappa}=\kappa
        \,\&\,
        (\exists \boldsymbol{\pi}\in\mathbb{S}_\kappa)
        (\forall Q\in\mathcal{Q}_{q:d})
        \ \text{s.t.}\
        \boldsymbol{\pi}
        \bigl(
            \widehat c^{\,\circ}(Q)
        \bigr)
        =
        c(Q)
    \Big\}.
\end{aligned}
\end{equation}
Here, $\widehat K_{:m}$ denotes the $m$th column of $\widehat K$, and $\mathbb{S}_{\kappa}$ denotes the permutation group of $[\kappa]$.  The event~\eqref{eq:source_pretraining_success} records the source-side condition on which the OOD theorem is conditioned; Theorem~\ref{thrm:Main1_Lp_fixedGridVersion}, the fixed-grid version of Theorem~\ref{thrm:Main_Lp}, supplies the corresponding source guarantee for an $(\kappa,\eta)$-symmetric source task.

Let
$
    \widetilde{\mathcal{D}}
    =
    \{(\widetilde X_n,\widetilde Y_n)\}_{n=1}^{\widetilde N}
$
be a $\tilde{\tau}$-training set.  For each $Q\in\mathcal{Q}_{q:d}$, define the index set
$
    \widetilde I_Q
    \eqdef
    \{n\in[\widetilde N]:\widetilde X_n\in Q\}
$
and let
$
    \widetilde N_Q
    \eqdef
    \#\widetilde I_Q.
$
The collection of cells containing at least one target observation is denoted by
$
    \widetilde{\mathcal{Q}}_{\operatorname{obs}}
    \eqdef
    \{Q\in\mathcal{Q}_{q:d}:\widetilde N_Q>0\}.
$
For every $j\in[\hat{\kappa}]$, define
$
    \widetilde{\mathcal{Q}}^{\,\operatorname{sim}}_j
    \eqdef
    \left\{
        Q\in
        \widetilde{\mathcal{Q}}_{\operatorname{obs}}
        \cap
        \mathcal{Q}_{\pi_\star}:
        \widehat c^{\,\circ}(Q)=j
    \right\}.
$
We denote the locally averaged label of each cell $Q$ and the resulting pooled value of each learned class $j$, respectively, by
\begin{equation}
\label{eq:globally_defined_target_cell_average}
    \widetilde Y_Q
    \eqdef
    \begin{cases}
        \displaystyle
        \frac{1}{\widetilde N_Q}
        \sum_{n\in\widetilde I_Q}
        \widetilde Y_n,
        &
        \widetilde N_Q>0,
        \\
        0,
        &
        \widetilde N_Q=0,
    \end{cases}
    \qquad
    \widetilde v_j
    \eqdef
    \begin{cases}
        \displaystyle
        \frac{1}{\#\widetilde{\mathcal{Q}}^{\,\operatorname{sim}}_j}
        \sum_{Q\in\widetilde{\mathcal{Q}}^{\,\operatorname{sim}}_j}
        \widetilde Y_Q,
        &
        \widetilde{\mathcal{Q}}^{\,\operatorname{sim}}_j\neq\emptyset,
        \\
        0,
        &
        \widetilde{\mathcal{Q}}^{\,\operatorname{sim}}_j=\emptyset.
    \end{cases}
\end{equation}
We also write
$
    \widetilde A
    \eqdef
    (\widetilde v_1,\dots,\widetilde v_{\hat{\kappa}})
    \in
    \mathbb R^{D\times\hat{\kappa}}
$
for the fine-tuned class-value factor and
$
    \widetilde V
    \eqdef
    \widetilde A C_{\widehat c}
$
for the resulting fine-tuned values matrix.

\paragraph{The OOD Result}
With this notation in place, we are ready to state our main fine-tuning result under out-of-distribution task shifts.

\begin{theorem}[Approximate-Symmetry Oracle Class-Value Fine-Tuning]
\label{thrm:2_Scalable_Fine_Tuning}
Under Setting~\ref{s:TheSetting_OOD}, suppose that $\mathbb{P}(\mathcal{E}_{\operatorname{src}})>0$, fix $1\le p<\infty$ and $0<\delta\le1$, and run Algorithm configuration~\ref{algconfig:fixed_grid}.
There is a constant $c>0$ such that, with
$
    \widetilde p^\star
    \eqdef
    \min_{Q\in\mathcal Q_{q:d}}
    \widetilde{\mathbb P}_X(Q)>0,
$
if 
\begin{equation}
\label{eq:target_sample_complexity}
        \widetilde{N}
    \ge
        \frac{c}{\widetilde p^\star}
        \left(
            1
            +
            \frac{
                \widetilde{\bar\sigma}^{\,2}
            }{
                \widetilde\omega(\widetilde r^\star)^2
            }
        \right)
        \log
        \left(
            \frac{4Dq^d}{\delta}
        \right)
\end{equation}
then there exists a $\lambda_0>0$%
\footnote{The threshold is deterministic and independent of both realized training samples.}~%
such that any $0<\lambda\le \lambda_0$ satisfies:
\begin{enumerate}
    \item[(i)] \textbf{Probability of Similar Query:} $
    \tilde{\mathbb{P}}_X
    \bigl(
        \mathcal{G}_{\pi_\star}
    \bigr)
    \ge
    \operatorname{Sim}_q
    \big(
        \widetilde{\mathcal{S}}
        \mid
        \mathcal{S}
    \big)
    \,
    q^d\widetilde p^\star
    >
    0
$
    \item[(ii)] \textbf{Conditional Probability of Approximation:} For each $0<\lambda\le\lambda_0$, the fine-tuned attention head 
$\widetilde{\mathcal T}_\lambda(x)
    \eqdef
    \operatorname{Attn}_\lambda
    \bigl(
        x\mid\widehat K,\widetilde V
    \bigr)
$ where $\widetilde V\eqdef \tilde{A}\hat{C}$ satisfies
\begin{align}
\label{eq:scalable_fine_tuning_Lp_bound}
\mathbb{P}\left(
     \mathbb{E}_{\widetilde X\sim\tilde{\mathbb{P}}_X}
        \left[
            \left\|
                \widetilde{\mathcal T}_\lambda(\widetilde X)
                -
                \tilde{f}(\widetilde X)
            \right\|_2^p
            \,\middle|\,
            \widetilde X\in\mathcal{G}_{\pi_\star}
        \right]^{1/p}
    \le 
        5\widetilde\omega\left(\frac{\sqrt d}{q}\right)
        +
        2\sqrt D\,\widetilde\eta
\middle|
    \mathcal{E}_{\operatorname{src}}
\right)
\ge 
1-\delta
.
\end{align}
\end{enumerate}
Moreover, $\tilde{A}$ is determined by at most $D\kappa$ entries.
\end{theorem}

\begin{proof}
See Section~\ref{s:Proofs__ss:OOD}.
\end{proof}

We now explain and examine some subtle aspects of our OOD fine-tuning guarantee.

\subsubsection{Discussion}
\label{s:Symmetries__ss:OOD___sss:Discussion}

\paragraph{The impact of source and OOD task assignments}
Part~(i) of Theorem~\ref{thrm:2_Scalable_Fine_Tuning} quantifies the \textit{probability} that a query falls in a grid cell on which the source and target tasks have aligned domain symmetries; that is, a cell on which the pre-trained features encoded by $\widehat K$ and $C_{\widehat c}$ remain applicable.  Intuitively, this probability is controlled by the similarity
$
    \operatorname{Sim}_q
    \big(
        \widetilde{\mathcal S}
        \mid
        \mathcal S
    \big)
$
between the source task $\tau$ and the OOD task $\widetilde\tau$.  In particular, if $\widetilde\tau\sim_1\tau$, then every grid cell is compatible, and hence
$
    \widetilde{\mathbb P}_X
    \bigl(
        \mathcal G_{\pi_\star}
    \bigr)
    =
    1.
$
For partially aligned tasks, the lower bound in Theorem~\ref{thrm:2_Scalable_Fine_Tuning}~(i) quantifies the corresponding loss in compatible-region mass.

\paragraph{Why $\eta$ appears only in Theorem~\ref{thrm:2_Scalable_Fine_Tuning} but not in Theorem~\ref{thrm:Main_Lp}}
Approximate symmetry affects only the fixed-grid and OOD chains, through the defect terms $\sqrt D\,\eta$ and $\sqrt D\,\widetilde\eta$.  
Although no $\eta$-defect appears in Theorem~\ref{thrm:Main_Lp}, that theorem is formally stated under Setting~\ref{s:TheSetting_pt1}, including its structured-target and separation assumptions; we therefore do not claim a separate symmetry-free theorem here.

\subsubsection{Illustrative OOD Experiments}

We use simple continuous piecewise-affine tasks.  Fix
$1\le \kappa\le q^d$, randomly generate a partition
$\mathcal{S}\eqdef\{\mathcal{S}_k\}_{k=1}^{\kappa}$
of $\mathcal{Q}_{q:d}$, and choose
$\beta_1,\dots,\beta_\kappa\in\mathbb{R}^D$ satisfying
\[
    \|\beta_i-\beta_j\|_\infty\ge \tfrac{1}{q}
\]
whenever $i\ne j$.  On a finite triangulation of $[0,1]^d$ containing every grid midpoint
$\bar Q$ as a vertex, let
$f:[0,1]^d\to[0,1]$
be a continuous piecewise-affine interpolant satisfying
\[
    f(\bar Q)=\beta_k
    \qquad
    \text{for every }Q\in\mathcal{S}_k.
\]
Thus $f$ is Lipschitz, exactly $(\mathcal{S},0)$-symmetric at scale $q$, and its distinct class values are separated in $\ell^2$ by at least $1/q$.  The experiments therefore instantiate the zero-slack case $\eta=\widetilde{\eta}=0$.

For each source task, we freeze the learned key matrix
$
    \widehat{K}
$
and class-assignment matrix
$
    C_{\widehat{c}},
$
and then generate ten OOD tasks with the same partition but newly sampled class
values.  Inputs are sampled uniformly from
$
    [0,1]^d,
$
and independent uniform noise on
$
    [-1,1]
$
is added to the observations.

The OOD experiments compare a class-value update with an unrestricted gridwise values update under a known shared partition.  Their runtime measurements concern only the implemented value-update step and exclude purification and prediction.  They are empirical diagnostics and are not consequences of the theorem's update-parameter count.

We vary
$
    \kappa\in\{8,64\},
$
$
    q\in\{2,4,8,16,32,64\},
$
and
$
    d\in\{2,5,10,50,100\},
$
while holding the remaining parameters fixed.
Each curve reports the mean over five source experiments, each averaged
over ten OOD tasks, together with one-standard-deviation bands.

\begin{figure}[H]
    \centering

    \includegraphics[width=1.0\linewidth]{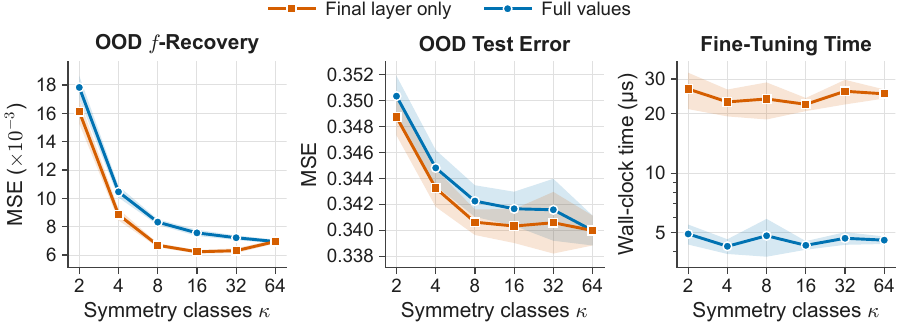}

    \caption{
        \textbf{Number of Symmetry Classes ($\kappa$) Ablation.}
        Fine-tuning only the shared class values matches or improves upon
        updating all cell values.
        The advantage is clearest when
        $
            \kappa\ll q^d,
        $
        while the two procedures agree when
        $
            \kappa=q^d.
        $
        This reflects the predicted gain from reusing a compressed symmetry
        representation.
    }
    \label{fig:OOD_kappa}
\end{figure}

\begin{figure}[H]
    \centering

    \includegraphics[width=1.0\linewidth]{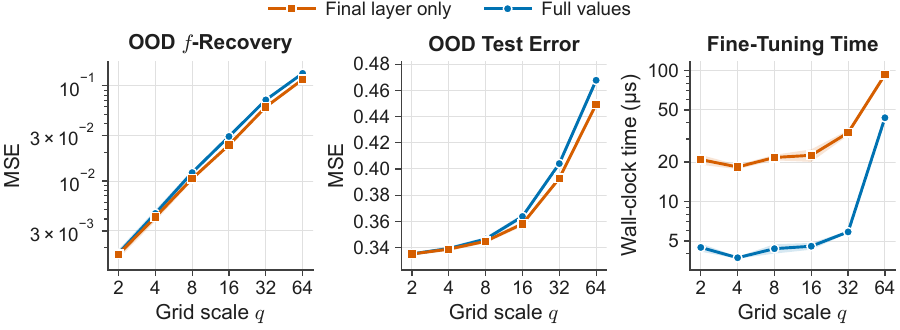}

    \caption{
        \textbf{Grid Scale ($q$) Ablation.}
        Both errors increase as the grid becomes finer at fixed sample size,
        but the class-value update remains uniformly better.
        The same trend reflects the growing separation between a fixed
        $
            \mathcal{O}(\kappa D)
        $
        update and an unrestricted update over
        $
            q^d
        $
        cells.
    }
    \label{fig:OOD_q}
\end{figure}

\section{Relation to Approximation Theorems}
\label{s:Discussion}
We now distinguish these guarantees and relate them to the scope-controlled of our results compared to other results in the literature.

\paragraph{Existential Universal Approximation Theorems}
Classical existence theorems, i.e., the classical UAT~\cite{hornik1989multilayer,cybenko1989approximation,kidger2020universal}, guarantee the existence of a sufficiently large neural network that can approximate any given target function to arbitrary accuracy. However, they offer no guidance on how to identify such a model from data, nor do they specify its size or neuronal structure.

\paragraph{Constructive Universal Approximation Theorems}
Guarantees based on constructive approximation, especially~\cite{yarotsky2017error,petersen2018optimal,bolcskei2019optimal,hanin2019universal,shen2022optimal,riegler2024generating}, improve upon this by providing estimates for the number of neurons and their arrangement (in terms of depth, width, and number of active neurons) required to approximate a function of a given regularity to a specified precision. Although these results do not explicitly rely on finitely many samples, later refinements, cf.~\cite{zhang2024deep,hong2024bridging}, do. While these approaches achieve minimax-optimal approximation rates, as dictated by matching lower bounds in constructive approximation theory~\cite{lorentz1996constructive} and Vapnik-Chervonenkis theory~\cite{yarotsky2017error,shen2022optimal}, they rely on precisely placed samples without measurement noise. These results are formulated under sampling and noise assumptions different from those used here and therefore do not directly provide the noisy-random-sample recovery guarantee analyzed in this paper.  Their primary objective is quantitative approximation by explicitly constructed architectures rather than the two-stage data-driven recovery procedure studied below.  Similar issues arise with the optimal memorization theorems which are finite counterparts to these results; cf.~\cite{vardi2022on,kim2023provable,kratsios2023small,dirksen2024memorization,madden2024memory} and to modern circuit-complexity based constructions yielding superior approximation rates~\cite{kratsios2025quantifying,kratsios2026algorithmicfoundationsdeeplearning}.

\paragraph{Kernelized Neural Network Methods}
Random feature networks (RFNs)~\cite{rahimi2007random} address some of these limitations by randomizing all but the final layer of a neural network, which is then trained via linear, or ridge, regression. 
This stylized linearization enables the tractable analysis of their $L^2$-approximation power~\cite{gonon2023random,neufeld2023universal} and statistical behaviour~\cite{gonon2020risk,mei2022generalization}, by reducing them to kernel methods for which a comprehensive statistical theory is by now well-developed~\cite{cheng2023a,pmlr-v235-cheng24g,cheng2024comprehensive}.
Neural tangent kernels (NTKs)~\cite{jacot2018ntk,geiger2020scaling} further narrow the gap between RFNs and practical neural networks by incorporating gradient dynamics through the analysis of infinitely wide network surrogates.  These linearized surrogates are analytically tractable, but their approximation behaviour need not coincide with that of nonlinear network classes.  This distinction is reflected in the comparison between linear widths~\cite{pinkus2012n} and nonlinear widths~\cite{MR1689432,cohen2022optimal,petrova2023lipschitz}.  Related statistical separations have been documented between RFNs and two-layer neural networks with one step of gradient training on their hidden features~\cite{el2024efficiency,mousavihosseini2024robust}.
Average-risk guarantees and worst-case uniform recovery address different objectives.  Uniform control is relevant in applications where rare extreme events matter, including fluid dynamics~\cite{dematteis2018rogue}, climate science~\cite{trenberth2015attribution}, and actuarial science~\cite{embrechts1997modelling}.

\paragraph{Linear Regression with MLP Bases and Compressed-Sensing Constructions}
Recently, progress has been made toward concretely trainable neural networks that achieve optimal approximation rates, albeit on idealized training data. Two parallel directions have emerged, both of which construct large dictionaries of basic functions (e.g., neural networks) and then perform linear regression over that dictionary. The first direction emulates classical compressed sensing results by constructing optimal polynomial bases using expressive $\operatorname{ReLU}^2$ networks~\cite{brugiapaglia2024physics,brugiapagliapractical2024,franco2025practical,shi2025learning}. These models enjoy near-optimal approximation rates, admit training via ridge regression, and can accommodate randomly sampled inputs. The second direction~\cite{daubechies2022nonlinear,MR4659237,schneider2025nonlocal} follows a similar strategy, but builds large dictionaries consisting of Riesz bases for standard function spaces that can be realized by very simple $\operatorname{ReLU}$ MLPs; sparse linear regression over this dictionary then yields an effective approximator. While the first approach accommodates finitely many noiseless but randomly located input samples, neither direction currently handles random measurement noise in the setting where uniform approximation of the target function is desired.

\paragraph{Additional related work.}
Our source procedure is closest to explicit Lipschitz-regression methods, including minimax theory, approximate extension, tractable delta-convex/local estimators, and~\citet{ghoukasian2026closedform}~\citep{stone1982optimal,gottlieb2017efficient,balazs2026nearoptimal,bettinger2026revisiting}; these treat single-task estimation, not cross-task adaptation with a $D\kappa$ update bound. Unlike deep-regression guarantees based on approximate ERM or optimization gaps~\citep{schmidt2020nonparametric}, our procedure maps data directly to attention parameters. Its normalized exponential-distance head connects to kernel views of attention and Nadaraya--Watson interpolation~\citep{tsai2019transformer,barzilai2025beyond}, while adapters and LoRA~\citep{houlsby2019parameter,hu2022lora} target parameter efficiency rather than symmetry-dependent OOD recovery. Our conditional structural claim is that source recovery and task compatibility permit adaptation through at most $D\kappa$ target-dependent class-value entries, without asserting lower end-to-end OOD computation.

\section{Conclusion}
\label{s:ConclusionFutureWork}
In summary, the principal adaptive source procedure recovers the regression function from noisy data through data-adaptive centres and tied-value compression, without claiming recovery of the prescribed grid classes.  A separate fixed-grid chain handles $(\mathcal S,\eta)$-approximate symmetry, recovers the represented classes up to relabelling under the stated margin condition, and supplies the frozen representation used for oracle out-of-distribution class-value fine-tuning.  The OOD theorem guarantees recovery only on the supplied compatible target region and confines target dependence to class-value entries after successful source training; exact symmetry is the zero-slack specialization.  
Extensions to symmetry-free adaptive transfer, compatible-class pooling, and data-driven recovery of compatibility regions are left for future work.
Thus the paper gives a constructive noisy nonparametric-regression and transfer-learning analysis for a simple transformer under structured task shift and approximate symmetry.  The positive-temperature thresholds remain deterministic and sample-independent, and the entry-count conclusions do not imply reduced target sample complexity, total memory, inference cost, preprocessing time, or total computation.

\section{Acknowledgements and Funding}
\label{s:AckFund}
A.\ Kratsios acknowledges financial support from an NSERC Discovery Grant No.\ RGPIN-2023-04482 and No.\ DGECR-2023-00230.  They acknowledge that resources used in preparing this research were provided, in part, by the Province of Ontario, the Government of Canada through CIFAR, and companies sponsoring the Vector Institute\footnote{\href{https://vectorinstitute.ai/partnerships/current-partners/}{https://vectorinstitute.ai/partnerships/current-partners/}}.

\section{Proofs}
\label{a:Proofs}

We now prove the paper's results.  
We will make use of the following ``hard attention''; which arises as the pointwise limit of our attention mechanism with extremal temperature.
For $z\in\mathbb R^M$, define its maximizer set by
$
    \mathcal M(z)
    \eqdef
    \operatorname*{argmax}_{i\in[M]}z_i
$.  
Since $M$ is finite, $\mathcal M(z)$ is nonempty.  For each $i\in [M]$, define
\begin{equation}
\label{eq:zero_temperature_softmax_definition}
    \operatorname{Softmax}_0(z)_i
    \eqdef
    \frac{
        I_{\{i\in\mathcal M(z)\}}
    }{
        \#\mathcal M(z)
    }
.
\end{equation}
Accordingly, for every key matrix
$
    K\in\mathbb R^{d\times M}
$
and value matrix
$
    V\in\mathbb R^{D\times M},
$
define
\begin{equation}
\label{eq:zero_temperature_attention_definition}
    \operatorname{Attn}_0(x\mid K,V)
    \eqdef
    V
    \operatorname{Softmax}_0
    \left(
        -
        \bigl(
            \|x-K_i\|_\infty
        \bigr)_{i=1}^M
    \right).
\end{equation}
Thus zero-temperature attention averages uniformly over all nearest keys.

\subsection{Pre-Processing: Denoising Algorithm}
\label{a:Proofs__ss:PreProcessingDenoising}

We will use the following variant of~\citep[Lemma 22]{hou2023instance}.
\begin{lemma}[Empirical Noise Concentration]
\label{lem:lp_concentration_empirical_noise}
Let $D,N\in\mathbb N_+$, let $\varepsilon_1,\dots,\varepsilon_N$ be independent centred $\mathbb R^D$-valued random vectors with common law $\mathbb{P}_\varepsilon$, and assume that
\begin{equation}
\label{eq:noise_radial_subgaussian_appendix}
        \mathbb{P}
        \bigl(
            \|\varepsilon_i\|_2>t
        \bigr)
    \le
        C\exp
        \left(
            -\frac{t^2}{\sigma^2}
        \right)
\end{equation}
for every $i\in[N]$ and $t>0$, where $C\ge0$ and $\sigma>0$.  Set
$
    \sigma_C\eqdef\sigma\max\{C,1\}
$
and adopt $D^{1/\infty}=1$.  Then, for every nonempty deterministic set $I\subseteq[N]$, every $1\le r\le\infty$, every $t>0$, and every $\sigma$-algebra $\mathcal H$ independent of $\sigma(\varepsilon_1,\dots,\varepsilon_N)$,
\begin{equation}
\label{eq:conditional_lp_empirical_noise_concentration_appendix}
        \mathbb{P}
        \left(
            \left.
            \left\|
                \frac{1}{\#I}
                \sum_{i\in I}\varepsilon_i
            \right\|_r
            >t
            \,\right|\,
            \mathcal H
        \right)
    \le
        2D
        \exp
        \left(
            -
            \frac{
                \#I\,t^2
            }{
                4D^{2/r}\sigma_C^2
            }
        \right)
.
\end{equation}
\end{lemma}

\begin{proof}[Proof of Lemma~\ref{lem:lp_concentration_empirical_noise}]
Fix a nonempty deterministic $I\subseteq[N]$, set $m\eqdef\#I$, and write
$
    Z\eqdef m^{-1}\sum_{i\in I}\varepsilon_i.
$
For $j\in[D]$, let $\varepsilon_{i,j}$ and $Z_j$ denote the corresponding coordinates.  Since $\varepsilon_i$ is centred, $\varepsilon_{i,j}$ is centred.  Moreover,
$
    |\varepsilon_{i,j}|\le\|\varepsilon_i\|_2,
$
so~\eqref{eq:noise_radial_subgaussian_appendix} implies
\[
    \mathbb{P}
    \bigl(
        |\varepsilon_{i,j}|>u
    \bigr)
    \le
    C\exp
    \left(
        -\frac{u^2}{\sigma^2}
    \right)
\]
for every $u>0$. We first derive the required moment-generating-function bound. Let $W$ be any centred scalar random variable satisfying
$
    \mathbb P(|W|>u)\le C e^{-u^2/\sigma^2}
$
for every $u>0$. Fix $s\in\mathbb R$, and set $a\eqdef\sigma|s|$ and $C_0\eqdef\max\{C,1\}$. The tail bound implies exponential integrability. Since $e^v-1-v\le e^{|v|}-1-|v|$ for every $v\in\mathbb R$, centring and Tonelli's theorem give
\begin{align}
\label{eq:scalar_tail_to_mgf_remainder}
    \mathbb E[e^{sW}]-1
&=
    \mathbb E[e^{sW}-1-sW]
\nonumber\\
&\le
    |s|\int_0^\infty
    (e^{|s|u}-1)\mathbb P(|W|>u)\,du
\nonumber\\
&\le
    Ca\int_0^\infty(e^{av}-1)e^{-v^2}\,dv
\nonumber\\
&\le
    Ca^2\int_0^\infty v e^{av-v^2}\,dv
\nonumber\\
&\le
    Ca^2e^{a^2/2}\int_0^\infty v e^{-v^2/2}\,dv
\nonumber\\
&=
    Ca^2e^{a^2/2}
\le
    C(e^{a^2}-1).
\end{align}
Here the third inequality uses $e^t-1\le te^t$ for $t\ge0$, the fourth uses $av\le(a^2+v^2)/2$, and the last uses
$
    e^t-1=2e^{t/2}\sinh(t/2)\ge te^{t/2}
$
for $t\ge0$. Since $C_0\ge1$, convexity of the exponential function yields $
    1+C_0(e^{a^2}-1)\le e^{C_0a^2}
$.
Consequently, \eqref{eq:scalar_tail_to_mgf_remainder} gives
\begin{equation}
\label{eq:scalar_noise_mgf_bound}
    \mathbb E[e^{sW}]
\le
    e^{C_0a^2}
\le
    e^{C_0^2a^2}
=
    e^{\sigma_C^2s^2},
\end{equation}
for each $s\in\mathbb{R}$.
Apply~\eqref{eq:scalar_noise_mgf_bound} to each $\varepsilon_{i,j}$. Independence across $i$ gives
$
    \mathbb E[e^{sZ_j}]
    =
    \prod_{i\in I}\mathbb E[e^{s\varepsilon_{i,j}/m}]
    \le
    e^{\tfrac{\sigma_C^2s^2}{m}}
$.
For $u>0$, exponential Markov inequality with $s=mu/(2\sigma_C^2)$ therefore gives
$
    \mathbb P(Z_j>u)
\le
    \exp\left(-\frac{mu^2}{4\sigma_C^2}\right)$.  
Applying the same argument to $-Z_j$ and taking a union bound yields
\begin{equation}
\label{eq:coordinate_empirical_noise_concentration_appendix}
    \mathbb P(|Z_j|>u)
    \le
    2\exp\left(-\frac{mu^2}{4\sigma_C^2}\right).
\end{equation}
For $1\le r\le\infty$,
$
    \|z\|_r\le D^{1/r}\|z\|_\infty.
$
Consequently,
$
    \{\|Z\|_r>t\}
\subseteq
    \cup_{j=1}^D
    \{
        |Z_j|>D^{-1/r}t
    \}$.  A union bound and~\eqref{eq:coordinate_empirical_noise_concentration_appendix} thus yield
$
        \mathbb{P}
        \bigl(
            \|Z\|_r>t
        \bigr)
    \le
        2D
        e^{
            -
            \tfrac{mt^2}{4D^{2/r}\sigma_C^2}
        }
$.
The event $\{\|Z\|_r>t\}$ is measurable with respect to $\sigma(\varepsilon_1,\dots,\varepsilon_N)$ and is independent of $\mathcal H$.  Its conditional probability given $\mathcal H$ therefore equals its unconditional probability almost surely, which proves~\eqref{eq:conditional_lp_empirical_noise_concentration_appendix}.
\end{proof}

\begin{lemma}[Localized Cell Averaging]
\label{lem:recovery_simple}
Let $f:[0,1]^d\to\mathbb R^D$ have modulus of continuity $\omega$ from $\ell^\infty$ on the domain to $\ell^2$ on the codomain.  For each $s\in\{0,1\}$ and $n\in[N]$, let $X_n^{(s)}$ have law $\mathbb{P}_X$, let $\varepsilon_n^{(s)}$ be centred with law $\mathbb{P}_\varepsilon$, and set
$
    Y_n^{(s)}
    =
    f(X_n^{(s)})
    +
    \varepsilon_n^{(s)}.
$
Assume that the two samples are independent, that the input variables within each sample are i.i.d., that the noise variables are i.i.d., and that all noise variables are independent of all input variables.  Suppose that, for some $C\ge0$ and $\sigma>0$,
\begin{equation}
\label{eq:localized_noise_tail_appendix}
        \mathbb{P}
        \bigl(
            \|\varepsilon_n^{(s)}\|_2>u
        \bigr)
    \le
        C\exp
        \left(
            -\frac{u^2}{\sigma^2}
        \right)
\end{equation}
for every $s\in\{0,1\}$, $n\in[N]$, and $u>0$, and set $\sigma_C\eqdef\sigma\max\{C,1\}$.  Let
$
    \mathcal{F}_0
    =
    \sigma(X_1^{(0)},\dots,X_N^{(0)}).
$
Let $\mathcal{Q}$ be an $\mathcal{F}_0$-measurable finite Borel partition with $\mathcal{F}_0$-measurable centres $\bar{Q}\in[0,1]^d$. For $Q\in\mathcal{Q}$, define
\begin{equation}
\label{eq:localized_cell_radius_appendix}
        r_{Q,\rho}
    \eqdef
        \sup
        \left(
            \left\{
                \|x-\bar{Q}\|_\infty:
                x\in Q_\rho\cap\operatorname{supp}(\mathbb{P}_X)
            \right\}
            \cup\{0\}
        \right).
\end{equation}
Let
$
    I_{Q,\rho}
    =
    \{n\in[N]:X_n^{(1)}\in Q_\rho\}
$
and $N_{Q,\rho}=\#I_{Q,\rho}$.  Define
\begin{equation}
\label{eq:globally_defined_cell_average}
        \bar Y_{Q,\rho}
    \eqdef
        \begin{cases}
        \displaystyle
        \frac{1}{N_{Q,\rho}}
        \sum_{n\in I_{Q,\rho}}
        Y_n^{(1)},
        &
        N_{Q,\rho}>0,
        \\
        0,
        &
        N_{Q,\rho}=0.
        \end{cases}
\end{equation}
Set
\[
    \mathcal{F}_X
    =
    \sigma
    \bigl(
        X_1^{(0)},\dots,X_N^{(0)},
        X_1^{(1)},\dots,X_N^{(1)}
    \bigr).
\]
Then, for every $Q\in\mathcal{Q}$ and $t>0$, almost surely on $\{N_{Q,\rho}>0\}$,
\begin{equation}
\label{eq:signal_recovery_random_index_appendix}
\begin{aligned}
&    \mathbb{P}
    \left(
        \left.
        \|\bar Y_{Q,\rho}-f(\bar{Q})\|_2
        >
        \omega(r_{Q,\rho})+t
        \,\right|\,
        \mathcal{F}_X
    \right)
\le
    2D
    \exp
    \left(
        -
        \frac{N_{Q,\rho}t^2}{4D\sigma_C^2}
    \right).
\end{aligned}
\end{equation}
\end{lemma}

\begin{proof}[Proof of Lemma~\ref{lem:recovery_simple}]
Condition on $\mathcal{F}_X$.  Then $I_{Q,\rho}$ and $N_{Q,\rho}$ are deterministic.  On $\{N_{Q,\rho}>0\}$,
\begin{align}
        \|\bar Y_{Q,\rho}-f(\bar{Q})\|_2
&\le
        \frac{1}{N_{Q,\rho}}
        \sum_{n\in I_{Q,\rho}}
        \|f(X_n^{(1)})-f(\bar{Q})\|_2
        +
        \left\|
            \frac{1}{N_{Q,\rho}}
            \sum_{n\in I_{Q,\rho}}
            \varepsilon_n^{(1)}
        \right\|_2
\nonumber
\\
&\le
        \omega(r_{Q,\rho})
        +
        \left\|
            \frac{1}{N_{Q,\rho}}
            \sum_{n\in I_{Q,\rho}}
            \varepsilon_n^{(1)}
        \right\|_2.
\label{eq:signal_bias_noise_appendix}
\end{align}
The only nontrivial point is that the index set is random before conditioning.  For every nonempty deterministic $I\subseteq[N]$, Lemma~\ref{lem:lp_concentration_empirical_noise}, with $r=2$ and $\mathcal H=\mathcal{F}_X$, gives
\[
    \mathbb{P}
    \left(
        \left.
        \left\|
            \frac{1}{\#I}
            \sum_{n\in I}\varepsilon_n^{(1)}
        \right\|_2
        >t
        \,\right|\,
        \mathcal{F}_X
    \right)
    \le
    2D
    \exp
    \left(
        -\frac{\#I\,t^2}{4D\sigma_C^2}
    \right).
\]
Since the events $\{I_{Q,\rho}=I\}$, $I\subseteq[N]$, form a finite $\mathcal{F}_X$-measurable partition, multiplication by their indicators and summation over all nonempty $I$ yield
\[
    \mathbb{P}
    \left(
        \left.
        \left\|
            \frac{1}{N_{Q,\rho}}
            \sum_{n\in I_{Q,\rho}}\varepsilon_n^{(1)}
        \right\|_2
        >t
        \,\right|\,
        \mathcal{F}_X
    \right)
    \le
    2D
    \exp
    \left(
        -\frac{N_{Q,\rho}t^2}{4D\sigma_C^2}
    \right)
\]
almost surely on $\{N_{Q,\rho}>0\}$.  Combining this estimate with~\eqref{eq:signal_bias_noise_appendix} proves the result.
\end{proof}

\subsection{Adaptive Source-Recovery Chain}
\label{s:Approx}

The following deterministic fine-grid configuration is used only to
certify feasibility and measurability of the adaptive centre selection.
Let
\begin{equation}
\label{eq:adaptive_candidate_fine_grid}
    \mathcal Z_{2q}
    \eqdef
    \left\{
        \left(
            \frac{2i_1+1}{4q},
            \dots,
            \frac{2i_d+1}{4q}
        \right):
        (i_1,\dots,i_d)
        \in
        \{0,\dots,2q-1\}^d
    \right\}.
\end{equation}
Fix once and for all an enumeration
$
    \mathcal Z_{2q}
    =
    \{
        z_k^{\operatorname{grid}}
    \}_{k=1}^K
$
and identify this set with the ordered tuple
$
    z^{\operatorname{grid}}
    \eqdef
    (
        z_k^{\operatorname{grid}})_{k=1}^K$.

\begin{lemma}[{Measurability of Selection: $(\bar Q_1,\dots,\bar Q_K)$ in Algorithm~\ref{alg:STEP_1_PURIFY}}]
\label{lem:measurability_check}
Let
$
    \mathcal F_0
    \eqdef
    \sigma(\{X_n^{(0)}\}_{n=1}^N)
$
be the pilot $\sigma$-algebra, and suppose that
$
    \mathfrak Z_{K,\delta_{\mathrm{sep}}}
$
is nonempty. Then there exists a Borel measurable map (selector)
\[
    \mathcal S:
    \mathbb R^{N\times d}
    \longrightarrow
    \mathfrak Z_{K,\delta_{\mathrm{sep}}}
\]
satisfying the following for every $(x_1,\dots,x_N)\in(\mathbb R^d)^N$
\[
    \max_{1\le n\le N}
    \min_{1\le k\le K}
    \|x_n-\mathcal S(x_1,\dots,x_N)_k\|_\infty
    =
    \inf_{
        (z_1,\dots,z_K)
        \in
        \mathfrak Z_{K,\delta_{\mathrm{sep}}}
    }
    \max_{1\le n\le N}
    \min_{1\le k\le K}
    \|x_n-z_k\|_\infty.
\]
\end{lemma}

\begin{proof}[Proof of Lemma~\ref{lem:measurability_check}]
Set $\mathsf X\eqdef(\mathbb R^d)^N$, $\mathsf Z\eqdef\mathfrak Z_{K,\delta_{\mathrm{sep}}}$.   By definition,
\[
    \mathsf Z
    =
    \left\{
        (z_1,\dots,z_K)\in([0,1]^d)^K:
        \|z_i-z_j\|_\infty\ge\delta_{\mathrm{sep}}
        \text{ whenever }i\neq j
    \right\}.
\]
Since each map
$
    (z_1,\dots,z_K)\mapsto\|z_i-z_j\|_\infty
$
is continuous, $\mathsf Z$ is closed in the compact metric space
$([0,1]^d)^K$. Hence $\mathsf Z$ is compact, and it is nonempty by
assumption. Define $
    \Phi:\mathsf X\times\mathsf Z\longrightarrow[0,\infty)
$
for $u=(u_1,\dots,u_N)$ and  $z=(z_1,\dots,z_K)$ by
\[
    \Phi(u,z)
    \eqdef
    \max_{n\in[N]}
    \min_{k\in[K]}
    \|u_n-z_k\|_\infty.
\]
The map $(u,z)\mapsto\Phi(u,z)$ is continuous, since each
$(u,z)\mapsto\|u_n-z_k\|_\infty$ is continuous and finite minima and
maxima preserve continuity.
Now, consider the constant correspondence
\[
    \varphi:\mathsf X\rightrightarrows\mathsf Z,
    \qquad
    \varphi(u)\eqdef\mathsf Z.
\]
It is continuous and has nonempty compact values. Applying the Berge
maximum theorem~\citep[Theorem~17.31]{aliprantis1994infinite} to the
continuous objective $-\Phi$ and the correspondence $\varphi$, the value
function
$(\mathbb R^d)^N\ni u
    \mapsto
    \max_{z\in\mathsf Z}\{-\Phi(u,z)\}
$
is continuous. Moreover, the argmin correspondence
$
    \Gamma(u)
\eqdef
    \operatorname*{arg\,min}_{z\in\mathsf Z}\Phi(u,z)
$
has nonempty compact values. Since $\mathsf Z$ is a metric, hence
Hausdorff, space, the same theorem also yields that $\Gamma$ is upper
hemicontinuous.
Set $v(u)\eqdef\min_{z\in\mathsf Z}\Phi(u,z) =-\max_{z\in\mathsf Z}\{-\Phi(u,z)\}$; this function is continuous.
The graph of $\Gamma$ can be written as
\begin{equation}
\label{eq:graph_Gamma}
    \operatorname{Gr}(\Gamma)
=
    \left\{
        (u,z)\in\mathsf X\times\mathsf Z:
        \Phi(u,z)=v(u)
    \right\}
\end{equation}
Now, since $\Phi$ and $v$ are continuous then from~\eqref{eq:graph_Gamma} we see that
$\operatorname{Gr}(\Gamma)$ is closed in
$\mathsf X\times\mathsf Z$, and so $\Gamma$ is Borel measurable. Furthermore, every section
$
    \operatorname{Gr}(\Gamma)_u
    =
    \Gamma(u)
$
is nonempty and compact. Since $\mathsf X$ and $\mathsf Z$ are Polish; thus we can apply  the Arsenin-Kunugui Theorem~\citep[Theorem~18.18]{kechris1995classical}, there exists a Borel measurable map (selector)
$
    \mathcal S:\mathsf X\to\mathsf Z
$
such that
$
    \mathcal S(u)\in\Gamma(u)
$
for every $u\in\mathsf X$.
Consequently, $
    \Phi\bigl(u,\mathcal S(u)\bigr)
    =
    \min_{z\in\mathsf Z}\Phi(u,z)
    =
    \inf_{z\in\mathsf Z}\Phi(u,z)
$,  which is precisely the asserted identity.  Finally, since
$
    (X_1^{(0)},\dots,X_N^{(0)})
$
is $\mathcal F_0$-measurable, and $\mathcal S$ is Borel measurable, their composition
$
    (\bar Q_1,\dots,\bar Q_K)
=
    \mathcal S(X_1^{(0)},\dots,X_N^{(0)})
$ is $\mathcal F_0$-measurable.
\end{proof}

The following lemma will be implicitly, and repeatedly, used throughout our arguments.
\begin{lemma}[Probability of Pilot Coverage]
\label{lem:ahlfors_sampling_consequences}
If $\operatorname{supp}(\mathbb{P}_X)=[0,1]^d$ and $\mathbb{P}_X$ is $d$-Ahlfors regular on $([0,1]^d,\|\cdot\|_\infty)$ with constants $0<c_-\le c_+<\infty$%
\footnote{That is, $
    c_-r^d\le \mathbb{P}_X(Q_d(x,r))\le c_+r^d
$ for every $x\in[0,1]^d$ and every $0<r\le1$.}%
. Then, for every $q\in\mathbb N_+$, there is a deterministic finite Borel cover
$
    \mathcal U_q=\{U_i\}_{i\in\{0,\dots,4q-1\}^d}
$
of $[0,1]^d$ satisfying 
\begin{equation}
\label{eq:ahlfors_uniform_local_mass}
\begin{aligned}
L& =(4q)^d
\\\operatorname{diam}_\infty(U_i)& =\frac{1}{4q}
\\p_{\operatorname{cov}}
& \eqdef
    \min_i\mathbb{P}_X(U_i)
    \ge
    c_-(8q)^{-d}
\\
    \underline p_\rho
&\eqdef
    \inf_{z\in[0,1]^d}
    \mathbb{P}_X\bigl([0,1]^d\cap Q_d(z,\rho)\bigr)
    \ge
    c_-\rho^d
    >
    0
\qquad
    (\forall 0<\rho\le 1)
.
\end{aligned}
\end{equation}
For any i.i.d.\ $\{X_n^{(0)}\}_{n=1}^N$ with law $\mathbb{P}_X$ the ``pilot-coverage event'' defined by
\begin{equation}
\label{eq:ahlfors_pilot_occupancy_event}
    \mathcal{E}_{\operatorname{pilot}}
    \eqdef
    \bigcap_{i\in\{0,\dots,4q-1\}^d}
    \left\{
        \sum_{n=1}^N I_{\{X_n^{(0)}\in U_i\}}
        \ge1
    \right\}
\end{equation}
satisfies $\mathbb{P}\bigl(\mathcal{E}_{\operatorname{pilot}}^c\bigr)
    \le (4q)^d e^{-Np_{\operatorname{cov}}}
    \le
    (4q)^d
    \,
    e^{-Nc_-(8q)^{-d}}$.
\end{lemma}
\begin{proof}[{Proof of Lemma~\ref{lem:ahlfors_sampling_consequences}}]
For $i=(i_1,\dots,i_d)\in\{0,\dots,4q-1\}^d$, set
$
    u_i\eqdef((2i_1+1)/(8q),\dots,(2i_d+1)/(8q))
$
and $
    U_i\eqdef Q_d(u_i,1/(8q))
$.
Equivalently,
$
    U_i=\prod_{j=1}^d[i_j/(4q),(i_j+1)/(4q)]
$.
Hence each $U_i$ is closed, and therefore Borel. The family $\mathcal U_q$ covers $[0,1]^d$, $\#\mathcal U_q=(4q)^d$, and $\operatorname{diam}_\infty(U_i)=1/(4q)$. Since $u_i\in[0,1]^d$ and $1/(8q)\le1$, the lower Ahlfors bound gives $
    \mathbb{P}_X(U_i)\ge c_-(8q)^{-d}
$ for every $i$. Taking the minimum over $i$ proves the asserted bound on $p_{\operatorname{cov}}$.
\hfill\\
\noindent
Next fix $0<\rho\le1$ and $z\in[0,1]^d$. Since $\mathbb{P}_X$ is supported on $[0,1]^d$,
$
    \mathbb{P}_X([0,1]^d\cap Q_d(z,\rho))
    =
    \mathbb{P}_X(Q_d(z,\rho))
$.
The lower Ahlfors bound therefore yields $
    \mathbb{P}_X([0,1]^d\cap Q_d(z,\rho))\ge c_-\rho^d
$.
Since the right-hand side is independent of $z$, taking the infimum over $z\in[0,1]^d$ proves~\eqref{eq:ahlfors_uniform_local_mass}.
Finally, fix $i\in\{0,\dots,4q-1\}^d$. Independence of the pilot inputs and the bound $\mathbb{P}_X(U_i)\ge c_-(8q)^{-d}$ give
\[
    \mathbb{P}\left(X_n^{(0)}\notin U_i\ \,\forall\,n\in[N]\right)
    =
    \bigl(1-\mathbb{P}_X(U_i)\bigr)^N
\le
    e^{-N\mathbb{P}_X(U_i)}
\le
    e^{-Nc_-(8q)^{-d}}
,
\]
where we used $1-u\le e^{-u}$ for $u\in[0,1]$. By~\eqref{eq:ahlfors_pilot_occupancy_event}, the complement $\mathcal{E}_{\operatorname{pilot}}^c$ is the union of the events that at least one $U_i$ receives no pilot observation. A union bound over the $(4q)^d$ indices therefore gives the claimed bound on $\mathbb{P}(\mathcal{E}_{\operatorname{pilot}}^c)$. Under the stated sample-size condition, $
    (4q)^d e^{-Nc_-(8q)^{-d}}\le\delta/2
$, and hence $
    \mathbb{P}(\mathcal{E}_{\operatorname{pilot}})\ge1-\delta/2
$.
\end{proof}

\begin{lemma}[Pilot Coverage Implies Population Coverage]
\label{lem:adaptive_pilot_population_cover}
Let
$
    K=(2q)^d
$
and
$
    \delta_{\mathrm{sep}}=1/(2q).
$
Let
$
    \widehat Z
    =
    \{\bar{Q}_1,\dots,\bar{Q}_K\}
$
be the centre set selected from the pilot sample by
Algorithm~\ref{alg:STEP_1_PURIFY}.  Let
$
    \mathcal Z_{2q}
$
be defined by~\eqref{eq:adaptive_candidate_fine_grid}.
Then
$
    z^{\operatorname{grid}}
    \in
    \mathfrak{Z}_{K,\delta_{\mathrm{sep}}}
$
and
\begin{equation}
\label{eq:fine_grid_cover_radius}
    \sup_{x\in[0,1]^d}
    \min_{z\in\mathcal Z_{2q}}
    \|x-z\|_\infty
    \le
    \frac{1}{4q}.
\end{equation}
Let
$
    \mathcal{U}=\{U_\ell\}_{\ell=1}^L
$ 
be as in Lemma~\ref{lem:ahlfors_sampling_consequences}
where $\mathcal{E}_{\operatorname{pilot}}$ is as in~\eqref{eq:ahlfors_pilot_occupancy_event}; then $
    \mathbb{P}
    \left(
        \mathcal{E}_{\operatorname{pilot}}^c
    \right)
    \le
    L
    \,e^{
        -Np_{\operatorname{cov}}
    }
$ and 
\begin{equation}
\label{eq:adaptive_population_cover_radius}
\mathbb{P}\Biggl(
    \sup_{x\in\operatorname{supp}(\mathbb{P}_X)}
    \min_{k\in[K]}
    \|x-\bar{Q}_k\|_\infty
    \le
    \frac{1}{2q}
\Big|
    \mathcal{E}_{\operatorname{pilot}}
\Biggr)
=
1
.
\end{equation}
\end{lemma}

\begin{proof}[Proof of Lemma~\ref{lem:adaptive_pilot_population_cover}]
Write each point of
$
    \mathcal Z_{2q}
$
as $
    z_i
    =
    (
        \frac{2i_1+1}{4q},
        \dots,
        \frac{2i_d+1}{4q}
    )
$ for each $i\in\{0,\dots,2q-1\}^d$.  
If $i\neq j$, then $i_r\neq j_r$ for at least one coordinate $r$.  In that
coordinate,
\begin{equation}
\label{eq:fine_grid_separation_coordinate}
    |(z_i)_r-(z_j)_r|
    =
    \frac{|i_r-j_r|}{2q}
    \ge
    \frac{1}{2q}
    =
    \delta_{\mathrm{sep}}.
\end{equation}
Therefore
$
    \|z_i-z_j\|_\infty\ge\delta_{\mathrm{sep}},
$
which proves that $z^{\operatorname{grid}}
    \in
    \mathfrak{Z}_{K,\delta_{\mathrm{sep}}}$.  
Fix $x\in[0,1]^d$.  For every coordinate $r\in[d]$, there is an index
$
    i_r\in\{0,\dots,2q-1\}
$
such that
\begin{equation}
\label{eq:one_dimensional_fine_grid_cover}
    \left|
        x_r-\frac{2i_r+1}{4q}
    \right|
    \le
    \frac{1}{4q}.
\end{equation}
Indeed, the points
$
    (2i+1)/(4q)
$
are the midpoints of the $2q$ intervals of length $1/(2q)$ partitioning
$
    [0,1].
$
For the resulting multi-index $i$,
$
    \|x-z_i\|_\infty
    =
    \max_{r\in[d]}
    \left|
        x_r-\frac{2i_r+1}{4q}
    \right|
    \le
    \tfrac{1}{4q}
$.  This proves~\eqref{eq:fine_grid_cover_radius}.  
Let $
    \widehat R_N
\eqdef
    \max_{n\in[N]}
    \min_{k\in[K]}
    \|X_n^{(0)}-\bar{Q}_k\|_\infty$.  
The selected centre configuration minimizes the empirical covering
objective over the feasible set.  The ordered configuration
$
    z^{\operatorname{grid}}
$
is a feasible competitor.  Hence
$
    \widehat R_N
\le
    \max_{n\in[N]}
    \min_{z\in\mathcal Z_{2q}}
    \|X_n^{(0)}-z\|_\infty
\le
    \tfrac{1}{4q}
$; 
where the last inequality follows from
\eqref{eq:fine_grid_cover_radius}. We work on
$
    \mathcal{E}_{\operatorname{pilot}}.
$
Fix
$
    x\in\operatorname{supp}(\mathbb{P}_X).
$
Because
$
    \mathcal{U}
$
covers the support, there is
$
    \ell\in[L]
$
such that
$
    x\in U_\ell.
$
On the pilot-coverage event there is
$
    n\in[N]
$
such that
$
    X_n^{(0)}\in U_\ell.
$
The diameter condition gives
\begin{equation}
\label{eq:query_to_pilot_cover_distance}
    \|x-X_n^{(0)}\|_\infty
    \le
    \operatorname{diam}_\infty(U_\ell)
    \le
    \frac{1}{4q}.
\end{equation}
By the definition of $\widehat R_N$, there is $k\in[K]$ with
\begin{equation}
\label{eq:pilot_to_selected_centre_distance}
    \|X_n^{(0)}-\bar{Q}_k\|_\infty
    \le
    \widehat R_N
    \le
    \frac{1}{4q}.
\end{equation}
Therefore,  
\eqref{eq:query_to_pilot_cover_distance}-\eqref{eq:pilot_to_selected_centre_distance}
give
$
    \|x-\bar{Q}_k\|_\infty
    \le
    \frac{1}{2q}
$.  
Taking the minimum over $k$ and then the supremum over the support proves
\eqref{eq:adaptive_population_cover_radius}.
For the probability estimate, fix $\ell\in[L]$.  Independence of the pilot
inputs gives
\begin{align}
\label{eq:single_pilot_cover_miss_probability}
    \mathbb{P}
    \left(
        X_n^{(0)}\notin U_\ell
        \, (\,\forall \, n\in[N] )\,
    \right)
=
    \left(
        1-\mathbb{P}_X(U_\ell)
    \right)^N
\le
    e^{
        -N\mathbb{P}_X(U_\ell)
    }
\le
    e^{
        -Np_{\operatorname{cov}}
    }
\end{align}
where
$
    1-u\le e^{-u}
$
for
$
    0\le u\le1
$
was used in the first inequality.  A union bound over
$
    \ell\in[L]
$
proves the upper bound on $\mathbb{P}
    (
        \mathcal{E}_{\operatorname{pilot}}^c
    )$.
\end{proof}

The next lemma converts centre separation into a deterministic local-cell inclusion, and hence into the positive mass bound needed for purification.

\begin{lemma}[Separated Centres Yield Positive-Mass Local Cells]
\label{lem:adaptive_local_ball_in_cell}
Let
$
    z_1,\dots,z_K\in[0,1]^d
$
satisfy
$
    \|z_i-z_j\|_\infty\ge\delta_{\mathrm{sep}}
$
whenever $i\neq j$.  For every $k\in[K]$, let
\begin{equation}
\label{eq:VoronoiCells}
    \mathcal{Q}_k
    \eqdef
    \bigg\{
        x\in\mathbb R^d:\,
        \|x-z_k\|_\infty
        =
        \min_{j\in[K]}
        \|x-z_j\|_\infty
    \bigg\}
    \setminus
    \bigcup_{i<k}\mathcal{Q}_i.
\end{equation}
If
$
    0<\rho<\delta_{\mathrm{sep}}/2,
$
then, for every $k\in[K]$,
$
    [0,1]^d\cap Q_d(z_k,\rho)\subseteq\mathcal{Q}_k.
$
Moreover, under~\eqref{eq:ahlfors_uniform_local_mass}
\begin{equation}
\label{eq:adaptive_local_cell_mass_lower_bound}
    \mathbb{P}_X
    \bigl(
        \mathcal{Q}_k\cap Q_d(z_k,\rho)
    \bigr)
    \ge
    \underline p_\rho
    \ge
    c_-\rho^d
    >
    0
    \qquad \mbox{ for each } k\in[K]
.
\end{equation}
\end{lemma}
\begin{proof}[Proof of Lemma~\ref{lem:adaptive_local_ball_in_cell}]
Fix $k\in[K]$ and $
    x\in[0,1]^d\cap Q_d(z_k,\rho)
$.  Then $\|x-z_k\|_\infty\le\rho$ and for each $j\neq k$, the reverse triangle inequality implies that
\begin{align}
\label{eq:adaptive_point_far_from_other_centre}
    \|x-z_j\|_\infty
&\ge
    \|z_j-z_k\|_\infty-\|x-z_k\|_\infty
\ge
    \delta_{\mathrm{sep}}-\rho
>
    \rho
\ge
    \|x-z_k\|_\infty.
\end{align}
Thus $z_k$ is the unique nearest centre to $x$, so the tie-breaking in~\eqref{eq:VoronoiCells} is irrelevant and $x\in\mathcal{Q}_k$.  Hence $
    [0,1]^d\cap Q_d(z_k,\rho)\subseteq\mathcal{Q}_k
$.  Since $\mathbb{P}_X$ is supported on $[0,1]^d$, this inclusion and~\eqref{eq:ahlfors_uniform_local_mass}, evaluated at $z=z_k$, give $
    \mathbb{P}_X(\mathcal{Q}_k\cap Q_d(z_k,\rho))
    \ge
    \mathbb{P}_X([0,1]^d\cap Q_d(z_k,\rho))
    \ge
    \underline p_\rho
    \ge
    c_-\rho^d
$.
\end{proof}

The next lemma combines the positive localized cell mass with noise concentration to certify that Algorithm~\ref{alg:STEP_1_PURIFY} returns, with high probability, sufficiently populated local averages that uniformly approximate the noiseless values $f(\bar{Q}_k)$.

\begin{lemma}[Algorithm~\ref{alg:STEP_1_PURIFY} Adaptive Purification]
\label{lem:Convergence_Samples}
Let
$
    \mathcal{D}^{(0)}
$
and
$
    \mathcal{D}^{(1)}
$
be the independent half-samples in
\eqref{eq:adaptive_two_half_samples}.  Let
$
    \mathcal{F}_0
    =
    \sigma(\{X_n^{(0)}\}_{n=1}^N)
$
and let
$
    \mathcal{Q}_1,\dots,\mathcal{Q}_K
$
and
$
    \bar{Q}_1,\dots,\bar{Q}_K
$
be the pilot-measurable cells and centres produced by
Algorithm~\ref{alg:STEP_1_PURIFY}.  For $k\in[K]$, define
\begin{equation}
\label{eq:adaptive_local_index_and_count}
    I_{k,\rho}
    \eqdef
    \left\{
        n\in[N]:
        X_n^{(1)}
        \in
        \mathcal{Q}_k
        \cap
        Q_d(\bar{Q}_k,\rho)
    \right\},
    \qquad
    N_{k,\rho}
    \eqdef
    \#I_{k,\rho},
\end{equation}
and define the globally defined average by
\begin{equation}
\label{eq:adaptive_globally_defined_average}
    \bar Y_{k,\rho}
    \eqdef
    \begin{cases}
    \displaystyle
    \frac{1}{N_{k,\rho}}
    \sum_{n\in I_{k,\rho}}
    Y_n^{(1)},
    &
    N_{k,\rho}>0,
    \\[2ex]
    0,
    &
    N_{k,\rho}=0.
    \end{cases}
\end{equation}
Assume
$
    0<\rho<\delta_{\mathrm{sep}}/2
$
and
$
    \underline{p}_\rho>0.
$
Set
$
    w_\rho=\omega(\rho)
$
and suppose
$
    w_\rho>0.
$
For $0<\delta\le1$, define
$
    M_{\delta}
\eqdef 
    \left\lceil
        \frac{
            4D\sigma_C^2
        }{
            w_\rho^2
        }
        \log
        \left(
            \frac{
                4DK
            }{
                \delta
            }
        \right)
    \right\rceil
$.  
If $
    N
\ge
    \big\lceil
        \max
        \big\{
            \tfrac{2M_{\delta}}{\underline{p}_\rho}
            ,
            \tfrac{8}{\underline{p}_\rho}
            \log
            \big(
                \frac{2K}{\delta}
            \big)
        \big\}
    \big\rceil
$
then
\begin{equation}
\label{eq:adaptive_purification_conditional_success}
\begin{aligned}
\mathbb{P}
\left(
    \min_{k\in[K]}N_{k,\rho}
    \ge
    M_{\delta}
    \text{ and }
    \max_{k\in[K]}
    \left\|
        \bar Y_{k,\rho}
        -
        f(\bar{Q}_k)
    \right\|_2
    \le
    2w_\rho
    \,\middle|\,
    \mathcal{F}_0
\right)
\ge
1-\delta.
\end{aligned}
\end{equation}
\end{lemma}

\begin{proof}[Proof of Lemma~\ref{lem:Convergence_Samples}]
Condition on
$
    \mathcal{F}_0.
$
The centres and cells are then deterministic.  By
Lemma~\ref{lem:adaptive_local_ball_in_cell}, for every $k\in[K]$,
$
    p_{k,\rho}
    \eqdef
    \mathbb{P}_X
    \left(
        \mathcal{Q}_k
        \cap
        Q_d(\bar{Q}_k,\rho)
    \right)
    \ge
    \underline{p}_\rho
$.  Since the averaging inputs are independent of
$
    \mathcal{F}_0
$
and are i.i.d.\ with law
$
    \mathbb{P}_X,
$
the count
$
    N_{k,\rho}
$
is conditionally binomial with parameters
$
    N
$
and
$
    p_{k,\rho}.
$
Its conditional mean is
$
    \mu_k
    =
    Np_{k,\rho}
    \ge
    N\underline{p}_\rho
$.   Since we have assumed that
$N
    \ge
    \lceil
    \tfrac{2M_{\delta}}{\underline{p}_\rho}
    \rceil$ then
$
    N\underline{p}_\rho\ge2M_{\delta}
$; thus,
$
    \mu_k\ge2M_{\delta}
$
and
\begin{equation}
\label{eq:adaptive_low_count_event_inclusion}
    \{N_{k,\rho}<M_{\delta}\}
    \subseteq
    \{N_{k,\rho}<\mu_k/2\}.
\end{equation}
The multiplicative Chernoff inequality for a binomial random variable gives
\begin{align}
\label{eq:adaptive_single_cell_chernoff}
    \mathbb{P}
    \left(
        N_{k,\rho}<M_{\delta}
        \,\middle|\,
        \mathcal{F}_0
    \right)
\le
    \mathbb{P}
    \left(
        N_{k,\rho}<\frac{\mu_k}{2}
        \,\middle|\,
        \mathcal{F}_0
    \right)
\le
    \exp
    \left(
        -\frac{\mu_k}{8}
    \right)
\le
    \exp
    \left(
        -\frac{N\underline{p}_\rho}{8}
    \right).
\end{align}
A union bound over $k\in[K]$ and since 
$
N
    \ge
    \lceil
            \tfrac{8}{\underline{p}_\rho}
            \log
            \big(
                \tfrac{2K}{\delta}
            \big)
    \rceil
$ then we have that
\begin{equation}
\label{eq:adaptive_all_cell_occupancy}
    \mathbb{P}
    \left(
        \min_{k\in[K]}N_{k,\rho}<M_{\delta}
        \,\middle|\,
        \mathcal{F}_0
    \right)
    \le
    \frac{\delta}{2}.
\end{equation}
Define the $\sigma$-algebra $
    \mathcal{F}_X
\eqdef
    \sigma
    \left(\{X_n^{(i)}\}_{n\in [N],\,i\in \{0,1\}}
    \right)
$.
Fix $k\in[K]$.  On
$
    \{N_{k,\rho}>0\},
$
the regression equation gives
\begin{align}
\label{eq:adaptive_bias_noise_decomposition}
&
\left\|
    \bar Y_{k,\rho}
    -
    f(\bar{Q}_k)
\right\|_2
\le
\frac{1}{N_{k,\rho}}
\sum_{n\in I_{k,\rho}}
\left\|
    f(X_n^{(1)})
    -
    f(\bar{Q}_k)
\right\|_2
+
\left\|
    \frac{1}{N_{k,\rho}}
    \sum_{n\in I_{k,\rho}}
    \varepsilon_n^{(1)}
\right\|_2.
\end{align}
For $n\in I_{k,\rho}$,
$
    \|X_n^{(1)}-\bar{Q}_k\|_\infty\le\rho.
$
Uniform continuity therefore yields
\begin{equation}
\label{eq:adaptive_deterministic_bias_bound}
    \frac{1}{N_{k,\rho}}
    \sum_{n\in I_{k,\rho}}
    \left\|
        f(X_n^{(1)})
        -
        f(\bar{Q}_k)
    \right\|_2
    \le
    \omega(\rho)
    =
    w_\rho.
\end{equation}
It remains to control the noise average with its random index set.  Conditional
on
$
    \mathcal{F}_X,
$
the set
$
    I_{k,\rho}
$
is deterministic.  For each nonempty deterministic
$
    I\subseteq[N],
$
the empirical-noise concentration lemma, applied with
$
    r=2
$
and
$
    t=w_\rho,
$
gives
\begin{equation}
\label{eq:adaptive_deterministic_index_noise_bound}
\begin{aligned}
\mathbb{P}
\left(
    \left.
    \left\|
        \frac{1}{\#I}
        \sum_{n\in I}
        \varepsilon_n^{(1)}
    \right\|_2
    >
    w_\rho
    \,\right|\,
    \mathcal{F}_X
\right)
\le
2D
\exp
\left(
    -
    \frac{
        \#I\,w_\rho^2
    }{
        4D\sigma_C^2
    }
\right).
\end{aligned}
\end{equation}
The events
$
    \{I_{k,\rho}=I\},
$
indexed by
$
    I\subseteq[N],
$
form a finite
$
    \mathcal{F}_X
$
-measurable partition.  Multiplying
\eqref{eq:adaptive_deterministic_index_noise_bound} by
$
    I_{\{I_{k,\rho}=I\}}
$
and summing over all nonempty $I$ gives, almost surely on
$
    \{N_{k,\rho}>0\},
$
\begin{equation}
\label{eq:adaptive_random_index_noise_bound}
\begin{aligned}
&
\mathbb{P}
\left(
    \left.
    \left\|
        \frac{1}{N_{k,\rho}}
        \sum_{n\in I_{k,\rho}}
        \varepsilon_n^{(1)}
    \right\|_2
    >
    w_\rho
    \,\right|\,
    \mathcal{F}_X
\right)
\le
2D
\exp
\left(
    -
    \frac{
        N_{k,\rho}w_\rho^2
    }{
        4D\sigma_C^2
    }
\right).
\end{aligned}
\end{equation}
On the occupancy event
$
    \mathcal{E}_{\operatorname{occ}}
    =
    \{\min_kN_{k,\rho}\ge M_{\delta}\},
$
equations
\eqref{eq:adaptive_bias_noise_decomposition},
\eqref{eq:adaptive_deterministic_bias_bound}, and
\eqref{eq:adaptive_random_index_noise_bound} imply
\begin{align}
\label{eq:adaptive_single_cell_purification_failure}
\mathbb{P}
\left(
    \left.
    \left\|
        \bar Y_{k,\rho}
        -
        f(\bar{Q}_k)
    \right\|_2
    >
    2w_\rho
    \,\right|\,
    \mathcal{F}_X
\right)
\le
2D
\exp
\left(
    -
    \frac{
        M_{\delta} w_\rho^2
    }{
        4D\sigma_C^2
    }
\right)
\le
\frac{\delta}{2K}.
\end{align}
The final inequality follows directly from the definition of
$
    M_{\delta}.
$
A union bound over $k\in[K]$ gives, almost surely on
$
    \mathcal{E}_{\operatorname{occ}},
$
\begin{equation}
\label{eq:adaptive_all_cell_purification_failure}
\mathbb{P}
\left(
    \left.
    \max_{k\in[K]}
    \left\|
        \bar Y_{k,\rho}
        -
        f(\bar{Q}_k)
    \right\|_2
    >
    2w_\rho
    \,\right|\,
    \mathcal{F}_X
\right)
\le
\frac{\delta}{2}.
\end{equation}
Define the simultaneous purification event, cf.~\eqref{eq:adaptive_all_cell_purification_failure}, by
$
    \mathcal{E}_{\operatorname{pur}}
    \eqdef
    \big\{
        \max_{k\in[K]}
        \|
            \bar Y_{k,\rho}-f(\bar Q_k)
        \|_2
        \le 
        2w_\rho
    \big\}
$.
Since
$
    \mathcal{E}_{\operatorname{occ}}\in\mathcal{F}_X,
$
the tower property gives
$
\mathbb{P}
\left(
    \mathcal{E}_{\operatorname{occ}}
    \cap
    \mathcal{E}_{\operatorname{pur}}^c
    \,\middle|\,
    \mathcal{F}_0
\right)
=
\mathbb{E}
\left[
    I_{\mathcal{E}_{\operatorname{occ}}}
    \mathbb{P}
    \left(
        \mathcal{E}_{\operatorname{pur}}^c
        \,\middle|\,
        \mathcal{F}_X
    \right)
    \,\middle|\,
    \mathcal{F}_0
\right]
\le
\frac{\delta}{2}
$.  
Combining this with 
\eqref{eq:adaptive_all_cell_occupancy} yields
\[
\mathbb{P}
\left(
    \mathcal{E}_{\operatorname{occ}}
    \cap
    \mathcal{E}_{\operatorname{pur}}
    \,\middle|\,
    \mathcal{F}_0
\right)
\ge
1
-
\mathbb{P}
\left(
    \mathcal{E}_{\operatorname{occ}}^c
    \,\middle|\,
    \mathcal{F}_0
\right)
-
\mathbb{P}
\left(
    \mathcal{E}_{\operatorname{occ}}
    \cap
    \mathcal{E}_{\operatorname{pur}}^c
    \,\middle|\,
    \mathcal{F}_0
\right)
\ge
1-\delta
\]
and concludes our proof.
\end{proof}

The next lemma shows that, once Algorithm~\ref{alg:STEP_1_PURIFY} has produced accurate denoised labels, Algorithm~\ref{alg:STEP_2_CLUSTER} compresses them into a tied-value representation that preserves the recovered values while reducing them to a smaller collection of shared representatives.

\begin{lemma}[Tied-Value Compression by Algorithm~\ref{alg:STEP_2_CLUSTER}]
\label{lem:adaptive_tied_value_compression}
Let
$
    \{(z_k,y_k)\}_{k=1}^K
$
be a nonempty finite dataset and run
Algorithm~\ref{alg:STEP_2_CLUSTER} with separation input
$
    8\vartheta,
$
where
$
    \vartheta>0.
$
Let
$
    \widehat c:[K]\to[\hat{\kappa}]
$
be the returned assignment and let
$
    \widehat A_1,\dots,\widehat A_{\hat{\kappa}}
$
be the returned representative values.  Then
\begin{equation}
\label{eq:algorithm2_direct_compression}
    \left\|
        \widehat A_{\widehat c(k)}
        -
        y_k
    \right\|_2
    \le
    \vartheta
    \qquad
    \text{for each }k\in[K].
\end{equation}
Moreover, if $
    \max_{k\in[K]}
    \|y_k-f(z_k)\|_2
    \le
    \eta$
then $
    \max_{k\in[K]}
    \left\|
        \widehat A_{\widehat c(k)}
        -
        f(z_k)
    \right\|_2
    \le
    \vartheta+\eta$.
\end{lemma}
\begin{proof}[Proof of Lemma~\ref{lem:adaptive_tied_value_compression}]
Algorithm~\ref{alg:STEP_2_CLUSTER} begins with a finite dataset and removes a nonempty cluster at every iteration because the selected representative belongs to its own cluster.  It therefore terminates after at most $K$ iterations.  Each input pair is removed at exactly one iteration and is assigned exactly one recovered label.
Fix $k\in[K]$ and let $j=\widehat c(k)$.  At the iteration at which
$
    (z_k,y_k)
$
is removed, the algorithm has selected a representative
$
    (\mathbf z_j,\mathbf y_j)
$
and has defined
$
    \mathcal C_j
    =
    \left\{
        (z,y)
        \text{ remaining at this iteration}:
        \|y-\mathbf y_j\|_2
        \le
        \vartheta
    \right\}
$.
Since
$
    (z_k,y_k)\in\mathcal C_j,
$ we have $
    \|y_k-\mathbf y_j\|_2
    \le
    \vartheta
$.  
The returned class value is
$
    \widehat A_j=\mathbf y_j.
$
Substituting this identity into
$\|y_k-\mathbf y_j\|_2
    \le
    \vartheta$ proves
\eqref{eq:algorithm2_direct_compression}.
Now, if $\max_{k\in[K]}
    \|y_k-f(z_k)\|_2
    \le
    \eta$ then
\begin{align}
\left\|
    \widehat A_{\widehat c(k)}
    -
    f(z_k)
\right\|_2
\le
\left\|
    \widehat A_{\widehat c(k)}
    -
    y_k
\right\|_2
+
\|y_k-f(z_k)\|_2
\nonumber
\le
\vartheta+\eta.
\end{align}
Taking the maximum over $k$ yields our result.
\end{proof}

The next proposition shows that, after the centres and their values have been learned, hard attention sends each input to the value attached to its nearest centre, which is the key step turning the learned representation into an approximation of $f$.

\begin{proposition}[Hard-Attention Routing]
\label{prop:HardAttentionRouting}
Let
$
    \{z_k\}_{k=1}^K\subset [0,1]^d
$
be distinct and let
$
    u_1,\dots,u_K\in\mathbb R^D.
$
Set
$
    K_Z=(z_1,\dots,z_K)
$
and
$
    V_U=(u_1,\dots,u_K).
$
For $x\in[0,1]^d$, define
$
    d_\star(x)
    \eqdef
    \min_{k\in[K]}\|x-z_k\|_\infty
$, and $\mathcal I_\star(x)
    \eqdef
    \left\{
        k\in[K]:
        \|x-z_k\|_\infty=d_\star(x)
    \right\}
$.  
Then
\begin{equation}
\label{eq:adaptive_hard_attention_exact_average}
    \operatorname{Attn}_0(x\mid K_Z,V_U)
    =
    \frac{1}{\#\mathcal I_\star(x)}
    \sum_{k\in\mathcal I_\star(x)}
    u_k.
\end{equation}
If, moreover, $
    \sup_{x\in\operatorname{supp}(\mathbb{P}_X)}
    d_\star(x)
    \le
    R
$ and $
    \max_{k\in[K]}
    \|u_k-f(z_k)\|_2
    \le
    e
$ hold then
\begin{equation}
\label{eq:adaptive_hard_attention_generic_bound}
    \left\|
        \operatorname{Attn}_0(\cdot\mid K_Z,V_U)
        -
        f
    \right\|_{L^\infty(\mathbb{P}_X)}
    \le
    e+\omega(R).
\end{equation}
No assumption that nearest keys share one prescribed source class is
required.
\end{proposition}

\begin{proof}[Proof of Proposition~\ref{prop:HardAttentionRouting}]
Fix
$
    x\in[0,1]^d.
$
For
$
    k\in[K],
$
define the score $
    s_k(x)
    \eqdef
    -\|x-z_k\|_\infty$.  
The maximizer set of the score vector is
\begin{align}
\mathcal M
\left(
    (s_k(x))_{k=1}^K
\right)
=
\operatorname*{argmax}_{k\in[K]}
\left(
    -\|x-z_k\|_\infty
\right)
\nonumber
=
\operatorname*{argmin}_{k\in[K]}
\|x-z_k\|_\infty
\nonumber
=
\mathcal I_\star(x).
\end{align}
By~\eqref{eq:zero_temperature_softmax_definition},
$
    \operatorname{Softmax}_0
    \left(
        (s_j(x))_{j=1}^K
    \right)_k
    =
    \frac{
        I_{\{k\in\mathcal I_\star(x)\}}
    }{
        \#\mathcal I_\star(x)
    }
$.
Multiplication by $V_U$ gives
\begin{align}
\operatorname{Attn}_0(x\mid K_Z,V_U)
&=
\sum_{k=1}^K
u_k
\operatorname{Softmax}_0
\left(
    (s_j(x))_{j=1}^K
\right)_k
=
\frac{1}{\#\mathcal I_\star(x)}
\sum_{k\in\mathcal I_\star(x)}
u_k,
\end{align}
which proves~\eqref{eq:adaptive_hard_attention_exact_average}. Now fix
$
    x\in\operatorname{supp}(\mathbb{P}_X)
$
and
$
    k\in\mathcal I_\star(x).
$
Then
$
    \|x-z_k\|_\infty=d_\star(x)\le R.
$
Since $\max_{k\in[K]}  \|u_k-f(z_k)\|_2 \le  e$, uniform continuity implies that
\begin{align}
\label{eq:adaptive_each_nearest_value_error}
    \|u_k-f(x)\|_2
&\le
    \|u_k-f(z_k)\|_2
    +
    \|f(z_k)-f(x)\|_2
\nonumber
\le
    e+\omega(R).
\end{align}
Using~\eqref{eq:adaptive_hard_attention_exact_average}, and by convexity of the
norm we have
\begin{equation}
\label{eq:sup_up}
\begin{aligned}
\left\|
    \operatorname{Attn}_0(x\mid K_Z,V_U)-f(x)
\right\|_2
& =
\left\|
    \frac{1}{\#\mathcal I_\star(x)}
    \sum_{k\in\mathcal I_\star(x)}
    \bigl(
        u_k-f(x)
    \bigr)
\right\|_2
\\
&\le
\frac{1}{\#\mathcal I_\star(x)}
\sum_{k\in\mathcal I_\star(x)}
\|u_k-f(x)\|_2
\\
&\le
e+\omega(R).
\end{aligned}
\end{equation}
Taking the essential supremum across~\eqref{eq:sup_up} proves
\eqref{eq:adaptive_hard_attention_generic_bound}.
\end{proof}

\begin{theorem}[Certified Recovery from Noisy Samples: $L^{\infty}$-Variant]
\label{thrm:Main_Linfty}
Under Setting~\ref{s:TheSetting_pt1} and Algorithm configuration~\ref{algconfig:adaptive_recovery}, fix $0<\delta\le1$, and let
$
    M_{\delta}
\eqdef
    \left\lceil
        \frac{4D\sigma_C^2}{\omega(\rho)^2}
        \log\left(\frac{8DK}{\delta}\right)
    \right\rceil.
$
If
$
    0<\rho\le\tfrac{\sqrt d}{q},
$
$
    2\omega(\rho)\le\omega(\tfrac{\sqrt d}{q}),
$
and
\begin{equation}
\label{eq:adaptive_hard_source_sample_bound}
N
    \ge
    \left\lceil
        \max
        \left\{
            \frac{1}{p_{\operatorname{cov}}}
            \log\left(\frac{2(4q)^d}{\delta}\right),
            \frac{2M_{\delta}}{\underline p_\rho},
            \frac{8}{\underline p_\rho}
            \log\left(\frac{4K}{\delta}\right)
        \right\}
    \right\rceil,
\end{equation}
then
$
    \hat{\mathcal{T}}_0
    \eqdef
    \operatorname{Attn}_0(\cdot\mid\widehat K,\hat V)
$
satisfies
\begin{equation}
\label{eq:adaptive_source_hard_attention_rate}
    \mathbb{P}
    \left(
        \left\|
            \hat{\mathcal{T}}_0-f
        \right\|_{L^\infty(\mathbb{P}_X)}
        \le
        3\omega\left(\frac{\sqrt d}{q}\right)
    \right)
    \ge
    1-\delta.
\end{equation}
\end{theorem}

\begin{proof}[Proof of Theorem~\ref{thrm:Main_Linfty}]
Apply Lemma~\ref{lem:ahlfors_sampling_consequences} and let
$
    \mathcal U=\{U_\ell\}_{\ell=1}^L,
$
$p_{\operatorname{cov}}$, $\underline p_\rho$, and
$\mathcal{E}_{\operatorname{pilot}}$
be as defined there.  The first term in~\eqref{eq:adaptive_hard_source_sample_bound} and the pilot-coverage estimate give
$
    L e^{-Np_{\operatorname{cov}}}\le\delta/2,
$
and hence
\begin{equation}
\label{eq:adaptive_hard_pilot_probability}
    \mathbb{P}(\mathcal{E}_{\operatorname{pilot}})
    \ge
    1-\frac{\delta}{2}.
\end{equation}
On $\mathcal{E}_{\operatorname{pilot}}$, Lemma~\ref{lem:adaptive_pilot_population_cover} yields
\begin{equation}
\label{eq:adaptive_hard_population_cover}
    \sup_{x\in\operatorname{supp}(\mathbb{P}_X)}
    \min_{k\in[K]}
    \|x-\bar{Q}_k\|_\infty
    \le
    \frac{1}{2q}.
\end{equation}
We next verify the hypotheses needed for purification.  By Algorithm configuration~\ref{algconfig:adaptive_recovery}, the selected centres belong to
$
    \mathfrak Z_{K,\delta_{\mathrm{sep}}}
$
with
$
    \delta_{\mathrm{sep}}=1/(2q)
$
and
$
    0<\rho<1/(4q)=\delta_{\mathrm{sep}}/2.
$
Therefore Lemma~\ref{lem:adaptive_local_ball_in_cell}, together with~\eqref{eq:ahlfors_uniform_local_mass}, gives
$
    \mathbb{P}_X(
        \mathcal{Q}_k\cap Q_d(\bar{Q}_k,\rho)
    )
    \ge
    \underline p_\rho
$
for every $k\in[K]$.  Moreover, since $\rho>0$ and $\omega$ is strictly increasing with $\omega(0)=0$, we have $\omega(\rho)>0$.  Define the purification event
\begin{equation}
\label{eq:adaptive_hard_purification_event}
    \mathcal{E}_{\operatorname{pur}}
    \eqdef
    \left\{
        \min_{k\in[K]}N_{k,\rho}\ge M_{\delta}
        \ \text{and}\
        \max_{k\in[K]}
        \|\bar Y_{k,\rho}-f(\bar{Q}_k)\|_2
        \le
        2\omega(\rho)
    \right\}.
\end{equation}
When Lemma~\ref{lem:Convergence_Samples} is applied with failure probability $\delta/2$, its required local count is exactly the present $M_{\delta}$, since
$
    \log(4DK/(\delta/2))
    =
    \log(8DK/\delta).
$
Likewise, the remaining two terms in~\eqref{eq:adaptive_hard_source_sample_bound} are precisely the corresponding sample-size requirements with $\delta$ replaced by $\delta/2$.  Consequently,
$
    \mathbb{P}
    \big(
        \mathcal{E}_{\operatorname{pur}}
        \,|\,
        \mathcal F_0
    \big)
\ge
    1-\tfrac{\delta}{2}
$.  On $\mathcal{E}_{\operatorname{pur}}$, every $N_{k,\rho}$ is positive, so Algorithm~\ref{alg:STEP_1_PURIFY} places every pair
$
    (\bar{Q}_k,\bar Y_{k,\rho})
$
in $\mathcal D^\star$.  Thus Algorithm~\ref{alg:STEP_2_CLUSTER} is applied to all $K$ learned centres.  Lemma~\ref{lem:adaptive_tied_value_compression}, with
$
    z_k=\bar{Q}_k,
$
$
    y_k=\bar Y_{k,\rho},
$
and
$
    \eta=2\omega(\rho),
$
therefore gives
\begin{equation}
\label{eq:adaptive_hard_trained_value_accuracy}
    \max_{k\in[K]}
    \left\|
        \widehat A_{\widehat c(k)}
        -
        f(\bar{Q}_k)
    \right\|_2
    \le
    \vartheta+2\omega(\rho).
\end{equation}
Set
$
    u_k\eqdef\widehat A_{\widehat c(k)}
$
for $k\in[K]$.  After indexing the columns of $\widehat K$ by the corresponding centres $\bar{Q}_k$, the effective values matrix $\hat V$ has columns $u_k$.  Hence, on
$
    \mathcal{E}_{\operatorname{pilot}}
    \cap
    \mathcal{E}_{\operatorname{pur}},
$
Proposition~\ref{prop:HardAttentionRouting}, applied with
$
    R=1/(2q)
$
by~\eqref{eq:adaptive_hard_population_cover} and
$
    e=\vartheta+2\omega(\rho)
$
by~\eqref{eq:adaptive_hard_trained_value_accuracy}, yields
\begin{equation}
\label{eq:adaptive_hard_pre_rate}
    \left\|
        \hat{\mathcal{T}}_0-f
    \right\|_{L^\infty(\mathbb{P}_X)}
    \le
    \vartheta
    +
    2\omega(\rho)
    +
    \omega\left(\frac{1}{2q}\right).
\end{equation}
Set $
    W_q\eqdef\omega(\tfrac{\sqrt d}{q})
$; by Algorithm configuration~\ref{algconfig:adaptive_recovery}, $
    \vartheta\le W_q,
$ while the scale assumption gives
$
    2\omega(\rho)\le W_q  
$. Finally,
$
    1/(2q)\le\sqrt d/q,
$ so monotonicity of $\omega$ gives
$
    \omega(1/(2q))\le W_q
$.  Substituting these three bounds into~\eqref{eq:adaptive_hard_pre_rate} shows that, on
$
    \mathcal{E}_{\operatorname{pilot}}
    \cap
    \mathcal{E}_{\operatorname{pur}}
$,
\begin{equation}
\label{eq:adaptive_hard_deterministic_rate}
    \|
        \hat{\mathcal{T}}_0-f
    \|_{L^\infty(\mathbb{P}_X)}
\le
    3\,
    W_q
=
    3\,
    \omega\big(\tfrac{\sqrt d}{q}\big).
\end{equation}
It remains only to bound the probability of this intersection.  Since
$
    \mathcal{E}_{\operatorname{pilot}}\in\mathcal F_0,
$
the tower property, our lower bound $    \mathbb{P}
    \big(
        \mathcal{E}_{\operatorname{pur}}
        \,|\,
        \mathcal F_0
    \big)
\ge
    1-\tfrac{\delta}{2}$, and~\eqref{eq:adaptive_hard_pilot_probability} give
\begin{align*}
    \mathbb{P}
    \left(
        \mathcal{E}_{\operatorname{pilot}}
        \cap
        \mathcal{E}_{\operatorname{pur}}
    \right)
=
    \mathbb{E}
    \left[
        I_{\mathcal{E}_{\operatorname{pilot}}}
        \mathbb{P}
        \left(
            \mathcal{E}_{\operatorname{pur}}
            \,\middle|\,
            \mathcal F_0
        \right)
    \right]
\ge
    \left(1-\frac{\delta}{2}\right)
    \mathbb{P}(\mathcal{E}_{\operatorname{pilot}})
\ge
    \left(1-\frac{\delta}{2}\right)^2
\ge
    1-\delta
.
\end{align*}
Together with~\eqref{eq:adaptive_hard_deterministic_rate}, this proves~\eqref{eq:adaptive_source_hard_attention_rate}.
\end{proof}

The total adaptive source sample size in Theorem~\ref{thrm:Main_Linfty}
is $2N$.  As in the finite-temperature result, the conclusion concerns
function recovery and does not identify the adaptive clusters with the
prescribed regular-grid symmetry classes.

\begin{lemma}[Finite-Temperature Perturbation]
\label{lem:adaptive_finite_temperature_covering}
Let $\mathbb P_X$ be a Borel probability measure on $[0,1]^d$.
Let $f:[0,1]^d\to\mathbb R^D$ be $\omega$-uniformly continuous from $\ell^\infty$ on the domain to $\ell^2$ on the codomain. Fix $e\ge0$ and set $B\eqdef e+\omega(1)$.
Let
$
    z_1,\dots,z_K\in[0,1]^d
$
and
$
    u_1,\dots,u_K\in\mathbb R^D
$ be such that there is some $R>0$ satisfying
\begin{equation}
\label{eq:adaptive_finite_key_value_assumption}
    \sup_{x\in\operatorname{supp}(\mathbb{P}_X)}
    \min_{k\in[K]}
    \|x-z_k\|_\infty
    \le
    R
\qquad \mbox{ and }\qquad
    \max_{k\in[K]}
    \|u_k-f(z_k)\|_2
    \le
    e.
\end{equation}
For $\lambda>0$, set
$
    T_{\lambda}(x)
    \eqdef
    \operatorname{Attn}_\lambda
    \big(
        x
        \,|,
        (z_k)_{k=1}^K,
        (u_k)_{k=1}^K
    \big)$ for every
$
    x\in\operatorname{supp}(\mathbb{P}_X)
$
\begin{equation}
\label{eq:adaptive_finite_pointwise_cover_bound}
    \|T_{\lambda}(x)-f(x)\|_2
    \le
    e+\omega(2R)
    +
    BK
    \exp
    \left(
        -\frac{R}{\lambda}
    \right).
\end{equation}
\end{lemma}

\begin{proof}[Proof of Lemma~\ref{lem:adaptive_finite_temperature_covering}]
Fix
$
    x\in\operatorname{supp}(\mathbb{P}_X)
$
and define $
    d_k(x)
    \eqdef
    \|x-z_k\|_\infty$ and $d_\star(x)
    \eqdef
    \min_{j\in[K]}d_j(x)$.  
By~\eqref{eq:adaptive_finite_key_value_assumption} $d_\star(x)\le R$.  
Define the near-key and far-key index sets
\begin{equation}
\label{eq:adaptive_near_far_sets}
\begin{aligned}
    \mathcal N_R(x)
    &\eqdef
    \left\{
        k\in[K]:
        d_k(x)
        \le
        d_\star(x)+R
    \right\},
\\
    \mathcal{F}_R(x)
    &\eqdef
    [K]\setminus\mathcal N_R(x).
\end{aligned}
\end{equation}
If
$
    k\in\mathcal N_R(x),
$
then the inequality $d_\star(x)\le R$ implies that $
    d_k(x)
    \le
    d_\star(x)+R
    \le
    2R$.  
Therefore, we have
\begin{align}
\label{eq:adaptive_near_key_value_error}
    \|u_k-f(x)\|_2
\le
    \|u_k-f(z_k)\|_2
    +
    \|f(z_k)-f(x)\|_2
\le
    e+\omega(2R).
\end{align}
Let
$
    a_k^\lambda(x)
$
denote the finite-temperature softmax weight of key $k$:
\begin{equation}
\label{eq:adaptive_finite_weight}
    a_k^\lambda(x)
    \eqdef
    \frac{
        \exp(-d_k(x)/\lambda)
    }{
        \sum_{j=1}^K
        \exp(-d_j(x)/\lambda)
    }.
\end{equation}
Choose
$
    k_\star\in[K]
$
such that
$
    d_{k_\star}(x)=d_\star(x).
$
The denominator in~\eqref{eq:adaptive_finite_weight} is at least
$
    \exp(-d_\star(x)/\lambda).
$
If
$
    k\in\mathcal{F}_R(x),
$
then
$
    d_k(x)>d_\star(x)+R
$
and therefore
\begin{align}
\label{eq:adaptive_far_weight_bound}
    a_k^\lambda(x)
\le
    \frac{
        \exp(-d_k(x)/\lambda)
    }{
        \exp(-d_\star(x)/\lambda)
    }
=
    \exp
    \left(
        -\frac{d_k(x)-d_\star(x)}{\lambda}
    \right)
\le
    \exp
    \left(
        -\frac{R}{\lambda}
    \right).
\end{align}
Summing across~\eqref{eq:adaptive_far_weight_bound} over at most $K$ far keys gives
\begin{equation}
\label{eq:adaptive_far_mass_bound}
    \sum_{k\in\mathcal{F}_R(x)}
    a_k^\lambda(x)
    \le
    K
    \exp
    \left(
        -\frac{R}{\lambda}
    \right).
\end{equation}
For every $k\in[K]$, since $x,z_k\in[0,1]^d$, \eqref{eq:adaptive_finite_key_value_assumption} and uniform continuity give
\[
    \|u_k-f(x)\|_2
    \le
    \|u_k-f(z_k)\|_2+\|f(z_k)-f(x)\|_2
    \le
    e+\omega(\|z_k-x\|_\infty)
    \le
    e+\omega(1)
    =
    B.
\]
Using
$
    \sum_ka_k^\lambda(x)=1
$
and splitting the sum into the near and far sets,
\begin{align}
\|T_{\lambda}(x)-f(x)\|_2
&=
\left\|
    \sum_{k=1}^K
    a_k^\lambda(x)
    \bigl(
        u_k-f(x)
    \bigr)
\right\|_2
\nonumber
\\
&\le
\sum_{k\in\mathcal N_R(x)}
a_k^\lambda(x)
\|u_k-f(x)\|_2
+
\sum_{k\in\mathcal{F}_R(x)}
a_k^\lambda(x)
\|u_k-f(x)\|_2
\nonumber
\\
&\le
\left(
    e+\omega(2R)
\right)
\sum_{k\in\mathcal N_R(x)}
a_k^\lambda(x)
+
B
\sum_{k\in\mathcal{F}_R(x)}
a_k^\lambda(x)
\nonumber
\\
&\le
e+\omega(2R)
+
BK
\exp
\left(
    -\frac{R}{\lambda}
\right)
\end{align}
where
\eqref{eq:adaptive_near_key_value_error} and
\eqref{eq:adaptive_far_mass_bound} were used.  This proves
\eqref{eq:adaptive_finite_pointwise_cover_bound}.
\end{proof}

We are now ready to deduce Theorem~\ref{thrm:Main_Lp} from Theorem~\ref{thrm:Main_Linfty} together with our finite-temperature perturbation lemma, Lemma~\ref{lem:adaptive_finite_temperature_covering}.
\begin{proof}[Proof of Theorem~\ref{thrm:Main_Lp}]
Let $\mathcal E_{\operatorname{pilot}}$ be the pilot-coverage event defined in~\eqref{eq:ahlfors_pilot_occupancy_event}.  By Lemma~\ref{lem:ahlfors_sampling_consequences} and the first term in~\eqref{eq:adaptive_hard_source_exact_sample_bound},
$
    \mathbb{P}(\mathcal E_{\operatorname{pilot}})
    \ge
    1-\delta/2.
$
On $\mathcal E_{\operatorname{pilot}}$, Lemma~\ref{lem:adaptive_pilot_population_cover} gives
\begin{equation}
\label{eq:adaptive_finite_population_cover}
    \sup_{x\in\operatorname{supp}(\mathbb{P}_X)}
    \min_{k\in[K]}
    \|x-\bar{Q}_k\|_\infty
    \le
    R,
    \qquad
    R\eqdef\frac{1}{2q}.
\end{equation}

We next control the values attached to these centres.  Under Algorithm configuration~\ref{algconfig:adaptive_recovery},
$
    \delta_{\mathrm{sep}}=1/(2q)
$
and
$
    0<\rho<1/(4q)=\delta_{\mathrm{sep}}/2.
$
Moreover,~\eqref{eq:ahlfors_uniform_local_mass} gives $\underline p_\rho>0$, while strict increase of $\omega$ and $\rho>0$ give $\omega(\rho)>0$.  Hence Lemma~\ref{lem:Convergence_Samples} applies.  
Define
\begin{equation}
\label{eq:adaptive_finite_purification_event}
    \mathcal E_{\operatorname{pur}}
    \eqdef
    \left\{
        \min_{k\in[K]}N_{k,\rho}\ge M_{\delta}
        \text{ and }
        \max_{k\in[K]}
        \|\bar Y_{k,\rho}-f(\bar{Q}_k)\|_2
        \le
        2\omega(\rho)
    \right\}.
\end{equation}
Applying Lemma~\ref{lem:Convergence_Samples} with failure probability $\delta/2$ shows that its required count equals the present $M_{\delta}$ because
$
    \log(4DK/(\delta/2))
    =
    \log(8DK/\delta).
$
The remaining two terms in~\eqref{eq:adaptive_hard_source_exact_sample_bound} are exactly the corresponding sample-size requirements with $\delta$ replaced by $\delta/2$.  Therefore
\begin{equation}
\label{eq:adaptive_finite_purification_probability}
    \mathbb{P}
    \left(
        \mathcal E_{\operatorname{pur}}
        \,\middle|\,
        \mathcal F_0
    \right)
    \ge
    1-\frac{\delta}{2}
.
\end{equation}
Work now on
$
    \mathcal E_{\operatorname{pilot}}
    \cap
    \mathcal E_{\operatorname{pur}}.
$
Since $N_{k,\rho}\ge M_{\delta}\ge1$ for every $k$, all $K$ pairs
$
    (\bar{Q}_k,\bar Y_{k,\rho})
$
appear in $\mathcal D^\star$.  Lemma~\ref{lem:adaptive_tied_value_compression}, with
$
    z_k=\bar{Q}_k
$,
$
    y_k=\bar Y_{k,\rho}
$,
and
$
    \eta=2\omega(\rho)
$,
therefore gives
\begin{equation}
\label{eq:adaptive_finite_trained_value_accuracy}
    \max_{k\in[K]}
    \left\|
        \widehat A_{\widehat c(k)}
        -
        f(\bar{Q}_k)
    \right\|_2
    \le
    e,
    \qquad
    e\eqdef\vartheta+2\omega(\rho).
\end{equation}
Set $u_k\eqdef\widehat A_{\widehat c(k)}$ for $k\in[K]$.  Up to the corresponding ordering of the columns, $\widehat K=(\bar{Q}_k)_{k=1}^K$ and $\hat V=(u_k)_{k=1}^K$.  Let
$
    W_q\eqdef\omega(\sqrt d/q).
$
By Algorithm configuration~\ref{algconfig:adaptive_recovery}, $\vartheta\le W_q$, while the hypothesis $\rho\le\sqrt d/q$ and monotonicity of $\omega$ imply $\omega(\rho)\le W_q$.  Consequently, $
    e
=
    \vartheta+2\omega(\rho)
\le
    3W_q
$ and $2R=1/q\le\sqrt d/q$, and hence $
    \omega(2R)
\le
    W_q
$.  Set
$
    B\eqdef e+\omega(1)
$
and
$
    s_q\eqdef\min\{W_q,BK\}.
$
Since $\sqrt d/q>0$ and $\omega$ is strictly increasing with $\omega(0)=0$, we have $W_q>0$.  Moreover, $e\ge\vartheta>0$, so $B>0$ and therefore $s_q>0$.  Since $s_q\le BK$, we have $2BK/s_q\ge2$, and hence
$
    \log(2BK/s_q)>0.
$
Define $
    \lambda_0
    \eqdef
    \tfrac{R}{
        \log\left(\frac{2BK}{s_q}\right)
    }$; 
thus $\lambda_0>0$, and $\lambda_0$ is determined by $q,\rho,\vartheta,K$, and the supplied modulus $\omega$, independently of the realized training sample and without requiring a bound on $\sup_{x\in[0,1]^d}\|f(x)\|_2$.  Fix $0<\lambda\le\lambda_0$.  
Since the map $\lambda\mapsto e^{-R/\lambda}$ is increasing on $(0,\infty)$ and takes values in $(0,1)$, we have
\begin{equation}
\label{eq:adaptive_temperature_exponential_control}
    BK
    \,
    e^{-\tfrac{R}{\lambda}}
    \le
    BK
    \,e^{-\tfrac{R}{\lambda_0}}
    =
    \frac{s_q}{2}
    \le
    \frac{W_q}{2}.
\end{equation}
The hypotheses of Lemma~\ref{lem:adaptive_finite_temperature_covering} now follow from~\eqref{eq:adaptive_finite_population_cover} and~\eqref{eq:adaptive_finite_trained_value_accuracy}.  Hence, for every $x\in\operatorname{supp}(\mathbb{P}_X)$,
\begin{align}
\label{eq:adaptive_finite_pointwise_final}
    \left\|
        \hat{\mathcal{T}}_\lambda(x)-f(x)
    \right\|_2
\le
    e+\omega(2R)
    +
    BK\exp\left(-\frac{R}{\lambda}\right)
\le
    3W_q+W_q+\frac{W_q}{2}
=
    \frac{9}{2}W_q
\le
    5W_q.
\end{align}
Raising~\eqref{eq:adaptive_finite_pointwise_final} to the power $p$ and integrating with respect to $\mathbb{P}_X$ gives
\begin{equation}
\label{eq:adaptive_finite_integrated_final}
    \mathbb{E}_{X\sim\mathbb{P}_X}
    \left[
        \left\|
            \hat{\mathcal{T}}_\lambda(X)-f(X)
        \right\|_2^p
    \right]
    \le
    (5W_q)^p.
\end{equation}
It remains to verify the probability of the event on which the preceding deterministic estimates hold.  Since $\mathcal E_{\operatorname{pilot}}\in\mathcal F_0$, the tower property and~\eqref{eq:adaptive_finite_purification_probability} give
\begin{align*}
    \mathbb{P}
    \left(
        \mathcal E_{\operatorname{pilot}}
        \cap
        \mathcal E_{\operatorname{pur}}
    \right)
=
    \mathbb{E}
    \left[
        I_{\mathcal E_{\operatorname{pilot}}}
        \mathbb{P}
        \left(
            \mathcal E_{\operatorname{pur}}
            \,\middle|\,
            \mathcal F_0
        \right)
    \right]
\ge
    \left(1-\frac{\delta}{2}\right)
    \mathbb{P}(\mathcal E_{\operatorname{pilot}})
\ge
    \left(1-\frac{\delta}{2}\right)^2
\ge
    1-\delta.
\end{align*}
Thus~\eqref{eq:adaptive_finite_integrated_final} holds with probability at least $1-\delta$, simultaneously for every $0<\lambda\le\lambda_0$.
\end{proof}

\subsection{Uniform-Grid Variant}
\label{s:UniformGridVersion}
We now prove the uniform grid variant of our recovery result, which we will use for Theorem~\ref{thrm:2_Scalable_Fine_Tuning}.

\begin{lemma}[Algorithm~\ref{alg:STEP_1_PURIFY__simplified} Fixed-Grid Purification]
\label{lem:Convergence_Samples__fixed_grid}
Let $\mathcal{Q}_{q:d}$ be the deterministic regular grid, assume
$
    \mathbb{P}_X([0,1]^d\setminus\mathcal{G})=0,
$
and define
$
    \mathcal{Q}^\star
=
    \{Q\in\mathcal{Q}_{q:d}:\mathbb{P}_X(Q)>0\}$. 
For $\rho>0$, define
$
        p^\star_\rho
    \eqdef
        \min_{Q\in\mathcal{Q}^\star}
        \mathbb{P}_X(Q_\rho)
$ and $r^\star_\rho
    \eqdef
        \max_{Q\in\mathcal{Q}^\star}
        \sup
        \big(
            \big\{
                \|x-\bar{Q}\|_\infty:
                x\in Q_\rho\cap\operatorname{supp}(\mathbb{P}_X)
            \big\}
            \cup\{0\}
        \big)$.
Let
$
    m=\#\mathcal{Q}^\star,
$
$
    p=p^\star_\rho,
$
$
    w=\omega(r^\star_\rho),
$
and
$
    \sigma_C=\sigma\max\{C,1\}.
$
Fix $0<\delta\le1$ and let $
    M_{\delta}
\eqdef
    \big\lceil
        \tfrac{4D\sigma_C^2}{w^2}
        \log
        \big(
            \tfrac{4Dm}{\delta}
        \big)
    \big\rceil$.  
If $p>0$, $w>0$, and the total fixed-grid sample size $N$ satisfies
\begin{equation}
\label{eq:adaptive_chernoff_sample_bound_appendix}
    N
    \ge
    \left\lceil
        \max
        \left\{
            \frac{2M_{\delta}}{p},
            \frac{8}{p}
            \log
            \left(
                \frac{2m}{\delta}
            \right)
        \right\}
    \right\rceil,
\end{equation}
then $r_{\rho}^{\star} 
    \le
        \min
        \{
            \rho,
            \frac{1}{2q}
        \}$ and $
    \mathbb{P}
    \Big(
        \underset{Q\in\mathcal{Q}^\star}{\min}
        \,
        N_{Q,\rho}
        \ge
        M_{\delta}
        \mbox{ and }
        \underset{Q\in\mathcal{Q}^\star}{\max}
        \|\bar Y_{Q,\rho}-f(\bar{Q})\|_2
        \le
        2w
    \Big)
\ge
    1-\delta
$.  
\end{lemma}

\begin{proof}[Proof of Lemma~\ref{lem:Convergence_Samples__fixed_grid}]
Repeat the occupancy and conditional-noise arguments in the proof of Lemma~\ref{lem:Convergence_Samples} with the deterministic partition $\mathcal{Q}_{q:d}$ and with no pilot sample.  In this case, $\mathcal{Q}^\star$, $p^\star_\rho$, and $r^\star_\rho$ are deterministic, and the single available sample is the averaging sample.  The same Chernoff and union-bound estimates therefore yield
\[
\mathbb{P}     \big(         \min_{Q\in\mathcal{Q}^\star}N_{Q,\rho}         \ge         M_{\delta}         \mbox{ and }         \max_{Q\in\mathcal{Q}^\star}         \|\bar Y_{Q,\rho}-f(\bar{Q})\|_2         \le         2w     \big)     \ge     1-\delta 
\]
If $x\in Q_\rho$, then $\|x-\bar{Q}\|_\infty\le\rho$.  Since $Q$ is an open regular grid cell with midpoint $\bar{Q}$, also $\|x-\bar{Q}\|_\infty<1/(2q)$.  Taking suprema and then the maximum proves $r^\star_\rho
    \le
        \min
        \left\{
            \rho,
            \frac{1}{2q}
        \right\}$.
\end{proof}

In what follows, for any cube $Q\in \mathcal{Q}_{q:d}$ for some fixed $q,d\in \mathbb{N}_+$, we write $
    z_Q
    \eqdef
    f(\bar{Q})-\beta_{c(Q)}
$.  Definition~\ref{defn:symmetric_functions} gives
$\|z_Q\|_\infty\le\eta$.  Therefore
\[
    \|z_Q\|_2^2
    =
    \sum_{r=1}^D |(z_Q)_r|^2
    \le
    \sum_{r=1}^D \|z_Q\|_\infty^2
    =
    D\|z_Q\|_\infty^2
    \le
    D\eta^2.
\]
Taking nonnegative square roots gives
$
    \|f(\bar{Q})-\beta_{c(Q)}\|_2
\le
    \sqrt D\,\eta$; 
thus, only inside the extension, set
$
    \xi\eqdef\sqrt D\,\eta
$.

The next lemma shows that, provided the class values are sufficiently separated relative to the approximation error, the clustering step recovers the correct approximate-symmetry classes up to relabelling.

\begin{lemma}[Margin-Stable Recovery of Approximate Symmetries]
\label{lem:approximate_symmetry_margin_stability}
Fix $d,D,q,\kappa\in\mathbb N_+$, a nonempty family
$\mathcal{Q}_0\subseteq\mathcal{Q}_{q:d}$, a map
$c:\mathcal{Q}_0\to[\kappa]$, vectors
$\beta_1,\dots,\beta_\kappa\in\mathbb R^D$, numbers
$e,\xi\ge0$ and $\Delta>0$, and vectors
$y_Q\in\mathbb R^D$, $Q\in\mathcal{Q}_0$, and write $J_0\eqdef c(\mathcal{Q}_0)$.
\hfill\\
\noindent
Let
$f:[0,1]^d\to\mathbb R^D$ satisfy
$ \|f(\bar{Q})-\beta_{c(Q)}\|_2\le\xi$ for every $Q\in\mathcal{Q}_0$ and suppose that $
    \|y_Q-f(\bar{Q})\|_2\le e
$ for every $Q\in\mathcal{Q}_0$, and assume that $    \|\beta_j-\beta_k\|_2
\ge
    \Delta
$ for all distinct  $j,k\in J_0$.
If $t>0$ satisfies
\begin{equation}
\label{eq:approx_margin_generic_threshold}
    2(e+\xi)\le t<\Delta-2(e+\xi),
\end{equation}
then the following hold:
\begin{enumerate}
    \item[(i)] $
    \|y_Q-y_R\|_2\le t
$ holds if and only if $
    c(Q)=c(R)
$ for each $Q,R\in\mathcal{Q}_0$
\item[(ii)] Algorithm~\ref{alg:STEP_2_CLUSTER}, run on
$\{(\bar{Q},y_Q):Q\in\mathcal{Q}_0\}$ with separation parameter $8t$, returns $\mathcal{C} =
    \bigl\{\{(\bar Q,y_Q):Q\in\mathcal Q_0,\ c(Q)=j\}:j\in J_0\bigr\}
$.
\end{enumerate}
Moreover, a $t>0$ satisfying~\eqref{eq:approx_margin_generic_threshold} exists if $4(e+\xi)<\Delta$.
\end{lemma}

\begin{proof}[Proof of Lemma~\ref{lem:approximate_symmetry_margin_stability}]
Fix $Q,R\in\mathcal{Q}_0$.  We first establish the deterministic separation between pairs belonging to the same represented class and pairs belonging to distinct represented classes.  Suppose that $c(Q)=c(R)$ and write $j\eqdef c(Q)=c(R)$.  By the triangle inequality and the assumptions on $f$ and the observations,
\begin{align}
\label{eq:approx_margin_same_class}
    \|y_Q-y_R\|_2
&\le
    \|y_Q-f(\bar{Q})\|_2
    +
    \|f(\bar{Q})-\beta_j\|_2
    +
    \|\beta_j-f(\bar R)\|_2
    +
    \|f(\bar R)-y_R\|_2
\nonumber\\
&\le
    e+\xi+\xi+e
=
    2(e+\xi).
\end{align}
Thus observations associated with the same class are separated by at most $2(e+\xi)$.  Now suppose that $c(Q)\neq c(R)$.  Since $Q,R\in\mathcal{Q}_0$, both $c(Q)$ and $c(R)$ belong to $J_0$, so the prototype-separation assumption gives
$
    \|\beta_{c(Q)}-\beta_{c(R)}\|_2\ge\Delta.
$
Another application of the triangle inequality yields
\begin{align}
\label{eq:approx_margin_distinct_class}
    \Delta
&\le
    \|\beta_{c(Q)}-\beta_{c(R)}\|_2
\nonumber\\
&\le
    \|\beta_{c(Q)}-f(\bar{Q})\|_2
    +
    \|f(\bar{Q})-y_Q\|_2
    +
    \|y_Q-y_R\|_2
\nonumber\\
&\qquad
    +
    \|y_R-f(\bar R)\|_2
    +
    \|f(\bar R)-\beta_{c(R)}\|_2
\nonumber\\
&\le
    2(e+\xi)+\|y_Q-y_R\|_2.
\end{align}
Consequently,
$
    \|y_Q-y_R\|_2
    \ge
    \Delta-2(e+\xi)
$.  
Let now $t>0$ satisfy~\eqref{eq:approx_margin_generic_threshold}.  If $c(Q)=c(R)$, then~\eqref{eq:approx_margin_same_class} gives
$
    \|y_Q-y_R\|_2\le2(e+\xi)\le t.
$
Conversely, if $c(Q)\neq c(R)$, then~\eqref{eq:approx_margin_distinct_class} and the strict upper inequality in~\eqref{eq:approx_margin_generic_threshold} give
$
    \|y_Q-y_R\|_2
    \ge
    \Delta-2(e+\xi)
    >
    t.
$
Hence, for every $Q,R\in\mathcal{Q}_0$,
\begin{equation}
\label{eq:approx_margin_pairwise_equivalence}
    \|y_Q-y_R\|_2\le t
    \quad\Longleftrightarrow\quad
    c(Q)=c(R),
\end{equation}
which proves~{\rm (i)}.

We next prove~{\rm (ii)}.  Algorithm~\ref{alg:STEP_2_CLUSTER} is run with separation parameter $8t$; therefore, by its definition, the cluster generated by a selected representative $(\bar{Q},y_Q)$ consists precisely of the currently unclustered pairs $(\bar R,y_R)$ satisfying
$
    \|y_R-y_Q\|_2\le (8t)/8=t.
$
By~\eqref{eq:approx_margin_pairwise_equivalence}, this criterion is equivalent to $c(R)=c(Q)$.
\end{proof}

\subsubsection{Fixed-Grid Routing and Approximation}
\label{s:Approx__fixed_grid}

We now derive the $L^{\infty}$-versions of our results before softening them to $L^p$ for finite $p\in [1,\infty)$ via a temperature perturbation argument.

\begin{proposition}[Fixed-Grid Hard-Attention Routing]
\label{prop:HardAttentionRouting__fixed_grid}
Let $\mathcal{Q}_0\subseteq\mathcal{Q}_{q:d}$ be nonempty, let $\widehat c:\mathcal{Q}_0\to[m]$, set $K=(\bar{Q})_{Q\in\mathcal{Q}_0}$, and, for $w_1,\dots,w_m\in\mathbb R^D$, set $V=(w_{\widehat c(Q)})_{Q\in\mathcal{Q}_0}$. Then, for every $x\in\bigcup_{Q\in\mathcal{Q}_0}Q$,
\begin{equation}
\label{eq:hard_attention_routing_conclusion_appendix}
    \operatorname{Attn}_0(x\mid K,V)
    =
    w_{\widehat c(Q_x)},
\end{equation}
where $Q_x$ is the unique cell containing $x$.
\end{proposition}

\begin{proof}[Proof of Proposition~\ref{prop:HardAttentionRouting__fixed_grid}]
Fix $x\in\bigcup_{Q\in\mathcal{Q}_0}Q$. Since the cells in $\mathcal{Q}_{q:d}$ are pairwise disjoint open cubes, there is a unique $Q_x\in\mathcal{Q}_0$ containing $x$. Write $Q_x=Q_i$ for some $i\in\{0,\dots,q-1\}^d$. By~\eqref{eq:cubes}, for every $\ell\in[d]$,
$
    i_\ell/q<x_\ell<(i_\ell+1)/q
$,
and therefore
$
    |x_\ell-(\bar{Q}_x)_\ell|
<
    \tfrac{1}{2q}$.
Hence
$
    \|x-\bar{Q}_x\|_\infty<1/(2q).
$
Now let $Q_j\in\mathcal{Q}_0$ with $Q_j\neq Q_x$. Then $j_\ell\neq i_\ell$ for some $\ell\in[d]$. If $j_\ell>i_\ell$, then
$
    (\bar{Q}_j)_\ell\ge(i_\ell+3/2)/q
$
while $x_\ell<(i_\ell+1)/q$, and therefore
$
    |x_\ell-(\bar{Q}_j)_\ell|>1/(2q).
$
The case $j_\ell<i_\ell$ is identical, using $x_\ell>i_\ell/q$ and $(\bar{Q}_j)_\ell\le(i_\ell-1/2)/q$. Consequently,
\[
    \|x-\bar{Q}_j\|_\infty
    >
    \frac{1}{2q}
    >
    \|x-\bar{Q}_x\|_\infty.
\]
Thus $\bar{Q}_x$ is the unique nearest column of $K$ to $x$. Equivalently, it is the unique maximizer of the score vector
$
    -(\|x-K_i\|_\infty)_i
$.
By~\eqref{eq:zero_temperature_softmax_definition}, $\operatorname{Softmax}_0$ therefore assigns weight one to the column indexed by $Q_x$ and zero to every other column. Using~\eqref{eq:zero_temperature_attention_definition} and the definition of $V$, we obtain
$
    \operatorname{Attn}_0(x\mid K,V)
    =
    V_{Q_x}
    =
    w_{\widehat c(Q_x)},
$
which proves~\eqref{eq:hard_attention_routing_conclusion_appendix}.
\end{proof}

\begin{proposition}[Deterministic Hard-Attention Approximation]
\label{prop:deterministic_Lp_routing_approximation}
In the setting of Proposition~\ref{prop:HardAttentionRouting__fixed_grid}, assume $f$ is $\omega$-uniformly continuous and
$
    \max_{Q\in\mathcal{Q}_0}
    \|w_{\widehat c(Q)}-f(\bar{Q})\|_2
    \le
    \eta
$.  
If $\mathbb{P}_X(\bigcup_{Q\in\mathcal{Q}_0}Q)=1$, then
\begin{equation}
\label{eq:deterministic_hard_approx_appendix}
    \|\operatorname{Attn}_0(\cdot\mid K,V)-f\|_{L^\infty(\mathbb{P}_X)}
    \le
    \eta+
    \omega
    \left(
        \frac{1}{2q}
    \right).
\end{equation}
\end{proposition}

\begin{proof}[Proof of Proposition~\ref{prop:deterministic_Lp_routing_approximation}]
For almost every $x$, Proposition~\ref{prop:HardAttentionRouting__fixed_grid} gives
$
    \operatorname{Attn}_0(x\mid K,V)=w_{\widehat c(Q_x)}$; thus
\begin{equation}
\label{eq:supmerup}
    \|\operatorname{Attn}_0(x\mid K,V)-f(x)\|_2
    \le
    \|w_{\widehat c(Q_x)}-f(\bar{Q}_x)\|_2
    +
    \|f(\bar{Q}_x)-f(x)\|_2
    \le
    \eta+
    \omega(1/(2q)).
\end{equation}
Taking the supremum, in $x$, over $[0,1]^d$ across~\eqref{eq:supmerup} proves~\eqref{eq:deterministic_hard_approx_appendix}.
\end{proof}

\begin{theorem}[Approximate-Symmetry Fixed-Grid Recovery]
\label{thrm:Main1_Lp_fixedGridVersion}
Fix $\eta\ge0$. Assume Setting~\ref{s:TheSetting_pt1} and instead of Setting~\ref{s:TheSetting_pt1}~(iv) we assume that: there are $\beta_1,\dots,\beta_\kappa\in\mathbb R^D$ such that
\[
    \|f(\bar Q)-\beta_k\|_\infty\le\eta,
    \qquad k\in[\kappa],\quad Q\in\mathcal S_k
\]
satisfying~\eqref{eq:pretraining_separation} whenever $\kappa\ge2$. 
\hfill\\
Under Algorithm configuration~\ref{algconfig:fixed_grid}, fix $1\le p<\infty$.
Let $
    \mathcal{Q}^\star
    \eqdef
    \{Q\in\mathcal{Q}_{q:d}:\mathbb{P}_X(Q)>0\}$, $
    J_\star\eqdef c(\mathcal{Q}^\star)$, $\kappa_\star\eqdef\#J_\star$, 
and let $p_\rho^\star$ and $r_\rho^\star$ be as in Lemma~\ref{lem:Convergence_Samples__fixed_grid}; write $\xi\eqdef\sqrt D\,\eta$,  $e_\rho\eqdef2\omega(r_\rho^\star)$,
    and $W_q\eqdef\omega\left(\frac{\sqrt d}{q}\right)
$.  If $p_\rho^\star>0$, $\omega(r_\rho^\star)>0$, for all distinct $j,k\in J_\star$: $
    \|\beta_j-\beta_k\|_2\ge\Delta
$ and 
\begin{equation}
\label{eq:approx_source_margin_condition}
    e_\rho+\xi
    =
    2\omega(r_\rho^\star)+\sqrt D\,\eta
    \le
    \frac{\Delta}{16}
\end{equation}
then there is a constant $c_0>0$ such that, for any $0<\delta\le 1$, if
\begin{equation}
\label{eq:approx_source_sample_condition}
    N
\ge
    \frac{c_0}{p_\rho^\star}
    \left(
        1+
        \frac{\bar\sigma^2}{\omega(r_\rho^\star)^2}
    \right)
    \log
    \left(
        \frac{4D\#\mathcal{Q}^\star}{\delta}
    \right),
\end{equation}
then there is a $\lambda_0^{(\eta)}>0$ such that for every $0<\lambda\le \lambda_0^{(\eta)}$ the model 
$
    \hat{\mathcal{T}}_\lambda(x)
    \eqdef
    \operatorname{Attn}_\lambda
    (x\mid\widehat K,\hat{V})$
satisfies the following with probability at least $1-\delta$
\begin{enumerate}
    \item[(i)] $    \mathbb{E}_{X\sim\mathbb{P}_X}
    \left[
        \|\hat{\mathcal{T}}_\lambda(X)-f(X)\|_2^p
    \right]^{1/p}
\le 
    5\omega\left(\frac{\sqrt d}{q}\right)
    +
    2\sqrt D\,\eta$
    \item[(ii)] There is a bijection
$\varpi_\star:[\hat{\kappa}]\to J_\star$ such that, for each $Q\in\mathcal{Q}^\star$ $\varpi_\star(\widehat c(Q))=c(Q)$ and
\begin{equation}
\label{eq:approx_source_representative_value_bound}
    \max_{Q\in\mathcal{Q}^\star}
    \|\widehat A_{\widehat c(Q)}-f(\bar{Q})\|_2
\le
    2\omega(r_\rho^\star)+2\xi.
\end{equation}
\end{enumerate}
\end{theorem}

\begin{proof}[Proof of Theorem~\ref{thrm:Main1_Lp_fixedGridVersion}]
Set $m\eqdef\#\mathcal{Q}^\star$ and $w\eqdef\omega(r_\rho^\star)$.  We first verify that the sample-size condition implies the fixed-grid purification event required below. The upper Ahlfors bound in~\eqref{eq:Ahlfors} implies $\mathbb{P}_X\ll\mathcal{L}^d$, while $[0,1]^d\setminus\mathcal{G}$ is a finite union of grid-cell boundaries and hence is $\mathcal{L}^d$-null. Thus $\mathbb{P}_X([0,1]^d\setminus\mathcal{G})=0$. With $\sigma_C\eqdef\sigma\max\{C,1\}$, Lemma~\ref{lem:Convergence_Samples__fixed_grid} uses
$
    M_\delta
    =
    \left\lceil
        4D\sigma_C^2w^{-2}
        \log(4Dm/\delta)
    \right\rceil.
$
Since $\sigma_C^2\le\max\{C,1\}^2\bar\sigma^2$, $\log(2m/\delta)\le\log(4Dm/\delta)$, and $\log(4Dm/\delta)\ge\log4$, there is a constant $c_0=c_0(C,D)>0$ such that~\eqref{eq:approx_source_sample_condition} implies the sample-size condition of Lemma~\ref{lem:Convergence_Samples__fixed_grid}.  Hence, with probability at least $1-\delta$,
\begin{equation}
\label{eq:approx_source_purification_event}
    \mathcal E_{\operatorname{pur}}
    \eqdef
    \left\{
        \min_{Q\in\mathcal{Q}^\star}N_{Q,\rho}\ge M_\delta
        \ \text{and}\
        \max_{Q\in\mathcal{Q}^\star}
        \|\bar Y_{Q,\rho}-f(\bar{Q})\|_2
        \le e_\rho
    \right\}
\end{equation}
occurs, where $e_\rho=2\omega(r_\rho^\star)$.  In particular, $M_\delta\ge1$, so every active cell is represented in the virtual dataset on $\mathcal E_{\operatorname{pur}}$.

We next record the deterministic approximate-symmetry error.  Since $f$ is $(\mathcal S,\eta)$-symmetric,
$
    \|f(\bar{Q})-\beta_{c(Q)}\|_\infty\le\eta
$
for every $Q\in\mathcal{Q}_{q:d}$, and therefore
$
    \|f(\bar{Q})-\beta_{c(Q)}\|_2
    \le\sqrt D\,\eta
    =
    \xi.
$
Work henceforth on $\mathcal E_{\operatorname{pur}}$.  Apply Lemma~\ref{lem:approximate_symmetry_margin_stability} with
$
    \mathcal{Q}_0=\mathcal{Q}^\star,
$
$
    y_Q=\bar Y_{Q,\rho},
$
$
    e=e_\rho,
$
and $t=\Delta/8$.  Indeed,~\eqref{eq:approx_source_margin_condition} gives
$
    e_\rho+\xi\le\Delta/16,
$
so
$
    2(e_\rho+\xi)\le\Delta/8
$
and
$
    \Delta/8<\Delta-2(e_\rho+\xi).
$
Thus~\eqref{eq:approx_margin_generic_threshold} holds.  Since Algorithm~\ref{alg:STEP_2_CLUSTER} is run with separation input $\Delta=8t$, Lemma~\ref{lem:approximate_symmetry_margin_stability} shows that, for every representative-selection order, its recovered clusters are exactly
$
    \{\{(\bar Q,\bar Y_{Q,\rho}):Q\in\mathcal{Q}^\star,c(Q)=j\}:j\in J_\star\}.
$
Consequently, $\hat{\kappa}=\kappa_\star$.  For each $r\in[\hat{\kappa}]$, let $\varpi_\star(r)$ be the unique $j\in J_\star$ whose true class is the recovered class indexed by $r$.  Then $\varpi_\star:[\hat{\kappa}]\to J_\star$ is a bijection and
$
    \varpi_\star(\widehat c(Q))=c(Q)
$
for every $Q\in\mathcal{Q}^\star$.

It remains to control the learned representative values.  Fix $Q\in\mathcal{Q}^\star$, put $j\eqdef\widehat c(Q)$, and let $Q_j$ be the cell selected as the representative of recovered class $j$.  Then $\widehat A_j=\bar Y_{Q_j,\rho}$ and, by exact recovery of the represented partition, $c(Q_j)=c(Q)$.  Hence
\begin{align}
    \|\widehat A_{\widehat c(Q)}-f(\bar{Q})\|_2
&\le
    \|\bar Y_{Q_j,\rho}-f(\bar{Q}_j)\|_2
    +
    \|f(\bar{Q}_j)-\beta_{c(Q)}\|_2
    +
    \|\beta_{c(Q)}-f(\bar{Q})\|_2
\nonumber\\
&\le
    e_\rho+2\xi
=
    2\omega(r_\rho^\star)+2\xi.
\end{align}
Taking the maximum over $Q\in\mathcal{Q}^\star$ proves~\eqref{eq:approx_source_representative_value_bound}.

We finally prove the finite-temperature recovery bound.  Under the Ahlfors-regular sampling assumption, every grid cell has positive mass: indeed, for every $Q\in\mathcal{Q}_{q:d}$,
$
    Q_d(\bar{Q},1/(4q))\subset Q
$
and therefore~\eqref{eq:Ahlfors} gives
$
    \mathbb{P}_X(Q)\ge c_-(4q)^{-d}>0.
$
Thus $\mathcal{Q}^\star=\mathcal{Q}_{q:d}$ and $m=q^d$.  In particular, the grid midpoints appearing as the columns of $\widehat K$ have covering radius
$
    R\eqdef1/(2q)
$
on $[0,1]^d$.

For each $Q\in\mathcal{Q}_{q:d}$, set
$
    u_Q\eqdef\widehat A_{\widehat c(Q)}
$
and
$
    \varepsilon_{\operatorname{val}}\eqdef e_\rho+2\xi.
$
The preceding bound gives
$
    \max_Q\|u_Q-f(\bar{Q})\|_2\le\varepsilon_{\operatorname{val}}.
$
Set
$B\eqdef\varepsilon_{\operatorname{val}}+\omega(1)$,
as in Lemma~\ref{lem:adaptive_finite_temperature_covering}.
By Lemma~\ref{lem:Convergence_Samples__fixed_grid},
$
    r_\rho^\star\le1/(2q),
$
so, with $W_q\eqdef\omega(\sqrt d/q)$,
$
    e_\rho\le2W_q
$
and therefore
$
    \varepsilon_{\operatorname{val}}\le2W_q+2\xi.
$
Moreover, $\omega(r_\rho^\star)>0$ implies $W_q>0$.  Define
\begin{equation}
\label{eq:approx_source_temperature_threshold}
    s_q
    \eqdef
    \min\left\{W_q,\frac{Bm}{2}\right\},
    \qquad
    \lambda_0^{(\eta)}
    \eqdef
    \frac{R}{\log(Bm/s_q)}.
\end{equation}
Since $B>0$, $m\ge1$, and $0<s_q\le Bm/2$, the logarithm in~\eqref{eq:approx_source_temperature_threshold} is at least $\log2$, so $\lambda_0^{(\eta)}>0$.  For every $0<\lambda\le\lambda_0^{(\eta)}$,
$
    Bm e^{-R/\lambda}
    \le
    Bm e^{-R/\lambda_0^{(\eta)}}
    =
    s_q
    \le
    W_q.
$

We may therefore apply Lemma~\ref{lem:adaptive_finite_temperature_covering} with the grid midpoints as the centres, the vectors $u_Q$ as their values, $R=1/(2q)$, and $e=\varepsilon_{\operatorname{val}}$.  Since $2R=1/q\le\sqrt d/q$, monotonicity of $\omega$ gives $\omega(2R)\le W_q$.  Hence, for every $x\in[0,1]^d$ and every $0<\lambda\le\lambda_0^{(\eta)}$,
\begin{align*}
    \|\hat{\mathcal{T}}_\lambda(x)-f(x)\|_2
&\le
    \varepsilon_{\operatorname{val}}
    +
    \omega(2R)
    +
    Bm e^{-R/\lambda}
\\
&\le
    (2W_q+2\xi)+W_q+W_q
=
    4W_q+2\xi
\le
    5W_q+2\xi.
\end{align*}
Raising to the power $p$, integrating with respect to $\mathbb{P}_X$, and taking the $p$th root yields
$
    \mathbb E_{X\sim\mathbb{P}_X}
    [\|\hat{\mathcal{T}}_\lambda(X)-f(X)\|_2^p]^{1/p}
    \le
    5W_q+2\xi
$.
Since $W_q=\omega(\sqrt d/q)$ and $\xi=\sqrt D\,\eta$, this is assertion~{\rm (i)}.
Thus both assertions of the theorem hold on $\mathcal E_{\operatorname{pur}}$.  The class-recovery conclusion above is valid for every representative-selection order, so the auxiliary randomness of Algorithm~\ref{alg:STEP_2_CLUSTER} contributes no additional failure event.  Since~\eqref{eq:approx_source_purification_event} satisfies $
    \mathbb{P}(\mathcal E_{\operatorname{pur}})\ge1-\delta
$, the result follows.
\end{proof}

\section{Proof of Main OOD Results}
\label{s:Proofs__ss:OOD}
\begin{proof}[Proof of Theorem~\ref{thrm:2_Scalable_Fine_Tuning}]
For $Q\in\mathcal{Q}_{q:d}$, set $\widetilde z_Q\eqdef\tilde{f}(\bar{Q})-\widetilde\beta_{\widetilde c(Q)}$.  
By~\eqref{eq:approx_ood_midpoint_condition}, $\|\widetilde z_Q\|_\infty\le\widetilde\eta$, and therefore
\begin{equation}
\label{eq:approx_ood_norm_conversion}
    \|\widetilde z_Q\|_2^2\le D\|\widetilde z_Q\|_\infty^2\le D\widetilde\eta^2.
\end{equation}
Hence
\begin{equation}
\label{eq:approx_ood_l2_midpoint_defect}
    \|\tilde{f}(\bar{Q})-\widetilde\beta_{\widetilde c(Q)}\|_2\le\sqrt D\,\widetilde\eta=\widetilde\xi.
\end{equation}

Let $U_{\operatorname{src}}$ denote the auxiliary representative-selection randomness used by \hyperref[alg:STEP_2_CLUSTER]{Algorithm~\ref*{alg:STEP_2_CLUSTER}} in the fixed-grid source construction, and set
\begin{equation}
\label{eq:approx_ood_source_sigma_algebra}
    \mathcal F_{\operatorname{src}}\eqdef\sigma(\mathcal D_{\operatorname{src}},U_{\operatorname{src}}).
\end{equation}
The target sample is independent of $\mathcal F_{\operatorname{src}}$, while $\mathcal E_{\operatorname{src}}\in\mathcal F_{\operatorname{src}}$. On $\mathcal E_{\operatorname{src}}$, the set of permutations satisfying~\eqref{eq:source_pretraining_success} is nonempty and finite; choosing its lexicographically first element therefore gives an $\mathcal F_{\operatorname{src}}$-measurable witness $\boldsymbol{\pi}\in\mathbb{S}_\kappa$ satisfying
\begin{equation}
\label{eq:approx_ood_frozen_class_identity}
    \boldsymbol{\pi}\bigl(\widehat c^{\,\circ}(Q)\bigr)=c(Q)
    \qquad\text{for every }Q\in\mathcal{Q}_{q:d}.
\end{equation}
On the same event, $\hat{\kappa}=\kappa$ and every grid midpoint occurs as a column of $\widehat K$. Define, only for the proof,
\begin{equation}
\label{eq:approx_ood_gamma}
    \gamma\eqdef\pi_\star\circ\boldsymbol{\pi}:[\hat{\kappa}]\to[\widetilde\kappa].
\end{equation}

Apply Lemma~\ref{lem:Convergence_Samples__fixed_grid} to the target task on the deterministic grid with localization radius $1/(2q)$. Since every $Q\in\mathcal{Q}_{q:d}$ is open and centred at $\bar{Q}$, one has $Q\cap Q_d(\bar{Q},1/(2q))=Q$; hence the counts and averages in that lemma are precisely $\widetilde{N}_Q$ and $\widetilde Y_Q$. Under~\eqref{eq:target_sample_complexity}, the lemma gives
\begin{equation}
\label{eq:approx_ood_target_purification_event}
\begin{aligned}
    \widetilde{\mathcal E}_{\operatorname{pur}}^{(\widetilde\eta)}
    \eqdef
    \left\{
        \min_{Q\in\mathcal{Q}_{q:d}}\widetilde{N}_Q\ge1,\quad
        \max_{Q\in\mathcal{Q}_{q:d}}
        \|\widetilde Y_Q-\tilde{f}(\bar{Q})\|_2
        \le2\widetilde w
    \right\}
\quad\mbox{ and }\quad
    \mathbb{P}\bigl(\widetilde{\mathcal E}_{\operatorname{pur}}^{(\widetilde\eta)}\bigr)
    \ge1-\delta.
\end{aligned}
\end{equation}
This event is target-sample measurable and therefore independent of $\mathcal E_{\operatorname{src}}$. Since $\mathbb{P}(\mathcal E_{\operatorname{src}})>0$,
\begin{equation}
\label{eq:approx_ood_conditional_target_probability}
    \mathbb{P}\left(
        \widetilde{\mathcal E}_{\operatorname{pur}}^{(\widetilde\eta)}
        \,\middle|\,
        \mathcal E_{\operatorname{src}}
    \right)
    =
    \mathbb{P}\bigl(\widetilde{\mathcal E}_{\operatorname{pur}}^{(\widetilde\eta)}\bigr)
    \ge1-\delta.
\end{equation}

Henceforth work on $\mathcal E_{\operatorname{src}}\cap\widetilde{\mathcal E}_{\operatorname{pur}}^{(\widetilde\eta)}$. Full target-cell occupancy implies $\widetilde{\mathcal{Q}}_{\operatorname{obs}}=\mathcal{Q}_{q:d}$. Fix $j\in[\hat{\kappa}]$ such that $\widetilde{\mathcal{Q}}^{\,\operatorname{sim}}_j\neq\emptyset$. If $R\in\widetilde{\mathcal{Q}}^{\,\operatorname{sim}}_j$, then $\widehat c^{\,\circ}(R)=j$ and $R\in\mathcal{Q}_{\pi_\star}$, so~\eqref{eq:approx_ood_frozen_class_identity} and the definition of $\mathcal{Q}_{\pi_\star}$ give
\begin{equation}
\label{eq:approx_ood_compatible_class_identity}
    \widetilde c(R)
    =\pi_\star(c(R))
    =\pi_\star\bigl(\boldsymbol{\pi}(j)\bigr)
    =\gamma(j).
\end{equation}
Thus every cell entering the average defining $\widetilde v_j$ has the same target prototype. By~\eqref{eq:approx_ood_target_purification_event}, \eqref{eq:approx_ood_l2_midpoint_defect}, and the triangle inequality,
\begin{align}
\label{eq:approx_ood_class_value_estimate}
    \|\widetilde v_j-\widetilde\beta_{\gamma(j)}\|_2
\le
    \frac{1}{\#\widetilde{\mathcal{Q}}^{\,\operatorname{sim}}_j}
    \sum_{R\in\widetilde{\mathcal{Q}}^{\,\operatorname{sim}}_j}
    \left(
        \|\widetilde Y_R-\tilde{f}(\bar R)\|_2
        +
        \|\tilde{f}(\bar R)-\widetilde\beta_{\gamma(j)}\|_2
    \right)
\le
    2\widetilde w+\widetilde\xi.
\end{align}

We next verify that the relevant class average is never a default value on the compatible region. Fix $Q\in\mathcal{Q}_{\pi_\star}$ and set $j\eqdef\widehat c^{\,\circ}(Q)$. Since $\widetilde{N}_Q\ge1$, the cell $Q$ belongs to $\widetilde{\mathcal{Q}}_{\operatorname{obs}}$, and therefore
\begin{equation}
\label{eq:approx_ood_nonempty_relevant_family}
    Q\in\widetilde{\mathcal{Q}}^{\,\operatorname{sim}}_j
\qquad \mbox{ and }\qquad
    \widetilde{\mathcal{Q}}^{\,\operatorname{sim}}_j\neq\emptyset
.
\end{equation}
Moreover, $\widetilde c(Q)=\gamma(j)$ by~\eqref{eq:approx_ood_compatible_class_identity}. Combining~\eqref{eq:approx_ood_class_value_estimate} with~\eqref{eq:approx_ood_l2_midpoint_defect} yields
\begin{equation}
\label{eq:approx_ood_relevant_cell_value_accuracy}
    \|\widetilde v_{\widehat c^{\,\circ}(Q)}-\tilde{f}(\bar{Q})\|_2
    \le2\widetilde w+2\widetilde\xi.
\end{equation}
The two copies of $\widetilde\xi$ arise separately from the approximate-symmetry defect of the cells used in the class average and from that of the queried cell.

Now fix $x\in\mathcal{G}_{\pi_\star}$ and let $Q_x$ be its unique containing grid cell. On $\mathcal E_{\operatorname{src}}$, every grid midpoint is a column of $\widehat K$. Since $x\in Q_x$ and the grid cells are open, Proposition~\ref{prop:HardAttentionRouting__fixed_grid} applies and gives
\begin{equation}
\label{eq:approx_ood_hard_routing}
    \widetilde{\mathcal T}_0(x)=\widetilde v_{\widehat c^{\,\circ}(Q_x)}.
\end{equation}
Since $\|x-\bar{Q}_x\|_\infty<1/(2q)$, uniform continuity and~\eqref{eq:approx_ood_relevant_cell_value_accuracy} imply
\begin{equation}
\label{eq:approx_ood_hard_pointwise}
    \|\widetilde{\mathcal T}_0(x)-\tilde{f}(x)\|_2
    \le2\widetilde w+2\widetilde\xi+\widetilde\omega\left(\frac{1}{2q}\right).
\end{equation}
Define
\begin{equation}
\label{eq:approx_ood_a_definition}
    \widetilde a_q^{(\widetilde\eta)}
    \eqdef
    2\widetilde w+2\widetilde\xi+\widetilde\omega\left(\frac{1}{2q}\right).
\end{equation}
Because $\widetilde r^\star\le1/(2q)\le\sqrt d/q$ and $\widetilde\omega$ is increasing,
\begin{equation}
\label{eq:approx_ood_a_bound}
    \widetilde a_q^{(\widetilde\eta)}
    \le3\widetilde W_q+2\widetilde\xi.
\end{equation}
Since $\widetilde\omega$ is strictly increasing, $\widetilde\omega(0)=0$, and $\sqrt d/q>0$, we have $\widetilde W_q=\widetilde\omega(\sqrt d/q)>0$.

We also establish the mass assertion. Since the grid cells are pairwise disjoint, $\widetilde{\mathbb{P}}_X$ assigns zero mass to the grid skeleton, every cell has mass at least $\widetilde p^\star$, and
$\#\mathcal{Q}_{\pi_\star}=q^d\operatorname{Sim}_q(\widetilde{\mathcal S}\mid\mathcal S)$, we obtain
\begin{equation}
\label{eq:approx_ood_mass_proof}
    \widetilde{\mathbb{P}}_X(\mathcal{G}_{\pi_\star})
    =
    \sum_{Q\in\mathcal{Q}_{\pi_\star}}\widetilde{\mathbb{P}}_X(Q)
    \ge
    \widetilde p^\star\#\mathcal{Q}_{\pi_\star}
    \ge
    \operatorname{Sim}_q(\widetilde{\mathcal S}\mid\mathcal S)q^d\widetilde p^\star
    >
    0.
\end{equation}
Thus the conditional law
\begin{equation}
\label{eq:approx_ood_conditional_law}
    \widetilde{\mathbb{P}}_X^{\,\pi_\star}(A)
    \eqdef
    \frac{\widetilde{\mathbb{P}}_X(A\cap\mathcal{G}_{\pi_\star})}
    {\widetilde{\mathbb{P}}_X(\mathcal{G}_{\pi_\star})}
\end{equation}
is well defined.

It remains to pass from hard to positive-temperature attention. Define
\begin{equation}
\label{eq:approx_ood_B_definition}
    \widetilde B^{(\widetilde\eta)}
    \eqdef
    1+\max_{\ell\in[\widetilde\kappa]}\|\widetilde\beta_\ell\|_2
    +2\widetilde w+\widetilde\xi.
\end{equation}
If $\widetilde{\mathcal{Q}}^{\,\operatorname{sim}}_j=\emptyset$, then $\widetilde v_j=0$ by definition. Otherwise~\eqref{eq:approx_ood_class_value_estimate} gives $\|\widetilde v_j\|_2\le\|\widetilde\beta_{\gamma(j)}\|_2+2\widetilde w+\widetilde\xi<\widetilde B^{(\widetilde\eta)}$. Hence every column of $\widetilde V$ has Euclidean norm at most $\widetilde B^{(\widetilde\eta)}$. By~\eqref{eq:approx_ood_nonempty_relevant_family}, hard attention at a compatible query never selects one of the zero defaults.

Set
\begin{equation}
\label{eq:approx_ood_s_definition}
    \widetilde s^{(\widetilde\eta)}
    \eqdef
    \min\left\{
        \widetilde W_q,\,
        \widetilde B^{(\widetilde\eta)}q^d
    \right\}.
\end{equation}
Then
\begin{equation}
\label{eq:approx_ood_s_positive}
    0<\widetilde s^{(\widetilde\eta)}
    \le\widetilde W_q,
    \qquad
    \widetilde s^{(\widetilde\eta)}
    \le\widetilde B^{(\widetilde\eta)}q^d.
\end{equation}

Let $\mathcal S_q=[0,1]^d\setminus\mathcal{G}$ be the grid skeleton and define
\begin{equation}
\label{eq:approx_ood_boundary_tube}
    \mathcal B_\zeta
    \eqdef
    \{x:\operatorname{dist}_\infty(x,\mathcal S_q)<\zeta\}.
\end{equation}
As $\zeta\downarrow0$, the sets $\mathcal B_\zeta$ decrease to $\mathcal S_q$, which has zero $\widetilde{\mathbb{P}}_X^{\,\pi_\star}$-mass. Continuity from above therefore allows us to choose a deterministic $0<\zeta<1/(2q)$ such that
\begin{equation}
\label{eq:approx_ood_tube_choice}
    \widetilde{\mathbb{P}}_X^{\,\pi_\star}(\mathcal B_\zeta)
    \le
    \left(
        \frac{\widetilde s^{(\widetilde\eta)}}
        {\widetilde a_q^{(\widetilde\eta)}+2\widetilde B^{(\widetilde\eta)}}
    \right)^p.
\end{equation}

Fix $x\in\mathcal{G}_{\pi_\star}\setminus\mathcal B_\zeta$ and let $Q_x$ be its cell. Since $x$ lies at $\ell^\infty$-distance at least $\zeta$ from every face of $Q_x$, $\|x-\bar{Q}_x\|_\infty\le1/(2q)-\zeta$. If $R\neq Q_x$, then the midpoints of $R$ and $Q_x$ differ by at least $1/q$ in some coordinate; the reverse triangle inequality in that coordinate therefore gives $\|x-\bar R\|_\infty\ge1/(2q)+\zeta$. Thus the score of $Q_x$ exceeds every competing score by at least $2\zeta$. Consequently, each competing softmax weight is at most $\exp(-2\zeta/\lambda)$. Since there are at most $q^d-1$ competitors and all values have norm at most $\widetilde B^{(\widetilde\eta)}$,
\begin{equation}
\label{eq:approx_ood_soft_hard_bound}
    \|\widetilde{\mathcal T}_\lambda(x)-\widetilde{\mathcal T}_0(x)\|_2
    \le
    2\widetilde B^{(\widetilde\eta)}q^d
    \exp\left(-\frac{2\zeta}{\lambda}\right).
\end{equation}
Define
\begin{equation}
\label{eq:approx_ood_temperature_threshold}
    \lambda_0
    \eqdef
    \frac{2\zeta}{
        \log\left(
            \frac{2\widetilde B^{(\widetilde\eta)}q^d}
            {\widetilde s^{(\widetilde\eta)}}
        \right)}.
\end{equation}
By~\eqref{eq:approx_ood_s_positive}, the logarithm has argument at least $2$, so $\lambda_0>0$. Every quantity in its definition is determined by the fixed source and target tasks, the population laws, $p$, $q$, the oracle-compatible region, and $\widetilde\eta$; in particular, the threshold is independent of both realized training samples. For every $0<\lambda\le\lambda_0$, \eqref{eq:approx_ood_soft_hard_bound} is at most $\widetilde s^{(\widetilde\eta)}$.

Hence, on $\mathcal{G}_{\pi_\star}\setminus\mathcal B_\zeta$, the pointwise error is at most $\widetilde a_q^{(\widetilde\eta)}+\widetilde s^{(\widetilde\eta)}$. On $\mathcal B_\zeta\cap\mathcal{G}_{\pi_\star}$, both $\widetilde{\mathcal T}_\lambda(x)$ and $\widetilde{\mathcal T}_0(x)$ lie in the convex hull of columns of $\widetilde V$ of norm at most $\widetilde B^{(\widetilde\eta)}$, so $\|\widetilde{\mathcal T}_\lambda(x)-\widetilde{\mathcal T}_0(x)\|_2\le2\widetilde B^{(\widetilde\eta)}$. Together with~\eqref{eq:approx_ood_hard_pointwise}, the error there is therefore at most $\widetilde a_q^{(\widetilde\eta)}+2\widetilde B^{(\widetilde\eta)}$. Integrating with respect to $\widetilde{\mathbb{P}}_X^{\,\pi_\star}$ and using~\eqref{eq:approx_ood_tube_choice} gives
\begin{align}
\label{eq:approx_ood_integrated_intermediate}
&\mathbb E\left[
    \|\widetilde{\mathcal T}_\lambda(\widetilde X)-\tilde{f}(\widetilde X)\|_2^p
    \,\middle|\,
    \widetilde X\in\mathcal{G}_{\pi_\star}
\right]
\nonumber\\
&\le
    \left(\widetilde a_q^{(\widetilde\eta)}+\widetilde s^{(\widetilde\eta)}\right)^p
    +
    \left(\widetilde a_q^{(\widetilde\eta)}+2\widetilde B^{(\widetilde\eta)}\right)^p
    \widetilde{\mathbb{P}}_X^{\,\pi_\star}(\mathcal B_\zeta)
\nonumber\\
&\le
    \left(\widetilde a_q^{(\widetilde\eta)}+\widetilde s^{(\widetilde\eta)}\right)^p
    +
    \left(\widetilde s^{(\widetilde\eta)}\right)^p
\le
    (4\widetilde W_q+2\widetilde\xi)^p+\widetilde W_q^p
\le
    (5\widetilde W_q+2\widetilde\xi)^p.
\end{align}
The last inequality uses $u^p+v^p\le(u+v)^p$ for $u,v\ge0$ and $p\ge1$. Taking $p$th roots, recalling $\widetilde W_q=\widetilde\omega(\sqrt d/q)$ and $\widetilde\xi=\sqrt D\,\widetilde\eta$, and combining with~\eqref{eq:approx_ood_conditional_target_probability} proves~\eqref{eq:scalable_fine_tuning_Lp_bound}, simultaneously for every $0<\lambda\le\lambda_0$, with the asserted conditional probability.

Finally, the fine-tuning construction leaves $\widehat K$ and $C_{\widehat c}$ unchanged. All target dependence enters through $\widetilde A\in\mathbb R^{D\times\hat{\kappa}}$, which contains $D\hat{\kappa}$ scalar entries. On $\mathcal E_{\operatorname{src}}$, $\hat{\kappa}=\kappa$, so at most $D\kappa$ scalar entries are target-dependent. This is precisely the claimed update count and entails none of the stronger computational conclusions excluded in the theorem.
\end{proof}


\section{Results on Symmetric Functions}
\label{s:AuxiliatyResults__ss:Symmetries}

\begin{proposition}[H\"older Constant Quantitatively Implies Symmetries]
\label{prop:holder_implies_symmetry}
Fix $d,q\in\mathbb{N}_+$, $\alpha\in(0,1]$, $L\geq0$, and $\eta>0$, and let
$
    f:[0,1]^d\to[0,1]
$
be $\alpha$-H\"older with constant $L$.  Then there exist
$
    \kappa\in[q^d]
$
and a $\kappa$-symmetry $\mathcal{S}$ at scale $q$ such that $f$ is $(\mathcal{S},\eta)$-symmetric and
\[
    \kappa
    \leq
    \min\left\{
        q^d,
        \max\left\{
            1,
            \left\lceil
                \frac{\min\{1,L\}}{2\eta}
            \right\rceil
        \right\}
    \right\}.
\]
In particular, if $q^d>1$, then $\kappa$ is nontrivially (i.e.\ $\kappa<q^d$) whenever
$
    \min\{1,L\}
\leq
    2\eta(q^d-1)$.
\end{proposition}
\begin{proof}[{Proof of Proposition~\ref{prop:holder_implies_symmetry}}]
Since $f$ is H\"older continuous and $[0,1]^d$ is compact, there exist $x_-,x_+\in[0,1]^d$ such that
\[
    m
    \eqdef
    f(x_-)
    =
    \min_{x\in[0,1]^d}f(x),
    \qquad
    M
    \eqdef
    f(x_+)
    =
    \max_{x\in[0,1]^d}f(x).
\]
Set $r\eqdef M-m$.  Since $f$ takes values in $[0,1]$, we have $r\leq1$.  Moreover, by the $\alpha$-H\"older continuity of $f$ and the fact that the diameter of $[0,1]^d$ is $1$,
$
    r
    =
    |f(x_+)-f(x_-)|
    \leq
    L\|x_+-x_-\|_\infty^\alpha
    \leq
    L
$.
Consequently,
\begin{equation}
\label{eq:holder_range_diameter}
    r
    \leq
    \min\{1,L\}.
\end{equation}
Define
$
    K
    \eqdef
    \max\{
        1,
        \left\lceil
            \frac{r}{2\eta}
        \right\rceil
    \}$.
For every $j\in[K]$, set
$
    \gamma_j
    \eqdef
    m+\frac{2j-1}{2K}r$.  Thus, $\gamma_j$ is the midpoint of the $j$th interval in the uniform subdivision of $[m,M]$ into $K$ intervals.  By the definition of $K$,
\begin{equation}
\label{eq:holder_range_quantization}
    \frac{r}{2K}
    \leq
    \eta.
\end{equation}
For each $Q\in\mathcal{Q}_{q:d}$, define
$
    j(Q)
\eqdef
    \min\{
        j\in[K]:
        f(\bar{Q})
        \leq
        m+\frac{j}{K}r
    \}
$.
This minimum is well defined because $f(\bar{Q})\leq M=m+r$.  Moreover, by minimality of $j(Q)$,
\[
    m+\frac{j(Q)-1}{K}r
    \leq
    f(\bar{Q})
    \leq
    m+\frac{j(Q)}{K}r,
\]
where the left-hand inequality is immediate when $j(Q)=1$.  Hence
$
    \bigl|
        f(\bar{Q})-\gamma_{j(Q)}
    \bigr|
    \leq
    \frac{r}{2K}
    \leq
    \eta$,
where the final inequality follows from~\eqref{eq:holder_range_quantization}.
Let
$
    J\eqdef\{j(Q):Q\in\mathcal{Q}_{q:d}\}
$
and enumerate its distinct elements increasingly as
$
    J=\{j_1,\dots,j_\kappa\}.
$
For each $k\in[\kappa]$, define
$
    \mathcal{S}_k
    \eqdef
    \{
        Q\in\mathcal{Q}_{q:d}:
        j(Q)=j_k
    \}
$ and $\beta_k
    \eqdef
    \gamma_{j_k}$.  
By construction, each $\mathcal{S}_k$ is non-empty, the sets $\mathcal{S}_1,\dots,\mathcal{S}_\kappa$ are pairwise disjoint, and their union is $\mathcal{Q}_{q:d}$.  Therefore
$
    \mathcal{S}\eqdef\{\mathcal{S}_k\}_{k=1}^{\kappa}
$
is a $\kappa$-symmetry at scale $q$.  Furthermore, for every $Q\in\mathcal{Q}_{q:d}$,
\[
    \bigl|
        f(\bar{Q})-\beta_{c_{\mathcal{S}}(Q)}
    \bigr|
    \leq
    \eta,
\]
and therefore $f$ is $(\mathcal{S},\eta)$-symmetric.
Finally, since $\mathcal{Q}_{q:d}$ contains exactly $q^d$ cells and $J\subseteq[K]$,
\[
    \kappa
    =
    \#J
    \leq
    \min\{q^d,K\}.
\]
By~\eqref{eq:holder_range_diameter} and the monotonicity of the ceiling function,
\[
    K
    \leq
    \max\left\{
        1,
        \left\lceil
            \frac{\min\{1,L\}}{2\eta}
        \right\rceil
    \right\},
\]
which proves the claimed bound.
\end{proof}

We now obtain our two specialized cases: piecewise linear functions and Pfaffian functions.

\subsection{The Piecewise Linear Scenario}
\label{s:AuxiliatyResults__ss:Symmetries___ssss:PWL}

\begin{lemma}[Piecewise Linearity Controls Symmetry Complexity]
\label{lem:PWL_Symmetry_Complexity}
Fix $d,q,\Pi,Y\in\mathbb{N}_+$, $L\geq0$, and $\eta>0$.  Let
$
    f:[0,1]^d\to[0,1]
$
be an $L$-Lipschitz $(\Pi,Y)$-polyhedral piecewise-linear function, with
$
    f(x)=\langle a_\pi,x\rangle+b_\pi
$
on $P_\pi$, and suppose that
\begin{equation}
\label{eq:PWL_nonflatness}
    \delta_{\nabla}
    \eqdef
    \min_{\pi\in[\Pi]}
    \|a_\pi\|_2
>
    0
.
\end{equation}
There exist 
$
    \kappa\in[q^d]
$
and a $\kappa$-symmetry
$
    \mathcal{S}
    =
    \{\mathcal{S}_k\}_{k=1}^{\kappa}
$
at scale $q$ such that $f$ is $(\mathcal{S},\eta)$-symmetric and
\begin{alignat}{3}
\label{eq:PWL_kappa_sandwich}
\Big(
    \min\big\{
        1,\,
        \Pi
        \big(
            \tfrac{1}{q}
            +
            \tfrac{2\sqrt d\,\eta}{\delta_{\nabla}}
        \big)
    \big\}
\Big)^{-1}
&\leq
\kappa
&\leq&
\min\Big\{
    q^d,
    \max\big\{
        1,
        \big\lceil
            \tfrac{\min\{1,L\}}{2\eta}
        \big\rceil
    \big\}
\Big\},
\\
\label{eq:PWL_class_complexity_bound}
\Big(
    \min\big\{
        q^d,
        \max\big\{
            1,
            \big\lceil
                \tfrac{\min\{1,L\}}{2\eta}
            \big\rceil
        \big\}
    \big\}
\Big)^{-1}
&\leq
\max_{k\in[\kappa]}
\tfrac{\#\mathcal{S}_k}{q^d}
&\leq&
\min\Big\{
    1,\,
    \Pi
    \big(
        \tfrac{1}{q}
        +
        \tfrac{2\sqrt d\,\eta}{\delta_{\nabla}}
    \big)
\Big\}.
\end{alignat}
\end{lemma}

\begin{proof}[{Proof of Lemma~\ref{lem:PWL_Symmetry_Complexity}}]
By Proposition~\ref{prop:holder_implies_symmetry}, applied with $\alpha=1$, there exist
$
    \kappa\in[q^d]
$
and a $\kappa$-symmetry
$
    \mathcal{S}
    =
    \{\mathcal{S}_k\}_{k=1}^{\kappa}
$
at scale $q$ such that $f$ is $(\mathcal{S},\eta)$-symmetric and
\begin{equation}
\label{eq:PWL_upper_kappa}
    \kappa
    \leq
    \min\left\{
        q^d,
        \max\left\{
            1,
            \left\lceil
                \frac{\min\{1,L\}}{2\eta}
            \right\rceil
        \right\}
    \right\}.
\end{equation}
We fix this symmetry $\mathcal{S}$ for the remainder of the proof, and let
$
    \beta_1,\dots,\beta_\kappa
$
be its associated class centres.  Fix $k\in[\kappa]$.  Since
$
    [0,1]^d=\bigcup_{\pi=1}^{\Pi}P_\pi,
$
every grid midpoint belongs to at least one polytope.  To avoid ambiguity at the boundaries of the polytopes, assign each $\bar{Q}$ to the smallest $\pi\in[\Pi]$ such that $\bar{Q}\in P_\pi$.  For every $\pi\in[\Pi]$, define
\[
    \mathcal{S}_{k,\pi}
    \eqdef
    \left\{
        Q\in\mathcal{S}_k:
        \pi
        =
        \min
        \left\{
            \pi'\in[\Pi]:
            \bar{Q}\in P_{\pi'}
        \right\}
    \right\}.
\]
Then
$
    \mathcal{S}_k
    =
    \sqcup_{\pi=1}^{\Pi}
    \,
    \mathcal{S}_{k,\pi}
$.  
Fix $\pi\in[\Pi]$.  By~\eqref{eq:PWL_nonflatness}, there exists
$
    j_\pi\in[d]
$
such that
$
    |(a_\pi)_{j_\pi}|
    \geq
    \tfrac{\|a_\pi\|_2}{\sqrt d}
    \geq
    \tfrac{\delta_{\nabla}}{\sqrt d}.
$
We now partition the grid midpoints according to the $q^{d-1}$ grid lines parallel to the $j_\pi$-th coordinate axis.  More precisely, for every
$
    \iota=(i_\ell)_{\ell\neq j_\pi}
    \in
    \{0,\dots,q-1\}^{d-1},
$
define the line segment
\[
    \ell_{\iota}(t)
    \eqdef
    \big(
        z_1(t),\dots,z_d(t)
    \big),
    \qquad
    t\in[0,1],
\]
where
$
    z_{j_\pi}(t)
    \eqdef
    t
$ and $z_\ell(t)
    \eqdef
    \frac{2i_\ell+1}{2q}$ for every 
    $\ell\neq j_\pi
$.   There are exactly $q^{d-1}$ such line segments, and every grid midpoint lies on exactly one of them.  Set
$
    I_{\pi,\iota}
\eqdef
    \big\{
        t\in[0,1]:
        \ell_{\iota}(t)\in P_\pi
    \big\}
$.  Since $P_\pi$ is convex and $\ell_{\iota}$ is affine, $I_{\pi,\iota}$ is a convex subset of $[0,1]$ and hence is either empty, a singleton, or an interval.  
\hfill\\
\noindent
On $I_{\pi,\iota}$, define
$
    h_{\pi,\iota}(t)
    \eqdef
    f(\ell_{\iota}(t)).
$
Since $f(x)=\langle a_\pi,x\rangle+b_\pi$ on $P_\pi$, there exists a constant
$
    b_{\pi,\iota}\in\mathbb{R}
$
such that, for every $t\in I_{\pi,\iota}$
\[
    h_{\pi,\iota}(t)
    =
    (a_\pi)_{j_\pi}t+b_{\pi,\iota}
.
\]
In particular, $h_{\pi,\iota}$ is affine on $I_{\pi,\iota}$ with slope $(a_\pi)_{j_\pi}$.  Hence
\[
    I_{k,\pi,\iota}
    \eqdef
    \left\{
        t\in I_{\pi,\iota}:
        |h_{\pi,\iota}(t)-\beta_k|
        \leq
        \eta
    \right\}
\]
is either empty, a singleton, or an interval.  If $s,t\in I_{k,\pi,\iota}$, then
\[
    \frac{\delta_{\nabla}}{\sqrt d}|t-s|
    \leq
    |(a_\pi)_{j_\pi}|\,|t-s|
    =
    |h_{\pi,\iota}(t)-h_{\pi,\iota}(s)|
    \leq
    |h_{\pi,\iota}(t)-\beta_k|
    +
    |h_{\pi,\iota}(s)-\beta_k|
    \leq
    2\eta.
\]
Therefore, $
    \operatorname{diam}(I_{k,\pi,\iota})
\leq
    \frac{2\sqrt d\,\eta}{\delta_{\nabla}}$.  
The grid midpoints on the line $\ell_{\iota}$ correspond to the parameters
$
    t=(2i+1)/(2q)
$
for $i=0,\dots,q-1$, and consecutive such parameters are separated by exactly $1/q$.  Hence any interval of length $\ell\geq0$ contains at most
$
    \lfloor q\ell\rfloor+1
    \leq
    q\ell+1
$
such grid midpoints.  Every midpoint of a cell in $\mathcal{S}_{k,\pi}$ lying on $\ell_{\iota}$ belongs to $P_\pi$ and, by $(\mathcal{S},\eta)$-symmetry, satisfies
$
    |f(\bar{Q})-\beta_k|\leq\eta.
$
It therefore corresponds to a parameter in $I_{k,\pi,\iota}$.  By our estimate $\operatorname{diam}(I_{k,\pi,\iota})
\leq
    \frac{2\sqrt d\,\eta}{\delta_{\nabla}}
$ we have that at most
$
    1
    +
    \tfrac{2\sqrt d\,q\eta}{\delta_{\nabla}}
$
grid midpoints on $\ell_{\iota}$ can correspond to cells in $\mathcal{S}_{k,\pi}$.  
Since there are exactly $q^{d-1}$ choices of $\iota$, we obtain
\begin{equation}
\label{eq:PWL_piece_class_bound}
    \#\mathcal{S}_{k,\pi}
    \leq
    q^{d-1}
    \big(
        1
        +
        \tfrac{2\sqrt d\,q\eta}{\delta_{\nabla}}
    \big)
.
\end{equation}
Summing~\eqref{eq:PWL_piece_class_bound} over $\pi\in[\Pi]$ and using
$
    \mathcal{S}_k
    =
    \bigsqcup_{\pi=1}^{\Pi}
    \mathcal{S}_{k,\pi}
$
gives $
    \#\mathcal{S}_k
\leq
    \Pi q^{d-1}
    \big(
        1
        +
        \tfrac{2\sqrt d\,q\eta}{\delta_{\nabla}}
    \big)
$; and dividing by $q^d$ yields
\begin{equation}
\label{eq:PWL_per_class_upper_bound}
    \frac{\#\mathcal{S}_k}{q^d}
    \leq
    \Pi
    \big(
        \tfrac{1}{q}
        +
        \tfrac{2\sqrt d\,\eta}{\delta_{\nabla}}
    \big)
.
\end{equation}
Since also
$
    \#\mathcal{S}_k\leq q^d,
$
we conclude that, for every $k\in[\kappa]$
\begin{equation}
\label{eq:PWL_per_class_upper_bound_truncated}
    \frac{\#\mathcal{S}_k}{q^d}
    \leq
    \min
    \Big\{
        1,\,
        \Pi
        \big(
            \tfrac{1}{q}
            +
            \tfrac{2\sqrt d\,\eta}{\delta_{\nabla}}
        \big)
    \Big\}
\end{equation}
as desired.   We next obtain the lower bound on $\kappa$.  Since $\mathcal{S}$ is a partition of $\mathcal{Q}_{q:d}$,
$
    q^d
    =
    \sum_{k=1}^{\kappa}
    \#\mathcal{S}_k.
$
Applying~\eqref{eq:PWL_per_class_upper_bound_truncated} to each term on the right-hand side gives
\begin{equation}
\label{eq:PWL_B2}
    q^d
    \leq
    \kappa q^d
    \min
    \Big\{
        1,\,
        \Pi
        \big(
            \tfrac{1}{q}
            +
            \tfrac{2\sqrt d\,\eta}{\delta_{\nabla}}
        \big)
    \Big\}
.
\end{equation}
Since $q^d>0$,~\eqref{eq:PWL_B2} implies that
\begin{equation}
\label{eq:PWL_B3}
    \kappa
    \geq
    \tfrac{1}{
        \min\Big\{
            1,\,
            \Pi
            \big(
                \tfrac{1}{q}
                +
                \tfrac{2\sqrt d\,\eta}{\delta_{\nabla}}
            \big)
        \Big\}
    }
.
\end{equation}
Combining~\eqref{eq:PWL_B3} with~\eqref{eq:PWL_upper_kappa} proves~\eqref{eq:PWL_kappa_sandwich}.  Finally, the partition identity gives
$
    q^d
    =
    \sum_{k=1}^{\kappa}
    \#\mathcal{S}_k
    \leq
    \kappa
    \max_{k\in[\kappa]}
    \#\mathcal{S}_k.
$
Hence
$
    \max_{k\in[\kappa]}
    \#\mathcal{S}_k
    \geq
    q^d/\kappa.
$
Using the upper bound on $\kappa$ from~\eqref{eq:PWL_upper_kappa} yields
\[
    \max_{k\in[\kappa]}
    \frac{\#\mathcal{S}_k}{q^d}
    \geq
    \Big(
        \min\big\{
            q^d,
            \max\big\{
                1,
                \big\lceil
                    \tfrac{\min\{1,L\}}{2\eta}
                \big\rceil
            \big\}
        \big\}
    \Big)^{-1}.
\]
Maximizing~\eqref{eq:PWL_per_class_upper_bound_truncated} over $k\in[\kappa]$ gives
$
    \max_{k\in[\kappa]}
    \frac{\#\mathcal{S}_k}{q^d}
\leq
    \min
    \big\{
        1,\,
        \Pi
        \big(
            \tfrac{1}{q}
            +
            \tfrac{2\sqrt d\,\eta}{\delta_{\nabla}}
        \big)
    \big\}
$.
\end{proof}
We are now ready to pass to the proof of Theorem~\ref{thrm:Lipschitz_PWL_Symmetry_Scaling}.  The key is simply to scale $\eta=\Theta(q^{-r})$.
\begin{proof}[{Proof of Theorem~\ref{thrm:Lipschitz_PWL_Symmetry_Scaling}}]
Fix $c>0$ and $r\in(0,1]$, and set
$
    \eta_q\eqdef cq^{-r}
$
for every $q\in\mathbb{N}_+$.  For each $q$, apply Lemma~\ref{lem:PWL_Symmetry_Complexity} with $\eta=\eta_q$, and choose the resulting $\kappa_q$-symmetry
$
    \mathcal{S}^{(q)}
    =
    \{\mathcal{S}^{(q)}_k\}_{k=1}^{\kappa_q}.
$
Then $f$ is $(\mathcal{S}^{(q)},\eta_q)$-symmetric at scale $q$.
Since $r\leq1$, we have
$
    q^{-1}\leq q^{-r}
$
for every $q\in\mathbb{N}_+$.  Therefore~\eqref{eq:PWL_per_class_upper_bound_truncated}, applied with $\eta=\eta_q$, gives
\[
    \frac{\#\mathcal{S}^{(q)}_k}{q^d}
    \leq
    \Pi
    \left(
        1
        +
        \frac{2\sqrt d\,c}{\delta_{\nabla}}
    \right)
    q^{-r}
    \qquad
    \mbox{ for every }
    k\in[\kappa_q].
\]
Summing this estimate over $k\in[\kappa_q]$ and using
$
    \sum_{k=1}^{\kappa_q}
    \#\mathcal{S}^{(q)}_k
    =
    q^d
$
gives
\begin{equation}
\label{eq:PWL_scaling_lower_kappa}
    \kappa_q
    \geq
    \frac{
        q^r
    }{
        \Pi
        \big(
            1
            +
            \tfrac{2\sqrt d\,c}{\delta_{\nabla}}
        \big)
    }
.
\end{equation}
On the other hand, the upper bound in~\eqref{eq:PWL_kappa_sandwich} gives
$
    \kappa_q
\leq
    \min
    \Big\{
        q^d,
        \max\big\{
            1,
            \big\lceil
                \tfrac{\min\{1,L\}}{2c}
                q^r
            \big\rceil
        \big\}
    \Big\}
$.   Since~\eqref{eq:PWL_nonflatness__theoremversion} implies that $f$ is non-constant, every Lipschitz constant for $f$ is strictly positive, and hence $L>0$.  Moreover, $q^r\geq1$, so
\begin{equation}
\label{eq:PWL_scaling_upper_kappa}
    \kappa_q
    \leq
    \left(
        1+
        \frac{\min\{1,L\}}{2c}
    \right)
    q^r.
\end{equation}
Combining~\eqref{eq:PWL_scaling_lower_kappa} and~\eqref{eq:PWL_scaling_upper_kappa} yields $\kappa_q\asymp q^r$ and 
\begin{equation}
\label{eq:PWL_scaling_upper_class}
    \max_{k\in[\kappa_q]}
    \#\mathcal{S}^{(q)}_k
    \leq
    \Pi
    \left(
        1
        +
        \frac{2\sqrt d\,c}{\delta_{\nabla}}
    \right)
    q^{d-r}.
\end{equation}
Conversely, since the classes in $\mathcal{S}^{(q)}$ partition $\mathcal{Q}_{q:d}$,
$    q^d
=
    \sum_{k=1}^{\kappa_q}
    \#\mathcal{S}^{(q)}_k
    \leq
    \kappa_q
    \max_{k\in[\kappa_q]}
    \#\mathcal{S}^{(q)}_k
$.  Using~\eqref{eq:PWL_scaling_upper_kappa} we have that
\begin{equation}
\label{eq:PWL_scaling_lower_class}
    \max_{k\in[\kappa_q]}
    \#\mathcal{S}^{(q)}_k
    \geq
    \frac{
        q^{d-r}
    }{
        1+
        \frac{\min\{1,L\}}{2c}
    }.
\end{equation}
Combining~\eqref{eq:PWL_scaling_upper_class} and~\eqref{eq:PWL_scaling_lower_class} gives
$
    \max_{k\in[\kappa_q]}
    \#\mathcal{S}^{(q)}_k
    \asymp
    q^{d-r}
$; together with $\kappa_q\asymp q^r$ we deduce~\eqref{eq:PWL_matching_scaling}.
\end{proof}

\subsection{The Pfaffian Scenario}
\label{s:AuxiliatyResults__ss:Symmetries___ssss:Pfaff}

\begin{proposition}[Pfaffianity Controls Class Complexity]
\label{prop:PfaffCellComplexity}
Fix $d,q\in\mathbb{N}_+$ and $\eta>0$.  Let
$
    f:\mathbb{R}^d\to[0,1]
$
be Pfaffian of degree $\beta_f\in\mathbb{N}_+$ in a Pfaffian chain of length
$\Lambda\in\mathbb{N}_+$ and degree $\Delta_{\mathrm{pf}}\in\mathbb{N}_+$.  
If $f$ is $(\mathcal{S},\eta)$-symmetric at scale $q$ for a $\kappa$-symmetry
$
    \mathcal{S}=\{\mathcal{S}_k\}_{k=1}^{\kappa}
$ and if
\begin{equation}
\label{eq:immerstion_type_condition}
    \delta_{\nabla}
    \eqdef
    \min_{x\in[0,1]^d}
    \|\nabla f(x)\|_2
>
    0
,
\end{equation}
then there exists a constant
$
    C_{d,\Delta_{\mathrm{pf}},\beta_f,\Lambda}>0,
$
depending only on $d,\Delta_{\mathrm{pf}},\beta_f$, and $\Lambda$, such that
\begin{equation}
\label{eq:class_complexity_bound}
    \frac{\#\mathcal{S}_k}{q^d}
    \leq
    \min
    \left\{
        1,\,
        C_{d,\Delta_{\mathrm{pf}},\beta_f,\Lambda}
        \left(
            \frac{1}{q}
            +
            \frac{\eta}{\delta_{\nabla}}
        \right)
    \right\}
\qquad
    \mbox{ for all }   k \in [\kappa]
.
\end{equation}
\end{proposition}

\begin{proof}[{Proof of Proposition~\ref{prop:PfaffCellComplexity}}]
Fix $k\in[\kappa]$ and let $\beta_1,\dots,\beta_\kappa$ be as in~\eqref{eq:symmetric_labels}.   
We partition the cells in $\mathcal{S}_k$ according to a coordinate direction along which $f$ varies quantitatively.  Indeed, by~\eqref{eq:immerstion_type_condition},
$
    \|\nabla f(\bar{Q})\|_2\geq\delta_{\nabla}
$
for every $Q\in\mathcal{Q}_{q:d}$.  Hence, for every such $Q$, there exists $j\in[d]$ such that
\[
    \left|
        \frac{\partial f}{\partial x_j}(\bar{Q})
    \right|
    \geq
    \frac{\delta_{\nabla}}{\sqrt d}.
\]
For definiteness, let $j(Q)$ be the smallest such $j$, and for every $j\in[d]$ set
$
    \mathcal{S}_{k,j}
    \eqdef
    \{Q\in\mathcal{S}_k:j(Q)=j\}.
$
Then
$
    \mathcal{S}_k
    =
    \bigsqcup_{j=1}^d\mathcal{S}_{k,j}
$.  
Fix $j\in[d]$.  We further partition $\mathcal{S}_{k,j}$ according to the $q^{d-1}$ grid lines parallel to the $j$-th coordinate axis.  Fix one such line by fixing the remaining coordinates at grid midpoints, and write
\[
    h(t)
    \eqdef
    f(x_1,\dots,x_{j-1},t,x_{j+1},\dots,x_d),
    \qquad
    t\in[0,1].
\]
Since restricting to an affine line preserves Pfaffian format, $h$ is Pfaffian of degree $\beta_f$ in a Pfaffian chain of length $\Lambda$ and degree $\Delta_{\mathrm{pf}}$; cf.~\cite[Remark~2.7]{lezeau2026tubular}.  Moreover, by~\cite[Lemma~2.4(3)]{lezeau2026tubular}, $h'$ is Pfaffian in the same chain with degree at most
$
    \Delta_{\mathrm{pf}}+\beta_f-1
$.  Consider
$
    E
\eqdef
    \{
        t\in[0,1]:
        |h'(t)|
        \geq
        \tfrac{\delta_{\nabla}}{\sqrt d}
    \}
$.  The boundary of $E$ in $(0,1)$ is contained in the zero set of the Pfaffian function
$
    r(t)
    \eqdef
    h'(t)^2-\tfrac{\delta_{\nabla}^2}{d}.
$
The function $r$ is Pfaffian in the same chain with degree at most
$
    2(\Delta_{\mathrm{pf}}+\beta_f-1).
$
Therefore, by the quantitative connected-component bound in~\cite[Corollary~2.11]{lezeau2026tubular}, the zero set of $r$ has at most
\[
    2^{\Lambda(\Lambda-1)/2+2}
    (\Delta_{\mathrm{pf}}+\beta_f-1)
    (5\Delta_{\mathrm{pf}}+4\beta_f-4)^\Lambda
\]
connected components.  Since the boundary of $E$ in $(0,1)$ is contained in this zero set, $E$ has at most
$
    M_{\Delta_{\mathrm{pf}},\beta_f,\Lambda}
$
connected components, where
\[
    M_{\Delta_{\mathrm{pf}},\beta_f,\Lambda}
    \eqdef
    1+
    2^{\Lambda(\Lambda-1)/2+2}
    (\Delta_{\mathrm{pf}}+\beta_f-1)
    (5\Delta_{\mathrm{pf}}+4\beta_f-4)^\Lambda.
\]

On every connected component $I$ of $E$, the derivative $h'$ has constant sign and satisfies
$
    |h'|\geq\delta_{\nabla}/\sqrt d
$.
Consequently, $h$ is strictly monotone on $I$.  Thus
$
    I_k
    \eqdef
    \{
        t\in I:
        |h(t)-\beta_k|
        \leq
        \eta
    \}
$
is either empty or an interval.  If $s,t\in I_k$, then the mean-value theorem gives
$
    \frac{\delta_{\nabla}}{\sqrt d}|t-s|
    \leq
    |h(t)-h(s)|
    \leq
    2\eta
$,
and therefore
\begin{equation}
\label{eq:pfaff_class_interval_length}
    \operatorname{diam}(I_k)
    \leq
    \frac{2\sqrt d\,\eta}{\delta_{\nabla}}
.
\end{equation}

The $j$th coordinates of the grid midpoints form the set
$
    \{(2i+1)/(2q):i=0,\dots,q-1\}.
$
Since consecutive such points are separated by $1/q$, any interval of length $\ell\geq0$ contains at most
$
    \lfloor q\ell\rfloor+1
    \leq
    q\ell+1
$
of them.  Moreover, every midpoint of a cell in $\mathcal{S}_{k,j}$ lying on the fixed grid line belongs to $E$ and, by $(\mathcal{S},\eta)$-symmetry, satisfies
$
    |h(t)-\beta_k|
    \leq
    \eta
$.
Hence~\eqref{eq:pfaff_class_interval_length} and the bound on the number of connected components of $E$ imply that, along the fixed grid line, at most
$
    M_{\Delta_{\mathrm{pf}},\beta_f,\Lambda}
    \big(
        1+
        \tfrac{2\sqrt d\,q\eta}{\delta_{\nabla}}
    \big)
$
grid midpoints can correspond to cells in $\mathcal{S}_{k,j}$.  
There are exactly $q^{d-1}$ grid lines parallel to the $j$-th coordinate axis.  It follows that
\begin{equation}
\label{eq:unsummed_bound}
    \#\mathcal{S}_{k,j}
    \leq
    M_{\Delta_{\mathrm{pf}},\beta_f,\Lambda}
    q^{d-1}
    \Big(
        1+
        \frac{2\sqrt d\,q\eta}{\delta_{\nabla}}
    \Big)
.
\end{equation}
Summing~\eqref{eq:unsummed_bound} over $j\in[d]$ and using
$
    \mathcal{S}_k
    =
    \bigsqcup_{j=1}^d\mathcal{S}_{k,j}
$
we deduce that
$
    \frac{\#\mathcal{S}_k}{q^d}
    \leq
    dM_{\Delta_{\mathrm{pf}},\beta_f,\Lambda}
    \big(
        \frac{1}{q}
        +
        \frac{2\sqrt d\,\eta}{\delta_{\nabla}}
    \big)
$.
Since $d\geq1$, this yields $
    \tfrac{\#\mathcal{S}_k}{q^d}
\leq
    2d^{3/2}M_{\Delta_{\mathrm{pf}},\beta_f,\Lambda}
    \,
    (
        \frac{1}{q}
        +
        \tfrac{\eta}{\delta_{\nabla}}
    )
$.
Since also $\#\mathcal{S}_k\leq q^d$, we conclude that
\[
    \frac{\#\mathcal{S}_k}{q^d}
    \leq
    \min
    \left\{
        1,\,
        2d^{3/2}M_{\Delta_{\mathrm{pf}},\beta_f,\Lambda}
        \left(
            \frac{1}{q}
            +
            \frac{\eta}{\delta_{\nabla}}
        \right)
    \right\}.
\]
The right-hand side is independent of $k$, and therefore taking the maximum over $k\in[\kappa]$ proves~\eqref{eq:class_complexity_bound}, with
$
    C_{d,\Delta_{\mathrm{pf}},\beta_f,\Lambda}
    \eqdef
    2d^{3/2}M_{\Delta_{\mathrm{pf}},\beta_f,\Lambda}.
$
\end{proof}
We now formulate a slightly more technical version of Theorem~\ref{thrm:Holder_Pfaffian_Symmetry_Scaling} which we will prove.  
\begin{lemma}[H\"older--Pfaffian Symmetry Scaling]
\label{lemma:Holder_Pfaffian_Symmetry_Scaling___technicalVersion}
Fix $d\in\mathbb{N}_+$, $\alpha\in(0,1]$, and $L\geq0$.  Let
$
    f:\mathbb{R}^d\to[0,1]
$
be Pfaffian of degree $\beta_f\in\mathbb{N}_+$ in a Pfaffian chain of length
$\Lambda\in\mathbb{N}_+$ and degree $\Delta_{\mathrm{pf}}\in\mathbb{N}_+$, and suppose that its restriction to $[0,1]^d$ is $\alpha$-H\"older with constant $L$.  Suppose further that
\[
    \delta_{\nabla}
    \eqdef
    \min_{x\in[0,1]^d}
    \|\nabla f(x)\|_2
    >
    0.
\]
There exists a constant $C_{d,\Delta_{\mathrm{pf}},\beta_f,\Lambda}>0$\footnote{As in Proposition~\ref{prop:PfaffCellComplexity}.} such that, for every $q\in\mathbb{N}_+$ and $\eta>0$, there exist
$
    \kappa\in[q^d]
$
and a $\kappa$-symmetry
$
    \mathcal{S}
    =
    \{\mathcal{S}_k\}_{k=1}^{\kappa}
$
at scale $q$ such that $f$ is $(\mathcal{S},\eta)$-symmetric and
\begin{align}
\label{eq:Holder_Pfaffian_kappa_sandwich}
    \frac{1}{
        \displaystyle
        \min\left\{
            1,\,
            C_{d,\Delta_{\mathrm{pf}},\beta_f,\Lambda}
            \left(
                \frac{1}{q}
                +
                \frac{\eta}{\delta_{\nabla}}
            \right)
        \right\}
    }
& \leq
    \kappa
\leq
    \min\left\{
        q^d,
        \max\left\{
            1,
            \left\lceil
                \frac{\min\{1,L\}}{2\eta}
            \right\rceil
        \right\}
    \right\},
\\
\label{eq:Holder_Pfaffian_class_sandwich}
    \frac{
        q^d
    }{
        \displaystyle
        \min\left\{
            q^d,
            \max\left\{
                1,
                \left\lceil
                    \frac{\min\{1,L\}}{2\eta}
                \right\rceil
            \right\}
        \right\}
    }
& \leq
    \max_{k\in[\kappa]}
    \#\mathcal{S}_k
  \leq
    q^d
    \min
    \left\{
        1,\,
        C_{d,\Delta_{\mathrm{pf}},\beta_f,\Lambda}
        \left(
            \frac{1}{q}
            +
            \frac{\eta}{\delta_{\nabla}}
        \right)
    \right\}
.
\end{align}
\end{lemma}
\begin{proof}[{Proof of Lemma~\ref{lemma:Holder_Pfaffian_Symmetry_Scaling___technicalVersion}}]
Fix $q\in\mathbb{N}_+$ and $\eta>0$.  Applying Proposition~\ref{prop:holder_implies_symmetry} to the restriction of $f$ to $[0,1]^d$, there exist
$
    \kappa\in[q^d]
$
and a $\kappa$-symmetry
$
    \mathcal{S}
    =
    \{\mathcal{S}_k\}_{k=1}^{\kappa}
$
at scale $q$ such that $f$ is $(\mathcal{S},\eta)$-symmetric and
\begin{equation}
\label{eq:Holder_Pfaffian_upper_kappa}
    \kappa
    \leq
    \min\left\{
        q^d,
        \max\left\{
            1,
            \left\lceil
                \frac{\min\{1,L\}}{2\eta}
            \right\rceil
        \right\}
    \right\}.
\end{equation}
We fix this symmetry $\mathcal{S}$ for the remainder of the proof.
\hfill\\
\noindent  
Since $f$ is Pfaffian with the stated format and
$
    \min_{x\in[0,1]^d}
    \|\nabla f(x)\|_2
    =
    \delta_{\nabla}
    >
    0,
$
Proposition~\ref{prop:PfaffCellComplexity} applies to this same symmetry $\mathcal{S}$.  Consequently,
\begin{equation}
\label{eq:Holder_Pfaffian_upper_class}
    \frac{\#\mathcal{S}_k}{q^d}
    \leq
    \min
    \left\{
        1,\,
        C_{d,\Delta_{\mathrm{pf}},\beta_f,\Lambda}
        \left(
            \frac{1}{q}
            +
            \frac{\eta}{\delta_{\nabla}}
        \right)
    \right\}
    \qquad
    \mbox{for every }
    k\in[\kappa].
\end{equation}
Since $\mathcal{S}$ is a partition of $\mathcal{Q}_{q:d}$ and
$
    \#\mathcal{Q}_{q:d}=q^d,
$
we have
$
    q^d
=
    \sum_{k=1}^{\kappa}
    \#\mathcal{S}_k
$.
Applying~\eqref{eq:Holder_Pfaffian_upper_class} to every term on the right-hand side yields
$
    q^d
\leq
    \kappa q^d
    \min
    \{
        1,\,
        C_{d,\Delta_{\mathrm{pf}},\beta_f,\Lambda}
        (
            \tfrac{1}{q}
            +
            \tfrac{\eta}{\delta_{\nabla}}
        )
    \}$.  Since $q^d>0$, division across this estimate by $q^d$ produces 
\[
    \kappa
    \geq
    \frac{1}{
        \displaystyle
        \min\left\{
            1,\,
            C_{d,\Delta_{\mathrm{pf}},\beta_f,\Lambda}
            \left(
                \frac{1}{q}
                +
                \frac{\eta}{\delta_{\nabla}}
            \right)
        \right\}
    }.
\]
Combining this estimate with~\eqref{eq:Holder_Pfaffian_upper_kappa} proves~\eqref{eq:Holder_Pfaffian_kappa_sandwich}.
\hfill\\
\noindent
We next bound the largest symmetry class.  Since the $\kappa$ non-empty classes partition the $q^d$ elements of $\mathcal{Q}_{q:d}$,
$
    q^d
=
    \sum_{k=1}^{\kappa}
    \#\mathcal{S}_k
\leq
    \kappa
    \max_{k\in[\kappa]}
    \#\mathcal{S}_k
$; thus, $
    \max_{k\in[\kappa]}
    \#\mathcal{S}_k
\geq
    \frac{q^d}{\kappa}
$.   Using~\eqref{eq:Holder_Pfaffian_upper_kappa} in the denominator gives
\begin{equation}
\label{eq:symmetries__muhahaha}
    \max_{k\in[\kappa]}
    \#\mathcal{S}_k
    \geq
    \frac{
        q^d
    }{
        \displaystyle
        \min\left\{
            q^d,
            \max\left\{
                1,
                \left\lceil
                    \frac{\min\{1,L\}}{2\eta}
                \right\rceil
            \right\}
        \right\}
    }.
\end{equation}
Taking the maximum over $k\in[\kappa]$ in~\eqref{eq:Holder_Pfaffian_upper_class} gives the upper bound and proves~\eqref{eq:Holder_Pfaffian_class_sandwich} as
\[
    \max_{k\in[\kappa]}
    \#\mathcal{S}_k
    \leq
    q^d
    \min
    \left\{
        1,\,
        C_{d,\Delta_{\mathrm{pf}},\beta_f,\Lambda}
        \left(
            \frac{1}{q}
            +
            \frac{\eta}{\delta_{\nabla}}
        \right)
    \right\}.
\]
\end{proof}
Now, considering the scaling $\eta_q\eqdef \Theta(q^{-r})$ in the estimates of Lemma~\ref{lemma:Holder_Pfaffian_Symmetry_Scaling___technicalVersion} yields Theorem~\ref{thrm:Holder_Pfaffian_Symmetry_Scaling}.
\begin{proof}[{Proof of Theorem~\ref{thrm:Holder_Pfaffian_Symmetry_Scaling}]}
It remains to prove~\eqref{eq:Holder_Pfaffian_matching_scaling}.  Fix $c>0$ and $r\in(0,1]$.  For every $q\in\mathbb{N}_+$, set
$
    \eta_q\eqdef cq^{-r}
$
and choose a symmetry
$
    \mathcal{S}^{(q)}
    =
    \{\mathcal{S}^{(q)}_k\}_{k=1}^{\kappa_q}
$
as above with $\eta=\eta_q$.  Since $r\leq1$,
$
    q^{-1}\leq q^{-r}
$
for every $q\in\mathbb{N}_+$.  Hence Proposition~\ref{prop:PfaffCellComplexity}, or equivalently~\eqref{eq:Holder_Pfaffian_upper_class}, gives
\begin{equation}
\label{eq:estimates__counting}
    \frac{\#\mathcal{S}^{(q)}_k}{q^d}
\leq
    C_{d,\Delta_{\mathrm{pf}},\beta_f,\Lambda}
    \,
    \Big(
        1+\frac{c}{\delta_{\nabla}}
    \Big)
    \,
    q^{-r}
\end{equation}
for each $k\in[\kappa_q]$.  Summing across~\eqref{eq:estimates__counting} over all $k\in[\kappa_q]$ and using the fact that
$
    \sum_{k=1}^{\kappa_q}
    \#\mathcal{S}^{(q)}_k
    =
    q^d
$ yields $
    1
\leq
    \kappa_q
    C_{d,\Delta_{\mathrm{pf}},\beta_f,\Lambda}
    \,(
        1+\tfrac{c}{\delta_{\nabla}}
    )
    \,
    q^{-r}
$.  Consequently,
\begin{equation}
\label{eq:Holder_Pfaffian_scaling_lower_kappa}
    \kappa_q
    \geq
    \frac{
        q^r
    }{
        C_{d,\Delta_{\mathrm{pf}},\beta_f,\Lambda}
        \left(
            1+\frac{c}{\delta_{\nabla}}
        \right)
    }.
\end{equation}
Moreover, since $\delta_{\nabla}>0$, the function $f$ is non-constant on $[0,1]^d$.  Since its restriction to $[0,1]^d$ is $\alpha$-H\"older with constant $L$, it follows that $L>0$.  Therefore Proposition~\ref{prop:holder_implies_symmetry} and the identity $\eta_q=cq^{-r}$ give $
    \kappa_q
\leq
    \min
    \{
        q^d,
        \max\{
            1,
            \lceil
                \tfrac{\min\{1,L\}}{2c}
                q^r
            \rceil
        \}
    \}$.  Since $q^r\geq1$, we further obtain
$
    \kappa_q
\leq
    (
        1+
        \frac{\min\{1,L\}}{2c}
    )
    \,
    q^r$.  Together with~\eqref{eq:Holder_Pfaffian_scaling_lower_kappa}, this proves
$
    \kappa_q
    \asymp
    q^r.
$
The estimate in~\eqref{eq:estimates__counting} implies that $
    \max_{k\in[\kappa_q]}
    \#\mathcal{S}^{(q)}_k
    \leq
    C_{d,\Delta_{\mathrm{pf}},\beta_f,\Lambda}
    \left(
        1+\frac{c}{\delta_{\nabla}}
    \right)
    q^{d-r}
$.   Conversely, since $    q^d
=
    \sum_{k=1}^{\kappa_q}
    \#\mathcal{S}^{(q)}_k
\leq
    \kappa_q
    \,
    \max_{k\in[\kappa_q]}
    \#\mathcal{S}^{(q)}_k$ then 
$
    \max_{k\in[\kappa_q]}
    \#\mathcal{S}^{(q)}_k
    \geq
    \frac{q^d}{\kappa_q}
$; combining this with the upper bound on $\kappa_q$ implied by~\eqref{eq:Holder_Pfaffian_upper_kappa} we deduce that
\begin{equation}
\label{eq:maximize_and_conclude}
    \max_{k\in[\kappa_q]}
    \#\mathcal{S}^{(q)}_k
    \geq
    \frac{
        q^{d-r}
    }{
        1+
        \frac{\min\{1,L\}}{2c}
    }.
\end{equation}
Therefore,~\eqref{eq:maximize_and_conclude} implies that  $
    \max_{k\in[\kappa_q]}
    \#\mathcal{S}^{(q)}_k
    \asymp
    q^{d-r},
$
which proves~\eqref{eq:Holder_Pfaffian_matching_scaling}.
\end{proof}


\bibliography{bib}

@article{bartlett2017spectrally,
	author = {Bartlett, Peter L and Foster, Dylan J and Telgarsky, Matus J},
	journal = {Advances in neural information processing systems},
	title = {Spectrally-normalized margin bounds for neural networks},
	volume = {30},
	year = {2017}}

@inproceedings{neyshabur2018a,
	author = {Behnam Neyshabur and Srinadh Bhojanapalli and Nathan Srebro},
	booktitle = {International Conference on Learning Representations},
	title = {A {PAC}-Bayesian Approach to Spectrally-Normalized Margin Bounds for Neural Networks},
	url = {https://openreview.net/forum?id=Skz_WfbCZ},
	year = {2018}}

@article{schmidt2020nonparametric,
	author = {Schmidt-Hieber, Johannes},
	journal = {The Annals of Statistics},
	number = {4},
	pages = {1875--1897},
	publisher = {Institute of Mathematical Statistics},
	title = {Nonparametric regression via deep neural networks},
	volume = {48},
	year = {2020}}

@article{dematteis2018rogue,
	author = {Dematteis, Giovanni and Grafke, Tobias and Vanden-Eijnden, Eric},
	doi = {10.1073/pnas.1710670115},
	journal = {Proceedings of the National Academy of Sciences},
	note = {Large-deviation and extreme value analysis for rare rogue-wave events in ocean dynamics},
	number = {5},
	pages = {855--860},
	title = {Rogue waves and large deviations in deep sea},
	volume = {115},
	year = {2018}}

@article{trenberth2015attribution,
	author = {Trenberth, Kevin E. and Fasullo, John T. and Shepherd, Theodore G.},
	journal = {Nature Climate Change},
	note = {Perspective on framing attribution of climate extremes},
	pages = {725--730},
	title = {Attribution of climate extreme events},
	volume = {5},
	year = {2015}}

@book{embrechts1997modelling,
	author = {Embrechts, Paul and Kl{\"u}ppelberg, Claudia and Mikosch, Thomas},
	note = {Foundational text bridging EVT and finance/insurance},
	publisher = {Springer},
	title = {Modelling Extremal Events: for Insurance and Finance},
	year = {1997}}

@article{madden2024memory,
	author = {Madden, Liam and Thrampoulidis, Christos},
	journal = {SIAM Journal on Mathematics of Data Science},
	number = {3},
	pages = {679--702},
	publisher = {SIAM},
	title = {Memory capacity of two layer neural networks with smooth activations},
	volume = {6},
	year = {2024}}

@article{kratsios2023small,
	author = {Kratsios, Anastasis and Debarnot, Valentin and Dokmani{\'c}, Ivan},
	journal = {Journal of Machine Learning Research},
	number = {170},
	pages = {1--48},
	title = {Small transformers compute universal metric embeddings},
	volume = {24},
	year = {2023}}

@inproceedings{kim2023provable,
	author = {Kim, Junghwan and Kim, Michelle and Mozafari, Barzan},
	booktitle = {The Eleventh International Conference on Learning Representations},
	title = {Provable memorization capacity of transformers},
	year = {2023}}

@inproceedings{vardi2022on,
	author = {Gal Vardi and Gilad Yehudai and Ohad Shamir},
	booktitle = {International Conference on Learning Representations},
	title = {On the Optimal Memorization Power of Re{LU} Neural Networks},
	url = {https://openreview.net/forum?id=MkTPtnjeYTV},
	year = {2022}}

@article{dirksen2024memorization,
	author = {Dirksen, Sjoerd and Finke, Patrick and Genzel, Martin},
	journal = {Journal of Machine Learning Research},
	number = {347},
	pages = {1--38},
	title = {Memorization With Neural Nets: Going Beyond the Worst Case},
	volume = {25},
	year = {2024}}

@article{neufeld2023universal,
	author = {Neufeld, Ariel and Schmocker, Philipp},
	journal = {arXiv preprint arXiv:2312.08410},
	title = {Universal approximation property of random neural networks},
	year = {2023}}

@article{hanin2019universal,
	author = {Hanin, Boris},
	journal = {Mathematics},
	number = {10},
	pages = {992},
	publisher = {MDPI},
	title = {Universal function approximation by deep neural nets with bounded width and relu activations},
	volume = {7},
	year = {2019}}

@article{zhang2024deep,
	author = {Zhang, Shijun and Lu, Jianfeng and Zhao, Hongkai},
	journal = {Journal of Machine Learning Research},
	number = {35},
	pages = {1--39},
	title = {Deep network approximation: Beyond {ReLU} to diverse activation functions},
	volume = {25},
	year = {2024}}

@article{bolcskei2019optimal,
	author = {Bolcskei, Helmut and Grohs, Philipp and Kutyniok, Gitta and Petersen, Philipp},
	journal = {SIAM Journal on Mathematics of Data Science},
	number = {1},
	pages = {8--45},
	publisher = {SIAM},
	title = {Optimal approximation with sparsely connected deep neural networks},
	volume = {1},
	year = {2019}}

@article{el2024efficiency,
	author = {El Hanchi, Ayoub and Maddison, Chris J and Erdogdu, Murat A},
	journal = {Advances in Neural Information Processing Systems},
	pages = {98596--98624},
	title = {On the efficiency of ERM in feature learning},
	volume = {37},
	year = {2024}}

@inproceedings{mousavihosseini2024robust,
	author = {Alireza Mousavi-Hosseini and Adel Javanmard and Murat A Erdogdu},
	booktitle = {NeurIPS 2024 Workshop on Mathematics of Modern Machine Learning},
	title = {Robust Feature Learning for Multi-Index Models in High Dimensions},
	url = {https://openreview.net/forum?id=3VDMVmQxED},
	year = {2024}}

@article{gonon2023random,
	author = {Gonon, Lukas},
	journal = {Journal of Machine Learning Research},
	number = {189},
	pages = {1--51},
	title = {Random feature neural networks learn Black-Scholes type PDEs without curse of dimensionality},
	volume = {24},
	year = {2023}}

@inproceedings{cheng2023a,
	author = {Tin Sum Cheng and Aurelien Lucchi and Anastasis Kratsios and Ivan Dokmani{\'c} and David Belius},
	booktitle = {Thirty-seventh Conference on Neural Information Processing Systems},
	title = {A Theoretical Analysis of the Test Error of Finite-Rank Kernel Ridge Regression},
	url = {https://openreview.net/forum?id=lk6KDG6qI7},
	year = {2023}}

@inproceedings{pmlr-v235-cheng24g,
	author = {Cheng, Tin Sum and Lucchi, Aurelien and Kratsios, Anastasis and Belius, David},
	booktitle = {Proceedings of the 41st International Conference on Machine Learning},
	editor = {Salakhutdinov, Ruslan and Kolter, Zico and Heller, Katherine and Weller, Adrian and Oliver, Nuria and Scarlett, Jonathan and Berkenkamp, Felix},
	month = {21--27 Jul},
	pages = {8141--8162},
	publisher = {PMLR},
	series = {Proceedings of Machine Learning Research},
	title = {Characterizing Overfitting in Kernel Ridgeless Regression Through the Eigenspectrum},
	url = {https://proceedings.mlr.press/v235/cheng24g.html},
	volume = {235},
	year = {2024}}

@article{gonon2020risk,
	author = {Gonon, Lukas and Grigoryeva, Lyudmila and Ortega, Juan-Pablo},
	journal = {Journal of Machine Learning Research},
	number = {240},
	pages = {1--61},
	title = {Risk bounds for reservoir computing},
	volume = {21},
	year = {2020}}

@article{daubechies2022nonlinear,
	author = {Daubechies, Ingrid and DeVore, Ronald and Foucart, Simon and Hanin, Boris and Petrova, Guergana},
	journal = {Constructive Approximation},
	number = {1},
	pages = {127--172},
	publisher = {Springer},
	title = {Nonlinear approximation and (deep) ReLU networks},
	volume = {55},
	year = {2022}}

@article{schneider2025nonlocal,
	author = {Schneider, Cornelia and Ullrich, Mario and Vybiral, Jan},
	journal = {arXiv preprint arXiv:2504.04847},
	title = {Nonlocal techniques for the analysis of deep ReLU neural network approximations},
	year = {2025}}

@article{MR4659237,
	author = {Schneider, Cornelia and Vyb\'iral, Jan},
	doi = {10.1016/j.acha.2023.101605},
	fjournal = {Applied and Computational Harmonic Analysis. Time-Frequency and Time-Scale Analysis, Wavelets, Numerical Algorithms, and Applications},
	issn = {1063-5203,1096-603X},
	journal = {Appl. Comput. Harmon. Anal.},
	mrclass = {68T07 (11A25 42C15)},
	mrnumber = {4659237},
	pages = {Paper No. 101605, 16},
	title = {A multivariate {R}iesz basis of {R}e{LU} neural networks},
	url = {https://doi.org/10.1016/j.acha.2023.101605},
	volume = {68},
	year = {2024}}

@article{shi2025learning,
	author = {Shi, Zhongjie and Yu, Zhan and Zhou, Ding-Xuan},
	journal = {Constructive Approximation},
	pages = {1--44},
	publisher = {Springer},
	title = {Learning theory of distribution regression with neural networks},
	year = {2025}}

@article{brugiapaglia2024physics,
	author = {Brugiapaglia, Simone and Dexter, Nick and Karam, Samir and Wang, Weiqi},
	journal = {arXiv preprint arXiv:2406.01539},
	title = {Physics-informed deep learning and compressive collocation for high-dimensional diffusion-reaction equations: practical existence theory and numerics},
	year = {2024}}

@book{lorentz1996constructive,
	author = {Lorentz, George G and von Golitschek, Manfred and Makovoz, Yuly},
	publisher = {Citeseer},
	title = {Constructive approximation: advanced problems},
	volume = {304},
	year = {1996}}

@article{kratsios2025kolmogorov,
	author = {Kratsios, Anastasis and Furuya, Takashi},
	journal = {arXiv preprint arXiv:2504.15110},
	title = {Kolmogorov-Arnold Networks: Approximation and Learning Guarantees for Functions and their Derivatives},
	year = {2025}}

@article{achiam2023gpt,
	author = {Achiam, Josh and Adler, Steven and Agarwal, Sandhini and Ahmad, Lama and Akkaya, Ilge and Aleman, Florencia Leoni and Almeida, Diogo and Altenschmidt, Janko and Altman, Sam and Anadkat, Shyamal and others},
	journal = {arXiv preprint arXiv:2303.08774},
	title = {Gpt-4 technical report},
	year = {2023}}

@article{alon1997scale,
	author = {Alon, Noga and Ben-David, Shai and Cesa-Bianchi, Nicolo and Haussler, David},
	journal = {Journal of the ACM (JACM)},
	number = {4},
	pages = {615--631},
	publisher = {ACM New York, NY, USA},
	title = {Scale-sensitive dimensions, uniform convergence, and learnability},
	volume = {44},
	year = {1997}}

@inproceedings{geiger2020scaling,
	author = {Geiger, Mario and Spigler, Stefano and d'Ascoli, Stephane and Sagun, Levent and Biroli, Giulio and Wyart, Matthieu},
	booktitle = {International Conference on Machine Learning},
	organization = {PMLR},
	pages = {3647--3657},
	title = {Scaling description of generalization with number of parameters in deep learning},
	year = {2020}}

@incollection{MR1689432,
	author = {DeVore, Ronald A.},
	booktitle = {Acta numerica, 1998},
	doi = {10.1017/S0962492900002816},
	isbn = {0-521-64316-3},
	mrclass = {41A46 (42C40 65D15)},
	mrnumber = {1689432},
	mrreviewer = {Rong\ Qing\ Jia},
	pages = {51--150},
	publisher = {Cambridge Univ. Press, Cambridge},
	series = {Acta Numer.},
	title = {Nonlinear approximation},
	url = {https://doi.org/10.1017/S0962492900002816},
	volume = {7},
	year = {1998}}

@inproceedings{rahimi2007random,
	author = {Rahimi, Ali and Recht, Benjamin},
	booktitle = {Advances in Neural Information Processing Systems},
	title = {Random Features for Large-Scale Kernel Machines},
	volume = {20},
	year = {2007}}

@article{cybenko1989approximation,
	author = {Cybenko, George},
	journal = {Mathematics of Control, Signals and Systems},
	number = {4},
	pages = {303--314},
	publisher = {Springer},
	title = {Approximation by superpositions of a sigmoidal function},
	volume = {2},
	year = {1989}}

@article{hornik1989multilayer,
	author = {Hornik, Kurt and Stinchcombe, Maxwell and White, Halbert},
	journal = {Neural Networks},
	number = {5},
	pages = {359--366},
	publisher = {Pergamon},
	title = {Multilayer feedforward networks are universal approximators},
	volume = {2},
	year = {1989}}

@inproceedings{jacot2018ntk,
	author = {Jacot, Arthur and Gabriel, Franck and Hongler, Cl{\'e}ment},
	booktitle = {Advances in Neural Information Processing Systems (NeurIPS)},
	title = {Neural tangent kernel: Convergence and generalization in neural networks},
	volume = {31},
	year = {2018}}

@article{mei2022generalization,
	author = {Mei, Song and Montanari, Andrea},
	journal = {Communications on Pure and Applied Mathematics},
	number = {4},
	pages = {667--766},
	publisher = {Wiley Online Library},
	title = {The generalization error of random features regression: Precise asymptotics and the double descent curve},
	volume = {75},
	year = {2022}}

@article{hong2024bridging,
	author = {Hong, Ruiyang and Kratsios, Anastasis},
	journal = {arXiv preprint arXiv:2409.12335},
	title = {Bridging the gap between approximation and learning via optimal approximation by relu mlps of maximal regularity},
	year = {2024}}

@inproceedings{brugiapagliapractical2024,
	author = {Brugiapaglia, Simone},
	booktitle = {15th International Conference on Sampling Theory and Applications},
	title = {Practical existence theorems for deep learning approximation in high dimensions},
	year = {2024}}

@article{franco2025practical,
	author = {Franco, Nicola Rares and Brugiapaglia, Simone},
	journal = {Foundations of Data Science},
	number = {1},
	pages = {72--98},
	publisher = {Foundations of Data Science},
	title = {A practical existence theorem for reduced order models based on convolutional autoencoders},
	volume = {7},
	year = {2025}}

@article{cheng2023theoretical,
	author = {Cheng, Tin Sum and Lucchi, Aurelien and Kratsios, Anastasis and Dokmani{\'c}, Ivan and Belius, David},
	journal = {Advances in Neural Information Processing Systems},
	pages = {4767--4798},
	title = {A theoretical analysis of the test error of finite-rank kernel ridge regression},
	volume = {36},
	year = {2023}}

@article{gottlieb2017efficient,
	author = {Gottlieb, Lee-Ad and Kontorovich, Aryeh and Krauthgamer, Robert},
	journal = {IEEE Transactions on Information Theory},
	number = {8},
	pages = {4838--4849},
	publisher = {IEEE},
	title = {Efficient regression in metric spaces via approximate lipschitz extension},
	volume = {63},
	year = {2017}}

@inproceedings{kidger2020universal,
	author = {Kidger, Patrick and Lyons, Terry},
	booktitle = {Conference on learning theory},
	organization = {PMLR},
	pages = {2306--2327},
	title = {Universal approximation with deep narrow networks},
	year = {2020}}

@article{riegler2024generating,
	author = {Riegler, Erwin and B{\"u}hler, Alex and Pan, Yang and B{\"o}lcskei, Helmut},
	journal = {arXiv preprint arXiv:2412.05109},
	title = {Generating Rectifiable Measures through Neural Networks},
	year = {2024}}

@article{yarotsky2017error,
	author = {Yarotsky, Dmitry},
	journal = {Neural networks},
	pages = {103--114},
	publisher = {Elsevier},
	title = {Error bounds for approximations with deep ReLU networks},
	volume = {94},
	year = {2017}}

@article{shen2022optimal,
	author = {Shen, Zuowei and Yang, Haizhao and Zhang, Shijun},
	doi = {10.1016/j.matpur.2021.07.009},
	fjournal = {Journal de Math\'ematiques Pures et Appliqu\'ees. Neuvi\`eme S\'erie},
	issn = {0021-7824,1776-3371},
	journal = {J. Math. Pures Appl. (9)},
	mrclass = {41A63 (41A25 41A46 68V15)},
	mrnumber = {4351074},
	mrreviewer = {Bao\ Huai\ Sheng},
	pages = {101--135},
	title = {Optimal approximation rate of {R}e{LU} networks in terms of width and depth},
	url = {https://doi.org/10.1016/j.matpur.2021.07.009},
	volume = {157},
	year = {2022}}

@article{hou2023instance,
	author = {Hou, Songyan and Kassraie, Parnian and Kratsios, Anastasis and Krause, Andreas and Rothfuss, Jonas},
	journal = {Journal of Machine Learning Research},
	number = {349},
	pages = {1--51},
	title = {Instance-dependent generalization bounds via optimal transport},
	volume = {24},
	year = {2023}}

@article{petersen2018optimal,
	author = {Petersen, Philipp and Voigtlaender, Felix},
	journal = {Neural Networks},
	pages = {296--330},
	publisher = {Elsevier},
	title = {Optimal approximation of piecewise smooth functions using deep ReLU neural networks},
	volume = {108},
	year = {2018}}

@inproceedings{yarotsky2018optimal,
	author = {Yarotsky, Dmitry},
	booktitle = {Conference on learning theory},
	organization = {PMLR},
	pages = {639--649},
	title = {Optimal approximation of continuous functions by very deep ReLU networks},
	year = {2018}}

@article{cohen2022optimal,
	author = {Cohen, Albert and DeVore, Ronald and Petrova, Guergana and Wojtaszczyk, Przemyslaw},
	journal = {Foundations of Computational Mathematics},
	number = {3},
	pages = {607--648},
	publisher = {Springer},
	title = {Optimal stable nonlinear approximation},
	volume = {22},
	year = {2022}}

@article{petrova2023lipschitz,
	author = {Petrova, Guergana and Wojtaszczyk, Przemys{\l}aw},
	journal = {Constructive Approximation},
	number = {2},
	pages = {759--805},
	publisher = {Springer},
	title = {Lipschitz widths},
	volume = {57},
	year = {2023}}

@book{pinkus2012n,
	author = {Pinkus, Allan},
	doi = {10.1007/978-3-642-69894-1},
	isbn = {3-540-13638-X},
	mrclass = {41-02 (41A46 46E35)},
	mrnumber = {774404},
	mrreviewer = {Klaus\ H\"ollig},
	pages = {x+291},
	publisher = {Springer-Verlag, Berlin},
	series = {Ergebnisse der Mathematik und ihrer Grenzgebiete (3) [Results in Mathematics and Related Areas (3)]},
	title = {{$n$}-widths in approximation theory},
	url = {https://doi.org/10.1007/978-3-642-69894-1},
	volume = {7},
	year = {1985}}

@article{kratsios2025quantifying,
	author = {Kratsios, Anastasis and Zvigelsky, Dennis and Hart, Bradd},
	journal = {arXiv preprint arXiv:2508.18526},
	title = {Quantifying The Limits of AI Reasoning: Systematic Neural Network Representations of Algorithms},
	year = {2025}}

@misc{kratsios2026algorithmicfoundationsdeeplearning,
	archiveprefix = {arXiv},
	author = {Anastasis Kratsios and Simone Brugiapaglia and Bum Jun Kim and Gregory Cousins and Haitz S{\'a}ez de Oc{\'a}riz Borde},
	eprint = {2606.26705},
	primaryclass = {cs.LG},
	title = {Algorithmic Foundations of Deep Learning: Complexity-Theoretic Rates and a Characterization of Universal Approximation},
	url = {https://arxiv.org/abs/2606.26705},
	year = {2026}}

@inproceedings{yosinski2014transferable,
	author = {Yosinski, Jason and Clune, Jeff and Bengio, Yoshua and Lipson, Hod},
	booktitle = {Advances in Neural Information Processing Systems},
	title = {How Transferable are Features in Deep Neural Networks?},
	volume = {27},
	year = {2014}}

@inproceedings{donahue2014decaf,
	author = {Donahue, Jeff and Jia, Yangqing and Vinyals, Oriol and Hoffman, Judy and Zhang, Ning and Tzeng, Eric and Darrell, Trevor},
	booktitle = {Proceedings of the 31st International Conference on Machine Learning},
	pages = {647--655},
	title = {{DeCAF}: A Deep Convolutional Activation Feature for Generic Visual Recognition},
	year = {2014}}

@inproceedings{girshick2014rich,
	author = {Girshick, Ross and Donahue, Jeff and Darrell, Trevor and Malik, Jitendra},
	booktitle = {Proceedings of the IEEE Conference on Computer Vision and Pattern Recognition},
	pages = {580--587},
	title = {Rich Feature Hierarchies for Accurate Object Detection and Semantic Segmentation},
	year = {2014}}

@article{ganin2016domain,
	author = {Ganin, Yaroslav and Ustinova, Evgeniya and Ajakan, Hana and Germain, Pascal and Larochelle, Hugo and Laviolette, Fran{\c{c}}ois and Marchand, Mario and Lempitsky, Victor},
	journal = {Journal of Machine Learning Research},
	number = {59},
	pages = {1--35},
	title = {Domain-Adversarial Training of Neural Networks},
	volume = {17},
	year = {2016}}

@inproceedings{houlsby2019parameter,
	author = {Houlsby, Neil and Giurgiu, Andrei and Jastrzebski, Stanislaw and Morrone, Bruna and de Laroussilhe, Quentin and Gesmundo, Andrea and Attariyan, Mona and Gelly, Sylvain},
	booktitle = {Proceedings of the 36th International Conference on Machine Learning},
	pages = {2790--2799},
	title = {Parameter-Efficient Transfer Learning for {NLP}},
	year = {2019}}

@inproceedings{hu2022lora,
	author = {Hu, Edward J. and Shen, Yelong and Wallis, Phillip and Allen-Zhu, Zeyuan and Li, Yuanzhi and Wang, Shean and Wang, Lu and Chen, Weizhu},
	booktitle = {International Conference on Learning Representations},
	title = {{LoRA}: Low-Rank Adaptation of Large Language Models},
	year = {2022}}

@inproceedings{dettmers2023qlora,
	author = {Dettmers, Tim and Pagnoni, Artidoro and Holtzman, Ari and Zettlemoyer, Luke},
	booktitle = {Advances in Neural Information Processing Systems},
	title = {QLoRA: Efficient Finetuning of Quantized LLMs},
	volume = {36},
	year = {2023}}

@article{bartlett2005local,
	author = {Bartlett, Peter L. and Bousquet, Olivier and Mendelson, Shahar},
	doi = {10.1214/009053605000000282},
	journal = {The Annals of Statistics},
	number = {4},
	pages = {1497--1537},
	publisher = {Institute of Mathematical Statistics},
	title = {Local Rademacher complexities},
	volume = {33},
	year = {2005}}

@article{bartl2025we,
	author = {Bartl, Daniel and Mendelson, Shahar},
	journal = {arXiv preprint arXiv:2502.15118},
	title = {Do We Really Need the Rademacher Complexities?},
	year = {2025}}

@inproceedings{barzilai2024generalization,
	author = {Barzilai, Daniel and Shamir, Ohad},
	booktitle = {Proceedings of the 41st International Conference on Machine Learning},
	pages = {3096--3132},
	publisher = {PMLR},
	series = {Proceedings of Machine Learning Research},
	title = {Generalization in Kernel Regression Under Realistic Assumptions},
	volume = {235},
	year = {2024}}

@article{tsigler2023benign,
	author = {Tsigler, Alexander and Bartlett, Peter L},
	journal = {Journal of Machine Learning Research},
	number = {123},
	pages = {1--76},
	title = {Benign overfitting in ridge regression},
	volume = {24},
	year = {2023}}

@inproceedings{cheng2024comprehensive,
	author = {Cheng, Tin Sum and Lucchi, Aurelien and Kratsios, Anastasis and Belius, David},
	booktitle = {Advances in Neural Information Processing Systems},
	pages = {24659--24723},
	publisher = {Curran Associates, Inc.},
	title = {A Comprehensive Analysis on the Learning Curve in Kernel Ridge Regression},
	volume = {37},
	year = {2024}}

@article{benitez2024out,
	author = {Benitez, Jose Antonio Lara and Furuya, Takashi and Faucher, Florian and Kratsios, Anastasis and Tricoche, Xavier and de Hoop, Maarten V},
	journal = {Journal of Computational Physics},
	pages = {113168},
	publisher = {Elsevier},
	title = {Out-of-distributional risk bounds for neural operators with applications to the Helmholtz equation},
	volume = {513},
	year = {2024}}

@article{pan2010survey,
	author = {Pan, Sinno Jialin and Yang, Qiang},
	doi = {10.1109/TKDE.2009.191},
	journal = {IEEE Transactions on Knowledge and Data Engineering},
	number = {10},
	pages = {1345--1359},
	title = {A Survey on Transfer Learning},
	volume = {22},
	year = {2010}}

@inproceedings{benDavid2006analysis,
	author = {Ben-David, Shai and Blitzer, John and Crammer, Koby and Pereira, Fernando},
	booktitle = {Advances in Neural Information Processing Systems},
	pages = {137--144},
	title = {Analysis of Representations for Domain Adaptation},
	volume = {19},
	year = {2006}}

@article{benDavid2010theory,
	author = {Ben-David, Shai and Blitzer, John and Crammer, Koby and Kulesza, Alex and Pereira, Fernando and Vaughan, Jennifer Wortman},
	doi = {10.1007/s10994-009-5152-4},
	journal = {Machine Learning},
	number = {1--2},
	pages = {151--175},
	title = {A Theory of Learning from Different Domains},
	volume = {79},
	year = {2010}}

@inproceedings{mansour2009domain,
	author = {Mansour, Yishay and Mohri, Mehryar and Rostamizadeh, Afshin},
	booktitle = {Proceedings of the 22nd Annual Conference on Learning Theory},
	publisher = {Omnipress},
	title = {Domain Adaptation: Learning Bounds and Algorithms},
	year = {2009}}

@article{maurer2016benefit,
	author = {Maurer, Andreas and Pontil, Massimiliano and Romera-Paredes, Bernardino},
	journal = {Journal of Machine Learning Research},
	number = {81},
	pages = {1--32},
	title = {The Benefit of Multitask Representation Learning},
	volume = {17},
	year = {2016}}

@inproceedings{long2015learning,
	author = {Long, Mingsheng and Cao, Yue and Wang, Jianmin and Jordan, Michael I.},
	booktitle = {Proceedings of the 32nd International Conference on Machine Learning},
	pages = {97--105},
	series = {Proceedings of Machine Learning Research},
	title = {Learning Transferable Features with Deep Adaptation Networks},
	volume = {37},
	year = {2015}}

@article{rojasCarulla2018invariant,
	author = {Rojas-Carulla, Mateo and Sch{\"o}lkopf, Bernhard and Turner, Richard and Peters, Jonas},
	journal = {Journal of Machine Learning Research},
	number = {36},
	pages = {1--34},
	title = {Invariant Models for Causal Transfer Learning},
	volume = {19},
	year = {2018}}

@article{bartl2022monte,
	author = {Bartl, Daniel and Mendelson, Shahar},
	journal = {The Annals of Applied Probability},
	number = {4},
	pages = {3146--3198},
	publisher = {Institute of Mathematical Statistics},
	title = {On Monte-Carlo methods in convex stochastic optimization},
	volume = {32},
	year = {2022}}

@article{bartl2022montecarlo,
	author = {Bartl, Daniel and Mendelson, Shahar},
	doi = {10.1214/22-AAP1781},
	journal = {The Annals of Applied Probability},
	number = {4},
	pages = {3146--3198},
	title = {On Monte-Carlo Methods in Convex Stochastic Optimization},
	volume = {32},
	year = {2022}}

@article{lugosi2020risk,
	author = {Lugosi, G{\'a}bor and Mendelson, Shahar},
	doi = {10.4171/JEMS/937},
	journal = {Journal of the European Mathematical Society},
	number = {3},
	pages = {925--965},
	title = {Risk Minimization by Median-of-Means Tournaments},
	volume = {22},
	year = {2020}}

@article{lugosi2019regularization,
	author = {Lugosi, G{\'a}bor and Mendelson, Shahar},
	doi = {10.3150/18-BEJ1046},
	journal = {Bernoulli},
	number = {3},
	pages = {2075--2106},
	title = {Regularization, Sparse Recovery, and Median-of-Means Tournaments},
	volume = {25},
	year = {2019}}

@inproceedings{holland2019better,
	author = {Holland, Matthew and Ikeda, Kazushi},
	booktitle = {Proceedings of the 36th International Conference on Machine Learning},
	pages = {2761--2770},
	publisher = {PMLR},
	series = {Proceedings of Machine Learning Research},
	title = {Better Generalization with Less Data Using Robust Gradient Descent},
	volume = {97},
	year = {2019}}

@inproceedings{vural2022mirror,
	author = {Vural, Nuri Mert and Yu, Lu and Balasubramanian, Krishna and Volgushev, Stanislav and Erdogdu, Murat A.},
	booktitle = {Proceedings of the Thirty-Fifth Conference on Learning Theory},
	pages = {65--102},
	publisher = {PMLR},
	series = {Proceedings of Machine Learning Research},
	title = {Mirror Descent Strikes Again: Optimal Stochastic Convex Optimization under Infinite Noise Variance},
	volume = {178},
	year = {2022}}

@article{polyak1963gradient,
	author = {Polyak, Boris T.},
	doi = {10.1016/0041-5553(63)90382-3},
	journal = {USSR Computational Mathematics and Mathematical Physics},
	number = {4},
	pages = {864--878},
	title = {Gradient Methods for Minimizing Functionals},
	volume = {3},
	year = {1963}}

@article{nesterov1983method,
	author = {Nesterov, Yurii E.},
	journal = {Doklady Akademii Nauk SSSR},
	number = {3},
	pages = {543--547},
	title = {A Method of Solving a Convex Programming Problem with Convergence Rate $O(1/k^2)$},
	volume = {269},
	year = {1983}}

@article{nemirovski2009robust,
	author = {Nemirovski, Arkadi and Juditsky, Anatoli and Lan, Guanghui and Shapiro, Alexander},
	doi = {10.1137/070704277},
	journal = {SIAM Journal on Optimization},
	number = {4},
	pages = {1574--1609},
	title = {Robust Stochastic Approximation Approach to Stochastic Programming},
	volume = {19},
	year = {2009}}

@inproceedings{moulines2011nonasymptotic,
	author = {Moulines, Eric and Bach, Francis},
	booktitle = {Advances in Neural Information Processing Systems},
	pages = {451--459},
	title = {Non-Asymptotic Analysis of Stochastic Approximation Algorithms for Machine Learning},
	volume = {24},
	year = {2011}}

@inproceedings{bach2013nonstrongly,
	author = {Bach, Francis and Moulines, Eric},
	booktitle = {Advances in Neural Information Processing Systems},
	pages = {773--781},
	title = {Non-Strongly-Convex Smooth Stochastic Approximation with Convergence Rate $O(1/n)$},
	volume = {26},
	year = {2013}}

@article{lan2012optimal,
	author = {Lan, Guanghui},
	doi = {10.1007/s10107-010-0434-y},
	journal = {Mathematical Programming},
	number = {1--2},
	pages = {365--397},
	title = {An Optimal Method for Stochastic Composite Optimization},
	volume = {133},
	year = {2012}}

@article{ghadimi2012optimal,
	author = {Ghadimi, Saeed and Lan, Guanghui},
	doi = {10.1137/110848864},
	journal = {SIAM Journal on Optimization},
	number = {4},
	pages = {1469--1492},
	title = {Optimal Stochastic Approximation Algorithms for Strongly Convex Stochastic Composite Optimization, {I}: A Generic Algorithmic Framework},
	volume = {22},
	year = {2012}}

@article{louzi2024stochastic,
	author = {Louzi, Azar},
	journal = {arXiv preprint arXiv:2412.06070},
	title = {Stochastic gradient descent revisited},
	year = {2024}}

@inproceedings{reddi2018convergence,
	author = {Reddi, Sashank J. and Kale, Satyen and Kumar, Sanjiv},
	booktitle = {International Conference on Learning Representations},
	title = {On the Convergence of {Adam} and Beyond},
	year = {2018}}

@inproceedings{zou2019sufficient,
	author = {Zou, Fangyu and Shen, Li and Jie, Zequn and Zhang, Weizhong and Liu, Wei},
	booktitle = {Proceedings of the IEEE/CVF Conference on Computer Vision and Pattern Recognition},
	pages = {11127--11135},
	title = {A Sufficient Condition for Convergences of {Adam} and {RMSProp}},
	year = {2019}}

@article{defossez2022simple,
	author = {D{\'e}fossez, Alexandre and Bottou, L{\'e}on and Bach, Francis and Usunier, Nicolas},
	journal = {Transactions on Machine Learning Research},
	title = {A Simple Convergence Proof of {Adam} and {AdaGrad}},
	year = {2022}}

@book{triebel2008function,
	author = {Triebel, Hans},
	number = {7},
	publisher = {European Mathematical Society},
	title = {Function spaces and wavelets on domains},
	year = {2008}}

@article{mendel2013ultrametric,
	author = {Mendel, Manor and Naor, Assaf},
	journal = {Inventiones Mathematicae},
	number = {1},
	pages = {1--54},
	title = {Ultrametric subsets with large Hausdorff dimension},
	volume = {192},
	year = {2013}}

@article{d2024vc,
	author = {D'Inverno, Giuseppe Alessio and Bianchini, Monica and Scarselli, Franco},
	journal = {arXiv preprint arXiv:2401.12362},
	title = {VC dimension of Graph Neural Networks with Pfaffian activation functions},
	year = {2024}}

@article{kratsios2026ominimal,
	author = {Kratsios, Anastasis and Cousins, Griffin and Borde, H. S. O. and Kim, B. J. and Brugiapaglia, Simone},
	journal = {arXiv preprint arXiv:2605.07097},
	title = {Every Feedforward Neural Network Definable in an o-Minimal Structure Has Finite Sample Complexity},
	year = {2026}}

@book{khovanskii1991fewnomials,
	address = {Providence, RI},
	author = {Khovanskii, Askold G.},
	doi = {10.1090/mmono/088},
	pages = {139},
	publisher = {American Mathematical Society},
	series = {Translations of Mathematical Monographs},
	title = {Fewnomials},
	translator = {Zdravkovska, Smilka},
	volume = {88},
	year = {1991}}

@article{lezeau2026tubular,
	author = {Lezeau, Paul and Lotz, Martin},
	journal = {arXiv preprint arXiv:2607.08370},
	title = {Tubular Neighbourhoods of Pfaffian Sets and Applications to Neural Networks},
	year = {2026}}

@article{karpinski1997polynomial,
	author = {Karpinski, Marek and Macintyre, Angus},
	doi = {10.1006/jcss.1997.1477},
	journal = {Journal of Computer and System Sciences},
	number = {1},
	pages = {169--176},
	title = {Polynomial Bounds for {VC} Dimension of Sigmoidal and General {Pfaffian} Neural Networks},
	volume = {54},
	year = {1997}}

@article{lion1998volumes,
	address = {Toulouse},
	author = {Lion, Jean-Marie and Rolin, Jean-Philippe},
	doi = {10.5802/afst.891},
	issn = {0240-2963},
	journal = {Annales de la Facult{\'e} des Sciences de Toulouse. Math{\'e}matiques},
	language = {French},
	mrnumber = {1658452},
	number = {1},
	pages = {93--112},
	publisher = {Universit{\'e} Paul Sabatier},
	series = {6},
	title = {Volumes, feuilles de {Rolle} de feuilletages analytiques et th{\'e}or{\`e}me de {Wilkie}},
	url = {https://doi.org/10.5802/afst.891},
	volume = {7},
	year = {1998},
	zbl = {0933.32014}}

@article{guingona2013vc,
	author = {Guingona, Vincent and Laskowski, Michael C},
	journal = {Archive for Mathematical Logic},
	number = {7},
	pages = {743--758},
	publisher = {Springer},
	title = {On VC-minimal theories and variants},
	volume = {52},
	year = {2013}}

@article{speissegger1997pfaffian,
	author = {Speissegger, Patrick},
	journal = {arXiv preprint math/9710220},
	title = {The Pfaffian closure on an o-minimal structure},
	year = {1997}}

@book{aliprantis1994infinite,
	address = {Berlin, Heidelberg},
	author = {Aliprantis, Charalambos D. and Border, Kim C.},
	doi = {10.1007/978-3-662-03004-2},
	isbn = {978-3-662-03006-6},
	publisher = {Springer-Verlag},
	series = {Studies in Economic Theory},
	title = {Infinite Dimensional Analysis: A Hitchhiker's Guide},
	volume = {4},
	year = {1994}}

@book{kechris1995classical,
	address = {New York},
	author = {Kechris, Alexander S.},
	isbn = {978-0-387-94374-9},
	publisher = {Springer},
	series = {Graduate Texts in Mathematics},
	title = {Classical Descriptive Set Theory},
	volume = {156},
	year = {1995}}

@article{FloydWarmuth1995,
	author = {Sally Floyd and Manfred K. Warmuth},
	doi = {10.1023/A:1022660318680},
	journal = {Machine Learning},
	pages = {269--304},
	title = {Sample Compression, Learnability, and the Vapnik--Chervonenkis Dimension},
	volume = {21},
	year = {1995}}

@article{MoranYehudayoff2016,
	author = {Shay Moran and Amir Yehudayoff},
	doi = {10.1145/2890490},
	journal = {Journal of the ACM},
	number = {3},
	pages = {21:1--21:10},
	title = {Sample Compression Schemes for {VC} Classes},
	volume = {63},
	year = {2016}}

@unpublished{ghoukasian2026closedform,
  author = {Ghoukasian, Hrad and Hong, Ruiyang and Kratsios, Anastasis},
  title  = {A Closed-Form Formula for Consistent Lipschitz Regression on
            Metric Spaces with Sparse Neural Network Realizations},
  year   = {2026},
  note   = {Manuscript submitted to Mathematical Statistics and Learning}
}

@article{stone1982optimal,
  author  = {Stone, Charles J.},
  title   = {Optimal Global Rates of Convergence for Nonparametric Regression},
  journal = {The Annals of Statistics},
  volume  = {10},
  number  = {4},
  pages   = {1040--1053},
  year    = {1982},
  doi     = {10.1214/aos/1176345969}
}

@article{balazs2026nearoptimal,
  author  = {Bal{\'a}zs, G{\'a}bor},
  title   = {Near-optimal Delta-convex Estimation of Lipschitz Functions},
  journal = {Journal of Machine Learning Research},
  volume  = {27},
  number  = {134},
  pages   = {1--41},
  year    = {2026},
  url     = {https://jmlr.org/papers/v27/25-2864.html}
}

@misc{bettinger2026revisiting,
  author        = {Bettinger, J{\'e}r{\'e}my and Portier, Fran{\c c}ois
                   and Saumard, Adrien},
  title         = {Revisiting Local Regression: Shape Regularity,
                   Uniform Rates, and the Limits of Random Splits},
  year          = {2026},
  eprint        = {2606.28641},
  archivePrefix = {arXiv},
  primaryClass  = {math.ST}
}

@inproceedings{tsai2019transformer,
  author    = {Tsai, Yao-Hung Hubert and Bai, Shaojie and Yamada, Makoto
               and Morency, Louis-Philippe and Salakhutdinov, Ruslan},
  title     = {Transformer Dissection: An Unified Understanding for
               Transformer's Attention via the Lens of Kernel},
  booktitle = {Proceedings of EMNLP-IJCNLP},
  pages     = {4344--4353},
  year      = {2019},
  doi       = {10.18653/v1/D19-1443}
}

@inproceedings{barzilai2025beyond,
  author    = {Barzilai, Daniel and Kornowski, Guy and Shamir, Ohad},
  title     = {Beyond Benign Overfitting in Nadaraya--Watson Interpolators},
  booktitle = {Advances in Neural Information Processing Systems},
  volume    = {38},
  pages     = {15320--15348},
  year      = {2025},
  doi       = {10.52202/085713-0515}
}

@inproceedings{du2021fewshot,
  author    = {Du, Simon S. and Hu, Wei and Kakade, Sham M.
               and Lee, Jason D. and Lei, Qi},
  title     = {Few-Shot Learning via Learning the Representation, Provably},
  booktitle = {International Conference on Learning Representations},
  year      = {2021},
  eprint    = {2002.09434},
  archivePrefix = {arXiv}
}

@inproceedings{tripuraneni2021provable,
  author    = {Tripuraneni, Nilesh and Jin, Chi and Jordan, Michael I.},
  title     = {Provable Meta-Learning of Linear Representations},
  booktitle = {Proceedings of the 38th International Conference on Machine Learning},
  series    = {Proceedings of Machine Learning Research},
  volume    = {139},
  pages     = {10434--10443},
  year      = {2021},
  publisher = {PMLR}
}

@misc{hong2026closedformformulaconsistentlipschitz,
      title={A Closed-Form Formula for Consistent Lipschitz Regression on Metric Spaces with Sparse Neural Network Realizations}, 
      author={Ruiyang Hong and Hrad Ghoukasian and Anastasis Kratsios},
      year={2026},
      eprint={2609.03129},
      archivePrefix={arXiv},
      primaryClass={stat.ML},
      url={https://arxiv.org/abs/2609.03129}, 
}

\clearpage 
\pagenumbering{Roman}

\appendix
\section{Supplement: Supporting Results for Paper's Prose}
\label{s:AuxiliatyResults}

We now collect the various claims made throughout our manuscript supporting its positioning and claims, but not needed for our main theorems.

\subsection{Supporting Non-Convexity Results}
\label{s:AuxiliatyResults__ss:NonConvexity}

We collect several supporting observations concerning the optimization landscape of our attention model.  These results are not used in the recovery or OOD proofs; rather, they substantiate the discussion that the relevant joint training objectives fall outside standard convex and global Polyak-{\L}ojasiewicz frameworks.

\begin{proposition}[Joint Non-Convexity of Attention Loss]
\label{prop:nonconvex_loss}
Fix $N,d,D,d_k\in\mathbb{N}_+$ with $d_k\geq2$, fix $\lambda>0$, and let $\{(x_n,y_n)\}_{n=1}^N\subset \mathbb{R}^d\times\mathbb{R}^D$. Let $\operatorname{Attn}_{\lambda}$ be as in~\eqref{eq:ENCODER}. Then the map $\operatorname{Loss}:\mathbb{R}^{D\times d_k}\times\mathbb{R}^{d\times d_k}\to[0,\infty)$ in~\eqref{eq:Training_Loss} is not convex on its domain.
\end{proposition}

The idea is to identify a two-dimensional linear slice of the parameter space along which the attention loss is smooth and has an indefinite Hessian, which rules out convexity of the full loss.

\begin{proof}[{Proof of Proposition~\ref{prop:nonconvex_loss}}]
Let $e_1\in\mathbb{R}^d$ be the first canonical basis vector, fix $u\in\mathbb{R}^D$ with $\|u\|_2=1$, and set $R\eqdef1+\max_{1\leq n\leq N}(x_{n,1}+\|x_n\|_\infty)$. For $(v,t)\in\mathbb{R}^2$, define $V(v)\in\mathbb{R}^{D\times d_k}$ and $K(t)\in\mathbb{R}^{d\times d_k}$ by $V_1(v)\eqdef vu$, $V_j(v)\eqdef0$ for $j\geq2$, $K_1(t)\eqdef te_1$, and $K_j(t)\eqdef0$ for $j\geq2$, and set $\ell(v,t)\eqdef\operatorname{Loss}(V(v),K(t))$.

Fix $t>R$. For every $n\in\{1,\dots,N\}$, the definition of $R$ gives $t-x_{n,1}>\|x_n\|_\infty$, and therefore $\|x_n-te_1\|_\infty=t-x_{n,1}$. Hence the softmax weight associated with the first key equals $p_n(t)=(1+c_ne^{t/\lambda})^{-1}$, where $c_n\eqdef(d_k-1)\exp(-(x_{n,1}+\|x_n\|_\infty)/\lambda)>0$. Since all value columns except the first vanish, $\operatorname{Attn}_{\lambda}(x_n|K(t),V(v))=vp_n(t)u$. Consequently,
\begin{equation}
\label{eq:nonconvex_loss_slice}
\ell(v,t)=A(t)v^2-2B(t)v+C,
\end{equation}
where $A(t)\eqdef \tfrac{1}{N}\sum_{n=1}^Np_n(t)^2>0$, $B(t)\eqdef \tfrac{1}{N}\sum_{n=1}^Np_n(t)\langle u,y_n\rangle$, and $C\eqdef \tfrac{1}{N}\sum_{n=1}^N\|y_n\|_2^2$. In particular, $\ell$ is infinitely differentiable on $\mathbb{R}\times(R,\infty)$.
Differentiating~\eqref{eq:nonconvex_loss_slice} yields
\begin{equation}
\label{eq:nonconvex_loss_det}
\det\bigl(\nabla^2\ell(v,t)\bigr)
=
\Gamma(t)v^2+\Lambda(t)v-4B'(t)^2
\mbox{ and }
\Gamma(t)\eqdef2A(t)A''(t)-4A'(t)^2
\end{equation}
where $\Lambda(t)\eqdef-4A(t)B''(t)+8A'(t)B'(t)$.
We claim that $\Gamma(t)<0$ for all sufficiently large $t$. Indeed, set $q\eqdef e^{-t/\lambda}$ and $a\eqdef \tfrac{1}{N}\sum_{n=1}^Nc_n^{-2}>0$. Since $p_n(t)=q/(q+c_n)$, direct differentiation gives, as $q\downarrow0$,
$
q^{-2}A(t)\to a
$,
$
q^{-2}A'(t)\to-2a/\lambda
$,
and
$
q^{-2}A''(t)\to4a/\lambda^2
$.
Therefore
$
e^{4t/\lambda}\Gamma(t)
=
q^{-4}\Gamma(t)$
and 
\[
\lim\limits_{q\downarrow 0}\,
e^{4t/\lambda}\Gamma(t)
=
2a\frac{4a}{\lambda^2}
-
4\left(\frac{2a}{\lambda}\right)^2
=
-\frac{8a^2}{\lambda^2}
<0.
\]
Thus there exists $T>R$ such that $\Gamma(t)<0$ for every $t\geq T$.
Fix any such $t$. By~\eqref{eq:nonconvex_loss_det}, 
\[
\lim\limits_{|v|\to\infty}\,
\det(\nabla^2\ell(v,t))/v^2 = \Gamma(t)<0
.
\]
Hence there exists $v\in\mathbb{R}$ for which $\det(\nabla^2\ell(v,t))<0$. The Hessian $\nabla^2\ell(v,t)$ is therefore indefinite. Since $(v,t)\mapsto(V(v),K(t))$ is linear, convexity of $\operatorname{Loss}$ would imply convexity of its restriction $\ell$. But a twice continuously differentiable convex function has a positive semidefinite Hessian at every point of its open domain, contradicting the preceding conclusion. Thus $\operatorname{Loss}$ is not convex.
\end{proof}

\begin{proposition}[Failure of the Global PL Inequality]
\label{prop:nonPL_loss}
For every $\lambda>0$, there exist $N,d,D,d_k\in\mathbb{N}_+$ with $d_k\geq2$ and training data $\{(x_n,y_n)\}_{n=1}^N\subset\mathbb{R}^d\times\mathbb{R}^D$ such that the loss $\operatorname{Loss}$ in~\eqref{eq:Training_Loss} admits a differentiability point $(V,K)$ which is stationary but not globally minimizing. Consequently, there is no $\mu>0$ for which the Polyak--{\L}ojasiewicz inequality $\tfrac{1}{2}\|\nabla\operatorname{Loss}(V,K)\|_F^2\geq\mu(\operatorname{Loss}(V,K)-\inf\operatorname{Loss})$ holds at every differentiability point $(V,K)$.
\end{proposition}

The idea is to construct an explicit smooth stationary point with strictly positive loss and then exhibit parameters attaining zero loss, thereby contradicting the Polyak--{\L}ojasiewicz inequality.

\begin{proof}[{Proof of Proposition~\ref{prop:nonPL_loss}}]
Fix $\lambda>0$ and take $N=2$, $d=D=1$, $d_k=2$, $(x_1,y_1)=(0,0)$, and $(x_2,y_2)=(1,1)$. Fix $t>1$ and $s>0$, and set $K=(t,t+s)$ and $V=(\tfrac12,\tfrac12)$. Since $|x_n-t|=t-x_n$ and $|x_n-(t+s)|=t+s-x_n$ for $n=1,2$, both observations have the same attention weights $\alpha=(1+e^{-s/\lambda})^{-1}(1,e^{-s/\lambda})$. Hence $\operatorname{Attn}_{\lambda}(x_n|K,V)=\tfrac12$ for $n=1,2$, and therefore $\operatorname{Loss}(V,K)=\tfrac14$.
\hfill\\
\noindent
Since $t>1$ and $t+s>1$, none of the absolute values defining the attention weights vanishes at $K$, so $\operatorname{Loss}$ is differentiable at $(V,K)$. Writing $\alpha=(\alpha_1,\alpha_2)$, for $j=1,2$ we obtain $\partial_{V_j}\operatorname{Loss}(V,K)=\alpha_j((\tfrac12-y_1)+(\tfrac12-y_2))=0$. Moreover, since $V_1=V_2=\tfrac12$ and the attention weights sum to one, $\operatorname{Attn}_{\lambda}(x|K,V)=\tfrac12$ independently of $K$ wherever the key derivatives exist; hence $\nabla_K\operatorname{Loss}(V,K)=0$. Thus $\nabla\operatorname{Loss}(V,K)=0$.
\hfill\\
\noindent
On the other hand, set $\rho\eqdef e^{-1/\lambda}\in(0,1)$, $K^\star=(0,1)$, and $V^\star=(-\rho/(1-\rho),1/(1-\rho))$. Direct substitution into~\eqref{eq:ENCODER} gives $\operatorname{Attn}_{\lambda}(0|K^\star,V^\star)=0$ and $\operatorname{Attn}_{\lambda}(1|K^\star,V^\star)=1$, whence $\operatorname{Loss}(V^\star,K^\star)=0$. Since $\operatorname{Loss}\geq0$, we have $\inf\operatorname{Loss}=0$, whereas $\operatorname{Loss}(V,K)=\tfrac14$ and $\nabla\operatorname{Loss}(V,K)=0$. Therefore, for every $\mu>0$, the Polyak-{\L}ojasiewicz inequality at $(V,K)$ would imply $0\geq\mu/4>0$, a contradiction.
\end{proof}

\begin{proposition}[Joint Non-Convexity of Factorized Values]
\label{prop:nonconvex_factorized_values}
Fix $N,d,D,d_k,\kappa\in\mathbb{N}_+$, $\lambda>0$, a training set $\{(x_n,y_n)\}_{n=1}^N\subset\mathbb{R}^d\times\mathbb{R}^D$, and $K\in\mathbb{R}^{d\times d_k}$. Define $\operatorname{Loss-ft}:\mathbb{R}^{D\times\kappa}\times\mathbb{R}^{\kappa\times d_k}\to[0,\infty)$ by
\begin{equation}
\label{eq:factorized_training_loss}
\operatorname{Loss-ft}(A,C)\eqdef\frac{1}{N}\sum_{n=1}^N\big\|\operatorname{Attn}_{\lambda}(x_n|K,AC)-y_n\big\|_2^2.
\end{equation}
Then $\operatorname{Loss-ft}$ is not jointly convex in $(A,C)$.
\end{proposition}

Intuitively, the proof restricts the factorized loss to a two-dimensional linear slice on which the softmax normalization reduces the objective to a scalar bilinear least-squares loss whose Hessian is indefinite at an explicit point.

\begin{proof}[{Proof of Proposition~\ref{prop:nonconvex_factorized_values}}]
For each $n\in\{1,\dots,N\}$, set $\alpha_n\eqdef\operatorname{Softmax}_{\lambda}(-(\|x_n-K_i\|_\infty)_{i=1}^{d_k})\in(0,1]^{d_k}$.
Since $K$ is fixed, each $\alpha_n$ is fixed, and the definition of the softmax gives $\mathbf{1}_{d_k}^{\top}\alpha_n=1$, where $\mathbf{1}_{d_k}\in\mathbb{R}^{d_k}$ denotes the vector whose entries are all equal to one.

Let $e_1\in\mathbb{R}^{\kappa}$ be the first canonical basis vector and fix $u\in\mathbb{R}^D$ with $\|u\|_2=1$. For $(a,c)\in\mathbb{R}^2$, define $A(a)\eqdef au e_1^\top\in\mathbb{R}^{D\times\kappa}$ and $C(c)\eqdef ce_1\mathbf{1}_{d_k}^\top\in\mathbb{R}^{\kappa\times d_k}$. By~\eqref{eq:ENCODER}, for every $n\in\{1,\dots,N\}$, $\operatorname{Attn}_{\lambda}(x_n|K,A(a)C(c))=A(a)C(c)\alpha_n=ac\,u\,\mathbf{1}_{d_k}^{\top}\alpha_n=ac\,u$. Hence, setting $\ell(a,c)\eqdef\operatorname{Loss-ft}(A(a),C(c))$, equation~\eqref{eq:factorized_training_loss} gives $\ell(a,c)=a^2c^2-2bac+\gamma$, where $b\eqdef\frac{1}{N}\sum_{n=1}^N\langle u,y_n\rangle$ and $\gamma\eqdef\frac{1}{N}\sum_{n=1}^N\|y_n\|_2^2$.
The function $\ell$ belongs to $C^\infty(\mathbb{R}^2)$, and direct differentiation yields
\begin{equation}
\label{eq:factorized_training_loss_hessian}
\nabla^2\ell(a,c)=
\begin{pmatrix}
2c^2 & 4ac-2b\\
4ac-2b & 2a^2
\end{pmatrix}.
\end{equation}
Set $t\eqdef\sqrt{|b|+1}$, so that $t^2>|b|$. Evaluating the determinant in~\eqref{eq:factorized_training_loss_hessian} at $(t,t)$ gives $\det(\nabla^2\ell(t,t))=4t^4-(4t^2-2b)^2=-4(3t^2-b)(t^2-b)$. Since $t^2>|b|$, both $3t^2-b>0$ and $t^2-b>0$, and therefore $\det(\nabla^2\ell(t,t))<0$. As $\nabla^2\ell(t,t)$ is a real symmetric $2\times2$ matrix, its eigenvalues are real and their product equals its determinant; hence one eigenvalue is strictly positive and the other is strictly negative. In particular, $\nabla^2\ell(t,t)$ is not positive semidefinite.
\hfill\\
Finally, define $\Phi:\mathbb{R}^2\to\mathbb{R}^{D\times\kappa}\times\mathbb{R}^{\kappa\times d_k}$ by $\Phi(a,c)\eqdef(A(a),C(c))$. The map $\Phi$ is linear and $\ell=\operatorname{Loss-ft}\circ\Phi$. If $\operatorname{Loss-ft}$ were convex, then $\ell$ would be convex on $\mathbb{R}^2$. Since $\ell\in C^\infty(\mathbb{R}^2)$, this would imply that $\nabla^2\ell(a,c)$ is positive semidefinite for every $(a,c)\in\mathbb{R}^2$, contradicting the preceding conclusion at $(t,t)$. Therefore $\operatorname{Loss-ft}$ is not jointly convex.
\end{proof}

\section{Extended Technical Notation}
\label{a:Notation}

The main body retains the task, grid, and positive-temperature notation needed by its principal theorem statements.  For readers entering the proofs and extensions, we collect the recurring notation here and record the hard-attention convention used only below.  The labelled definitions and equations in the manuscript remain controlling, and symbols introduced and discharged within a single proof retain their local definitions.

\begin{enumerate}[leftmargin=*]

    \item $\mathbb N_+\eqdef\{1,2,\dots\}$,
    $[m]\eqdef\{1,\dots,m\}$ for $m\in\mathbb N_+$, and $\#A$ denotes the
    cardinality of a finite set $A$.

    \item $I_A$ and $\chi_A$ denote indicators, $e_j$ is the $j$-th standard
    basis vector, and $\eqdef$ denotes equality by definition.

    \item $\|\cdot\|_r$ is the usual finite-dimensional $\ell^r$-norm, with
    $\|\cdot\|_\infty$ the maximum norm.  Moreover,
    $\|g\|_{L^\infty(\mathbb{P})}$ is the $\mathbb{P}$-essential supremum of
    $\|g\|_2$, and
    $\operatorname{dist}_\infty(x,A)\eqdef\inf_{a\in A}\|x-a\|_\infty$.

    \item $\operatorname{supp}(\mathbb{P})$ is the topological support of a
    Borel probability measure $\mathbb{P}$, and $\partial A$ is the boundary
    of $A$.

    \item $d$ and $D$ are the input and output dimensions, $q$ is the
    regular-grid resolution, and $p\in[1,\infty)$ is the recovery-moment
    exponent.  In the adaptive theorem, $N$ is the size of each half-sample,
    so the total source sample size is $2N$; in the fixed-grid chain, $N$ is
    the total source sample size.  The target sample size is
    $\widetilde{N}$.

    \item
    $Q_d(x,\rho)\eqdef\{z\in\mathbb R^d:\|z-x\|_\infty\le\rho\}$ is the
    closed $\ell^\infty$-cube of radius $\rho$ centred at $x$, distinct from
    an open grid cell $Q\in\mathcal{Q}_{q:d}$.

    \item $\mathcal{Q}_{q:d}$ is the family of $q^d$ open regular-grid cells
    in $[0,1]^d$; $Q_i$ is indexed by
    $i\in\{0,\dots,q-1\}^d$, $\bar{Q}$ is the midpoint of $Q$, and
    $\mathcal{G}=\bigcup_{Q\in\mathcal{Q}_{q:d}}Q$ is the complement of the
    grid skeleton.

    \item A regression task is
    $\tau=(f,\mathbb{P}_X,\mathbb{P}_\varepsilon)$, where
    $f:[0,1]^d\to\mathbb R^D$ is the regression function and
    $\mathbb{P}_X$ and $\mathbb{P}_\varepsilon$ are the input and noise laws.
    A $\tau$-training set satisfies
    $Y_n=f(X_n)+\varepsilon_n$, with i.i.d.\ inputs, i.i.d.\ noises, and
    independence between the two families.

    \item Tildes distinguish target-task objects from their source
    analogues, as in
    $\tilde{\tau}=(\tilde{f},\tilde{\mathbb{P}}_X,
    \tilde{\mathbb{P}}_\varepsilon)$; hats denote empirically constructed
    objects, such as $\widehat K$, $\widehat c$, and
    $\hat{\mathcal{T}}_\lambda$.

    \item $C$ and $\sigma^2$ are the radial sub-Gaussian tail prefactor and
    scale,
    $\bar\sigma\eqdef\max\{1,\sigma\}$, and
    $\sigma_C\eqdef\sigma\max\{C,1\}$.  Their target analogues are
    $\widetilde C$, $\widetilde\sigma$,
    $\widetilde{\bar\sigma}\eqdef\max\{1,\widetilde\sigma\}$, and
    $\widetilde\sigma_{\widetilde C}
    \eqdef\widetilde\sigma\max\{\widetilde C,1\}$.

    \item $\omega$ and $\widetilde\omega$ are the source and target moduli of
    continuity.  The symbols $L_f$, $\Delta$, $\vartheta$, $\varepsilon$,
    and $\delta$ denote, respectively, the source Lipschitz constant in the
    uniform specialization, source class-value separation, adaptive
    clustering tolerance, recovery tolerance, and failure-probability
    level.

    \item $\mathcal S=\{\mathcal S_k\}_{k=1}^{\kappa}$ partitions
    $\mathcal{Q}_{q:d}$ into $\kappa$ source classes,
    $c:\mathcal{Q}_{q:d}\to[\kappa]$ is the class-assignment map, and
    $\beta_k$ is the common source midpoint value on $\mathcal S_k$.
    The target analogues are
    $\widetilde{\mathcal S}$, $\widetilde\kappa$,
    $\widetilde c$, and $\widetilde\beta_j$.

    \item $\operatorname{Softmax}_\lambda$ is the softmax map at temperature
    $\lambda>0$, while $\operatorname{Softmax}_0$ is its
    uniform-on-maximizers version from
    \eqref{eq:zero_temperature_softmax_definition}.
    The corresponding attention map is
    $\operatorname{Attn}_\lambda(\cdot\mid K,V)$, with $\lambda=0$
    denoting hard attention.

    \item $d_k$ is the number of keys,
    $K\in\mathbb R^{d\times d_k}$ is the key matrix with columns $K_i$, and
    $V\in\mathbb R^{D\times d_k}$ is the effective values matrix.  Only in
    the input line of Algorithm~\ref{alg:STEP_1_PURIFY},
    $K\in\mathbb N_+$ denotes the requested number of adaptive centres.

    \item In the tied-values factorization $V=AC$,
    $A=(v_1,\dots,v_\kappa)\in\mathbb R^{D\times\kappa}$ stores the class
    values and $C\in\{0,1\}^{\kappa\times d_k}$ assigns keys to classes.
    The learned factors satisfy
    $\hat{V}=\widehat A C_{\widehat c}$.

    \item $\mathcal D$, $\mathcal D^{(0)}$, and $\mathcal D^{(1)}$ denote an
    adaptive dataset and its pilot and averaging halves;
    $\mathcal D^\star$ is the denoised virtual dataset; and
    $\mathcal D_{\operatorname{src}}$ is the fixed-grid source dataset used
    in the OOD construction.

    \item In Algorithm~\ref{alg:STEP_1_PURIFY},
    $\delta_{\mathrm{sep}}$ is the centre-separation level,
    $\bar{Q}_1,\dots,\bar{Q}_K$ are the selected centres,
    $\mathcal{Q}=\{\mathcal{Q}_k\}_{k=1}^K$ is their smallest-index-tie-broken
    Voronoi partition, and
    $\mathcal{Q}_{k,\rho}=\mathcal{Q}_k\cap Q_d(\bar{Q}_k,\rho)$ is the
    localized $k$-th cell.

    \item $\mathcal Z_{2q}$ is the deterministic fine midpoint grid used as
    the reference adaptive centre configuration, and
    $z^{\operatorname{grid}}$ is its fixed ordering.
    The family $\mathcal{U}=\{U_\ell\}_{\ell=1}^L$ is the deterministic pilot
    cover, $p_{\operatorname{cov}}$ is its minimum mass, and
    $\underline{p}_\rho$ is the uniform lower bound on localized population
    mass.

    \item For a fixed-grid cell $Q$,
    $Q_\rho\eqdef Q\cap Q_d(\bar{Q},\rho)$ is its localized region,
    $\mathcal{Q}^\star$ is the family of active cells,
    $p_\rho^\star$ is the minimum localized-cell mass, and
    $r_\rho^\star$ is the maximal localized support radius.

    \item $I_{k,\rho}$, $N_{k,\rho}$, and $\bar Y_{k,\rho}$ are the adaptive
    localized index set, count, and empirical label average.  Their
    fixed-grid analogues are
    $I_{Q,\rho}$, $N_{Q,\rho}$, and $\bar Y_{Q,\rho}$; the latter is also
    denoted by $Y_{Q,\rho}$ in the fixed-grid clustering statements.

    \item $\mathcal F_0$ is generated by the adaptive pilot inputs, while
    $\mathcal F_X$ is generated by the inputs relevant to a localized-average
    argument.  In the OOD proof, $U_{\operatorname{src}}$ is the source-side
    clustering randomness and $\mathcal F_{\operatorname{src}}$ contains
    the complete source experiment.

    \item $M_{\delta}$ is the required adaptive per-cell occupancy.
    The quantities $W_q$, $a_q$, $B_q$, and $s_q$ are the grid-scale
    modulus, hard-attention error bound, deterministic value bound, and
    auxiliary tolerance in the finite-temperature source proof; tildes
    denote their target analogues.

    \item $\hat{\kappa}$, $\widehat c$, $\widehat K$, $\widehat A$, and
    $\hat{V}$ are the recovered class count, class assignment, key
    matrix, class-value matrix, and effective values matrix.
    The adaptive and fixed-grid source models are denoted by
    $\hat{\mathcal{T}}_\lambda$.

    \item
    $\operatorname{Sim}_q(\widetilde{\mathcal S}\mid\mathcal S)$ is the task
    similarity ratio, $\alpha$ is its prescribed lower bound,
    $\pi_\star$ is an optimizing source-to-target class membership,
    $\mathcal{Q}_{\pi_\star}$ is its compatible cell family, and
    $\mathcal{G}_{\pi_\star}$ is their union.

    \item $\mathcal{Q}_{\mathbb{P}}$ and
    $\widetilde{\mathcal{Q}}_{\mathbb{P}}$ are the source- and target-active
    grid cells, and $\widetilde p^\star$ is the minimum target mass of a
    grid cell.

    \item $\widehat c^{\,\circ}$ is the recovered source assignment extended
    to the full grid, $\mathbb{S}_\kappa$ is the permutation group of
    $[\kappa]$, and $\mathcal{E}_{\operatorname{src}}$ is the event of
    complete source-grid occupancy and exact class recovery up to
    permutation.

    \item $\widetilde I_Q$, $\widetilde{N}_Q$,
    $\widetilde{\mathcal{Q}}_{\operatorname{obs}}$, and $\widetilde Y_Q$ are
    the target-cell index set, count, observed-cell family, and empirical
    label average.

    \item $\widetilde{\mathcal{Q}}^{\,\operatorname{sim}}_j$ is the compatible
    observed family assigned to recovered source class $j$,
    $\widetilde v_j$ is its fine-tuned value, and
    $\widetilde A$, $\widetilde V$, and
    $\widetilde{\mathcal T}_\lambda$ are the resulting class-value matrix,
    effective values matrix, and fine-tuned attention model.

    \item $\widetilde r_Q$ is the target support radius of $Q$, and
    $\widetilde r^\star$ is its maximum over the grid.
    Moreover,
    $\mathcal S_q=[0,1]^d\setminus\mathcal{G}$ is the grid skeleton,
    $\mathcal B_\zeta$ is its open $\zeta$-tube in
    $\ell^\infty$-distance, and $\lambda_0$ and
    $\widetilde\lambda_0$ are deterministic source and target temperature
    thresholds.

    \item $C_{d,A}$ and $C_d^{\operatorname{gap}}$ are the
    packing-controlled global-leakage and score-gap constants,
    respectively; both are deterministic and independent of the number of
    keys.

    \item $\xi$ and $\widetilde\xi$ are the source and target
    approximate-symmetry tolerances.
    The set $J_{\pi_\star}$ contains the positive-target-mass compatible
    source classes,
    $\tilde{p}_{\operatorname{cls}}^\star$ is their minimum positive
    target mass, and
    $\widetilde{\mathcal T}^{\operatorname{pool}}_\lambda$ is the
    oracle-masked pooled fine-tuned model.

    \item $\widetilde{\mathcal D}_{\operatorname{disc}}$ and
    $\widetilde{\mathcal D}_{\operatorname{val}}$ are the independent target
    discovery and value samples, with sizes
    $\widetilde{N}_{\operatorname{disc}}$ and
    $\widetilde{N}_{\operatorname{val}}$; their sum is the total target sample
    size in the data-driven extension.

    \item $\widehat\pi$ is the lexicographically selected optimizer in
    recovered labels, $\Pi_\star$ is the set of population optimizers, and
    $\widetilde\lambda_0^{\operatorname{disc}}$ is the minimum deterministic
    positive temperature threshold over $\Pi_\star$.
\end{enumerate}

\addcontentsline{toc}{chapter}{Appendices}
\renewcommand{\addcontentsline}[3]{}

\end{document}